\documentclass{article}

\usepackage[final,main]{neurips_2026}

\usepackage[utf8]{inputenc} 
\usepackage[T1]{fontenc}    
\usepackage{hyperref}       
\usepackage{url}            
\usepackage{booktabs}       
\usepackage{amsfonts}       
\usepackage{nicefrac}       
\usepackage{microtype}      
\usepackage{xcolor}         
\usepackage{tikz}
\usetikzlibrary{positioning,arrows.meta}

\usepackage{xspace}
\usepackage{enumitem}
\usepackage{amsmath,amsthm,amssymb}
\usepackage{nicefrac}
\usepackage{bm}
\usepackage{multirow}
\usepackage{pifont}

\theoremstyle{plain}
\newtheorem{theorem}{\textbf{Theorem}}
\newtheorem{lemma}[theorem]{\textbf{Lemma}}
\newtheorem{corollary}[theorem]{Corollary}
\newtheorem{proposition}[theorem]{Proposition}

\theoremstyle{definition}
\newtheorem{definition}{Definition}
\newtheorem{assumption}{Assumption}

\newcommand{\emphb}[1]{\textbf{#1}}

\newcommand{\PP}{\mathbb{P}}
\newcommand{\EE}{\mathbb{E}}

\newcommand{\bone}{\mathbf{1}}

\DeclareMathOperator*{\argmin}{arg\,min}
\DeclareMathOperator*{\argmax}{arg\,max}

\newcommand{\Fcal}{\mathcal{F}}
\newcommand{\Gcal}{\mathcal{G}}
\newcommand{\Scal}{\mathcal{S}}
\newcommand{\Acal}{\mathcal{A}}
\newcommand{\Tcal}{\mathcal{T}}

\newcommand{\Lcal}{\mathcal{L}}

\newcommand{\Wcal}{\mathcal{W}}
\newcommand{\Bcal}{\mathcal{B}}
\newcommand{\Ccal}{\mathcal{C}}

\newcommand{\Rmax}{R_{\max}}
\newcommand{\Vmax}{V_{\max}}
\newcommand{\RR}{\mathbb{R}}
\newcommand{\Var}{\mathrm{Var}}
\newcommand{\Xcal}{\mathcal{X}}
\newcommand{\emp}[1]{\widehat{#1}}

\newcommand{\alg}{\textsc{Q-MMR}\xspace}
\newcommand{\ii}{{(i)}}
\newcommand{\nn}{{[n]}}
\newcommand{\wh}{\emp{w}}
\newcommand{\wt}{\widetilde{w}}
\newcommand{\wb}{\bar{w}}
\newcommand{\lossmatch}{\emp{\Lcal}}
\newcommand{\dm}{{\mathsf{d}}}
\newcommand{\nmbd}{\Theta}
\newcommand{\tab}{\textrm{tab}}

\newcommand{\train}{\textup{train}}
\newcommand{\test}{\textup{test}}
\newcommand{\xmean}{\overline{x}_\test}
\newcommand{\hatSigma}{\emp{\Sigma}}
\newcommand{\hatLambda}{\emp{\Lambda}}
\newcommand{\hattheta}{\emp{\theta}}
\newcommand{\hatP}{\emp{P}}
\newcommand{\hatnu}{\emp{\psi}}

\newcommand{\epsstat}{\epsilon^{\textup{stat}}}
\newcommand{\secmom}[1]{\|\textstyle #1\|_{[n]}}

\newcommand{\ft}{\psi}
\newcommand{\Bpi}{B^\pi}
\newcommand{\Sigmacr}{\Sigma^{\textup{cr}}}

\newcommand{\Hilb}{\overline{\Fcal}}

\newcommand{\bp}{K^\pi}
\newcommand{\evn}{|_n}

\newcommand{\wstar}{w^\star}
\newcommand{\poploss}{\Lcal}

\usepackage{tcolorbox}
\newcounter{resct}

\usepackage{algorithm}
\usepackage{algpseudocode}
\algrenewcommand\algorithmicrequire{\textbf{Input:}}
\algrenewcommand\algorithmicensure{\textbf{Output:}}

\title{\alg: Off-Policy Evaluation via \\ Recursive Reweighting and Moment Matching}

\author{%
  Xiang Li \\
  UIUC\\
  \texttt{xiang306@illinois.edu} \\
  \And
  Nan Jiang \\
  UIUC \\
  \texttt{nanjiang@illinois.edu}
}

\begin{document}

\maketitle

\begin{abstract}
We present a novel framework, \alg, for off-policy evaluation in finite-horizon MDPs. \alg learns a set of scalar weights, one for each data point, such that the reweighted rewards approximate the expected return under the target policy. The weights are learned inductively in a top-down manner via a moment matching objective against a value-function discriminator class. Notably, and perhaps surprisingly, a data-dependent finite-sample guarantee for general function approximation can be established under only the realizability of $Q^\pi$, with a \textit{\textbf{dimension-free}} bound---that is, the error does \textbf{\textit{not}} depend on the statistical complexity of the function class. We also establish connections to several existing methods, such as  importance sampling and 
linear FQE. Further theoretical analyses shed new light on the nature of coverage, a concept of fundamental importance to offline RL. 
\end{abstract}

\section{Introduction}
\label{sec:intro}

Off-policy evaluation (OPE) using value-function approximation is a fundamental and core problem in the theory of reinforcement learning (RL). In this task, one estimates the expected return of a target policy $\pi$ using data collected from some behavior policy, by leveraging a function class $\Fcal$ that captures the target policy's Q-value function $Q^\pi$. For general function approximation, the canonical approach is Fitted-Q Evaluation (FQE), which solves a sequence of least-square regression problems \citep{ernst2005tree, chen2019information, le2019batch}. FQE can also be viewed as the theoretical prototype of the TD algorithm, which is a foundational building block of RL. 

While the theory of OPE  with general value-function approximation is largely considered a mature area, where finite-sample guarantees for FQE and related algorithms have been established \citep{jiang2025offline}, there are still notable gaps in our understanding: 
\begin{itemize}[leftmargin=*]
\item \textbf{Dimension dependence.} All existing analyses under general function approximation (GFA) need to translate the empirical loss to its  population counterpart, during which one pays for the statistical dimension of the function class (e.g., $\log|\Fcal|$ for finite $\Fcal$ or the log covering number for continuous classes). While this dependence is widely accepted, there are clearly situations where this is loose: for example, when the data is on-policy, Monte-Carlo (MC) estimate provides straightforward guarantees without any $\log|\Fcal|$ dependence, raising the question of whether the dimension dependence is real or an artifact of the algorithm design or its analysis. 
\item \textbf{Inconsistency between the general and the linear analyses.} Apart from the analyses under GFA, researchers have also obtained refined, sharp (i.e., often minimax optimal) guarantees in the linear and the tabular settings \citep{duan2020minimax, yin2021towards}. As it turns out, the main $1/\sqrt{n}$ terms in these guarantees are indeed dimension-free. 
Moreover, the linear analyses hold under strictly weaker conditions than the Bellman-completeness assumption widely adopted in the GFA analyses. In fact, not only the linear analyses cannot be obtained by specializing the GFA analyses to the linear classes, they also give rise to novel concepts---such as seemingly unconventional notions of coverage (see Section~\ref{sec:weight})---that are nowhere to be found in the GFA analyses. 
\end{itemize}

In this paper, we address the above gaps with a novel algorithmic framework for OPE 
in finite-horizon MDPs. The \alg algorithm learns a set of weights, one for each data point, inductively in a top-down manner via a moment-matching objective against a value-function discriminator class $\Fcal$. Our main guarantee is a data-dependent finite-sample bound (Theorem~\ref{thm:main}, Section~\ref{sec:method}), which is \textbf{dimension-free} in that it has no explicit dependence on the complexity of $\Fcal$. (A data-independent bound will be provided in Section~\ref{sec:weight}, where only the higher-order term is dimension dependent.) For linear $\Fcal$, \alg coincides with linear FQE; unlike the GFA analyses of FQE (which is dimension-dependent and requires Bellman completeness), our guarantees fully recover the key properties of the fine-grained linear and tabular analyses, under only the realizability of $Q^\pi$. 

A key quantity in our bound is the second moment of the learned weights, which plays the role of an empirical coverage parameter \citep{amortila2026unifying}. In further analyses (Section~\ref{sec:weight}), we identify its population counterpart and examine its properties, which shed light on the nature of coverage in offline RL. 
The high-level take-away messages are consistent with the (somewhat counterintuitive) recent findings \citep{perdomo2023complete, amortila2026unifying} (e.g., the coverage target is \textit{\textbf{not}} the feature occupancy in the real MDP, but that from an approximate linear dynamical system; see Eq.\eqref{eq:lds}), but those findings were restricted to the linear setting and difficult to generalize. Our algorithms and analyses provide a much-needed generalization of such insights to GFA; see Table~\ref{tab:compare} for a summary.

\begin{table}[]
\centering
\caption{Comparison between existing analyses and ours. Each entry corresponds to an analysis of an algorithm (row) in a setting (column). Columns to the right represent more specialized settings, whose results are plug-in instantiations of more general analyses to the left. (This is why ``FQE'' and ``Linear FQE'' do not coincide in the linear columns, even though they become the same algorithm there.) 
For our analyses, we use the population version of coverage (see Section~\ref{sec:weight}) for better comparison with the existing literature. ``Dim-free" refers to the property of not depending on the dimension $\dm$ or $\log|\mathcal{F}|$-type terms in the data-dependent bound, or only depending on it in higher-order terms in the data-independent bound. ${\color{blue} {}^{(*)}}$Technically, the linear analysis of \citet{duan2020minimax} considers the linear-MDP setup which is stronger than Bellman completeness, but it is likely that their proofs and results can be adapted  to the weaker settings (e.g., only realizability), especially given the later insights of \citet{perdomo2023complete, amortila2026unifying}; see Appendix~\ref{app:prior-linear} for details. 
\label{tab:compare}}
\makebox[\textwidth][c]{%
\begin{tabular}{|l|cl|cc|cc|cc|}
\hline
                   & \multicolumn{2}{c|}{\begin{tabular}[c]{@{}c@{}}General + \\  realizability\end{tabular}} & \multicolumn{2}{c|}{\begin{tabular}[c]{@{}c@{}}General +\\  completeness\end{tabular}}                       & \multicolumn{2}{c|}{\begin{tabular}[c]{@{}c@{}}Linear +\\  realizability\end{tabular}}                                  & \multicolumn{2}{c|}{\begin{tabular}[c]{@{}c@{}}Linear +
                   \\ completeness\end{tabular}}                                    \\ \cline{2-9} 
\multirow{-2}{*}{} & \multicolumn{1}{c|}{Coverage}  & \begin{tabular}[c]{@{}l@{}}Dim-\\ free?\end{tabular}  & \multicolumn{1}{c|}{Coverage}   & \multicolumn{1}{l|}{\begin{tabular}[c]{@{}l@{}}Dim-\\ free?\end{tabular}} & \multicolumn{1}{c|}{Coverage}              & \multicolumn{1}{l|}{\begin{tabular}[c]{@{}l@{}}Dim-\\ free?\end{tabular}} & \multicolumn{1}{c|}{Coverage}               & \multicolumn{1}{l|}{\begin{tabular}[c]{@{}l@{}}Dim-\\ free?\end{tabular}} \\ \hline
FQE                & \multicolumn{2}{c|}{{\color[HTML]{FE0000} }}                                           & \multicolumn{1}{c|}{$\underset{f,f'\in\Fcal_h}{\sup}\frac{|\EE_{d_h^\pi}[f-f']|}{\sqrt{\EE_{d_h^D}[(f - f')^2]}}$} & \textcolor{red}{\ding{55}}                                                & \multicolumn{2}{c|}{{\color[HTML]{FE0000} \textbf{N/A}}}                                                                        & \multicolumn{1}{c|}{$\|\EE_{d_h^\pi}[\phi]\|_{\Sigma_h^{-1}}$}                   & \textcolor{red}{\ding{55}}                                                 \\ \cline{1-1} \cline{4-9} 
\begin{tabular}[c]{@{}l@{}}Linear\\ FQE\end{tabular}             & \multicolumn{2}{c|}{\multirow{-2}{*}{{\color[HTML]{FE0000} \textbf{N/A}}}}                      & \multicolumn{2}{c|}{{\color[HTML]{FE0000} \textbf{N/A}}}                                                             & \multicolumn{1}{c|}{}                      &                                                                           & \multicolumn{1}{c|}{}                       &                                                                           \\ \cline{1-5}
Ours               & \multicolumn{1}{c|}{$\sqrt{\EE_{d_h^D}[{\wstar_h}^2]}$}        & \multicolumn{1}{c|}{\textcolor{green}{\ding{51}}}                              & \multicolumn{1}{c|}{$\underset{f\in\text{span}(\Fcal_h)}{\sup}\frac{|\EE_{d_h^\pi}[f]|}{\sqrt{\EE_{d_h^D}[f^2]}}$}          & \textcolor{green}{\ding{51}}                                                                       & \multicolumn{1}{c|}{\multirow{-2}{*}{
${\color{blue} {}^{(*)}}$  $\|\psi_h\|_{\Sigma_h^{-1}}$}} & \multirow{-2}{*}{\textcolor{green}{\ding{51}}}                                                     & \multicolumn{1}{c|}{\multirow{-2}{*}{${\color{blue} {}^{(*)}}$ $\|\EE_{d_h^\pi}[\phi]\|_{\Sigma_h^{-1}}$}} & \multirow{-2}{*}{\textcolor{green}{\ding{51}}}                                                     \\ \hline
\end{tabular}}
\end{table}
\section{Preliminaries}
\label{sec:prelim}


\paragraph{Finite-horizon MDPs} We consider an $H$-step finite-horizon MDP, where a random trajectory takes the form of 
$s_0, a_0, r_0, s_1, a_1, r_1, \ldots, s_H, a_H, r_H.$ 
W.l.o.g., $s_0, a_0$ is fixed and $r_0\equiv 0$, which will allow us to write our algorithm in a more concise manner.
We also assume that the state space is layered, i.e., $\Scal = \bigcup_{h=0}^H \Scal_h$, where $\{\Scal_h\}_{h=0}^H$ are disjoint sets and it always holds that $s_h \in \Scal_h$. The layered state space makes it convenient to use stationary notation for time-inhomogeneous objects. Let $\Acal$ be the action space, and we consider finite and discrete $\Scal$ and $\Acal$ although their cardinalities can be arbitrarily large. $R: \Scal\times\Acal\to\Delta([0, \Rmax])$ is the reward function, and $P: \Scal\times\Acal\to\Delta(\Scal)$ is the transition dynamics, where $P(\cdot|s, a)$ for $s\in\Scal_h$ and $a\in\Acal$ is always supported on $\Scal_{h+1}$. 

A policy $\pi: \Scal\to\Delta(\Acal)$ determines a distribution of trajectories starting from the fixed $s_0, a_0$ as $r_h \sim R(\cdot|s_h, a_h)$, $s_{h+1} \sim P(\cdot|s_h, a_h)$, $a_{h+1} \sim \pi(\cdot|s_{h+1})$, $\forall h\ge 0$. We use $\EE_{\pi}[\cdot]$ to denote the expectation under such a distribution, and $J(\pi):= \EE_{\pi}[\sum_{h=1}^H r_h]$ is the expected return. 
Note that the random return $\sum_{h=1}^H r_h \in [0, \Vmax]$ where $\Vmax:= H \Rmax$. The occupancy distribution $d_h^\pi$ is the marginal distribution of $(s_h, a_h)$ under $\pi$. The Q-function of $\pi$, $Q^\pi(s,a)$, is defined as: $\forall s\in\Scal_h, a\in\Acal$, $Q^\pi(s,a) = \EE_\pi[\sum_{h'=h}^H r_{h'} | s_h=s, a_h=a]$. We use $Q_h^\pi$ to denote $Q^\pi$ restricted to $\Scal_h\times\Acal$. It satisfies the Bellman equation: 
$Q_{h-1}^\pi = \Tcal^\pi Q_{h}^\pi$, where $\Tcal^\pi$ is the Bellman operator: given any $f:\Scal_h\times\Acal\to\RR$ and $s\in\Scal_{h-1}, a\in\Acal$, $(\Tcal^\pi f)(s,a) = \EE_{r \sim R(\cdot|s,a), s'\sim P(\cdot|s,a)} [r + f(s',\pi)]$. Here $f(s,\pi)$ is a shorthand for $\EE_{a\sim \pi(\cdot|s)}[f(s,a)]$ and $f(s',\pi)$ for $s'$ at time $H+1$ is always treated as $0$.  
Since $s_0, a_0$ is fixed, we have $Q^\pi(s_0, a_0) = J(\pi)$.

\paragraph{Off-Policy Evaluation (OPE)}~ In OPE, we want to estimate $J(\pi)$ for a given target policy $\pi$, using data trajectories collected with a behavior policy $\pi_D$. We assume the dataset consists of $n$ i.i.d.~trajectories, $\{(s_0^\ii, a_0^\ii, r_0^\ii, \ldots, s_H^\ii, a_H^\ii, r_H^\ii)\}_{i=1}^n$, and $d_h^D$ is the marginal distribution of $(s_h, a_h)$ in the data. $\EE_D[\cdot]$ denotes expectation w.r.t.~the data distribution. 

To handle large state and action spaces, we consider the standard setting of function approximation, where a function class $\Fcal = \{\Fcal_h\}_{h=1}^H$ is used  
to model $Q^\pi$; each $\Fcal_h$ consists of functions of the type $\Scal_h\times\Acal\to\RR$. 
Throughout the paper we assume the standard realizability assumption, that is:

\begin{assumption}[Realizability] \label{asm:realizable}
$Q^\pi \in \Fcal$, which is shorthand for $Q_h^\pi \in \Fcal_h$, $\forall h \in [H]$. 
\end{assumption}

Extension to handle misspecification is routine (see comment below Eq.\eqref{eq:Q2F}). Note that most standard analyses of OPE algorithms, especially those that use a general function class $\Fcal$, typically require the stronger \textit{Bellman-completeness} assumption \citep{antos2008learning, chen2019information}:
\begin{assumption}[Bellman Completeness]\label{asm:completeness}
$\Tcal^\pi f \in \Fcal_{h}, \forall f\in \Fcal_{h+1}, h\in[H]$. 
\end{assumption}

Our main guarantee (Theorem~\ref{thm:main}), however, \textbf{does not rely on such a stronger assumption}, and we discuss the implications of Bellman completeness in our results in Section~\ref{sec:weight}.

While our main results are derived for GFA, to develop intuitions and make connections with the literature, it will be useful to consider the special case of linear and tabular function approximation:

\begin{definition}[Linear function class] 
\label{def:linear}
$\Fcal$ is linear in a given feature map $\phi:\Scal\times\Acal\to\RR^\dm$, if $\Fcal=\Fcal_\phi = \{(s,a)\mapsto \phi(s,a)^\top \theta: \|\theta\|_2 \le \nmbd\}$ for some boundedness parameter $\nmbd > 0$.
\end{definition}

\begin{definition}[Tabular function class] \label{def:tabular}
$\Fcal$ is tabular if $\Fcal=\Fcal_\tab = [0, \Vmax]^{\Scal\times\Acal}$.
\end{definition}

\paragraph{Fixed-design analyses of linear regression} 
\label{sec:lr}
We briefly recall 
the \textit{fixed-design} analysis of linear regression (LR), which will be very helpful for understanding our algorithm and analyses, since they correspond to the special case of OPE with $H=1$ and linear $\Fcal$. Consider LR over a dataset $\{(x_i, y_i)\}_{i=1}^n$ where $x_i \in \RR^\dm$, and assume 
$\hatSigma:=\tfrac1n\sum_i x_i x_i^\top$ is invertible. LR gives $\emp{\theta} = \hatSigma^{-1} \tfrac{1}{n}\sum_i  x_i y_i$, and thus makes the prediction of $x_\test^\top \emp{\theta} = \tfrac{1}{n}\sum_i x_\test^\top \hatSigma^{-1} x_i y_i$ on a new $x_\test$. Standard fixed-design analysis provides a $O(\|x_\test\|_{\hatSigma^{-1}} \sqrt{\log(1/\delta)/n})$ high-probability error bound for this prediction, where $\|x_\test\|_{\hatSigma^{-1}} = \sqrt{x_\test^\top \hatSigma^{-1} x_\test}$ is the Mahalanobis norm. This bound is ``fixed-design'' as the concentration argument is w.r.t.~the randomness of $y_i|x_i$, where $\{x_i\}_{i=1}^n$ are treated as fixed instead of stochastic. 
It can be proved via a moment-matching formulation of LR, which is a special case of our Algorithm~\ref{alg:main}; see Appendix~\ref{app:covariate} for further detail.

\paragraph{Math notation} $[H]= \{1, 2, \ldots, H\}$. $\lesssim$ and $\gtrsim$ are inequalities that hold up to a positive absolute constant. We also use the big-Oh notation $O(\cdot)$ to highlight dependencies on selected variables. For a function $f\in\Fcal_h$, we use $f\evn\in\RR^n$ to denote its sample evaluation $\bigl(f(s_h^{(1)},a_h^{(1)}),\ldots,f(s_h^{(n)},a_h^{(n)})\bigr)$.

\begin{algorithm}[t]
\caption{\alg (\textbf{Q}-function-based \textbf{M}oment-\textbf{M}atched \textbf{R}eweighting)}\label{alg:main}
\begin{algorithmic}[1]
\Require Target policy $\pi$, data $\{(s_0^\ii, a_0^\ii, r_0^\ii, \ldots, s_H^\ii, a_H^\ii, r_H^\ii)\}_{i=1}^n$, function class $\Fcal = \{\Fcal_h\}_h$. 
\Ensure $\{\wh_h^\nn\}_{h}$, where $\wh_h^\nn := [\wh_h^\ii]_i \in \RR^n$; $\emp{J}(\pi) := \sum_{h=1}^H\frac{1}{n}\sum_{i=1}^n \wh_h^\ii r_h^\ii$.
\State Initialize $
\wh_0^\ii \equiv 1$, $\forall i$. 
\For{$h=1$ to $H$}
    \State Choose $\wh_h^\nn$ as the ${\color{red} w_h^\nn}$ that (approximately) minimizes the loss
    \begin{align} \label{eq:loss-match}
    \lossmatch_h({\color{red} w_h^\nn}; \wh_{h-1}^\nn) := \sup_{f\in\Fcal_h}  \left| \frac{1}{n}\sum_{i=1}^n {\color{red} w_h^\ii} f(s_{h}^\ii, a_{h}^\ii) - \frac{1}{n}\sum_{i=1}^n \wh_{h-1}^\ii f(s_{h}^\ii, \pi) \right|.
    \end{align}
    \quad ~ (Optional) Among $\wh_h^\nn$ with the same loss, minimize $\secmom{\wh_h^\nn} := \sqrt{\frac{1}{n}\sum_i (\wh_h^\ii)^2}$.
\EndFor
\end{algorithmic}
\end{algorithm}

\section{Algorithm and Analyses}
\label{sec:method}

Our algorithm, \alg, learns a scalar weight $\wh_h^\ii$ for each state-action pair $(s_h^\ii, a_h^\ii)$ in the data, and forms the estimate of $J(\pi)$ using the reweighted reward: $\emp{J}(\pi) := \frac{1}{n}\sum_{i=1}^n\sum_{h=1}^H \wh_h^\ii r_h^\ii$. In fact, the estimate will be unbiased if $\wh_h^\ii$ is set to the following expression, as well understood from the literature on marginalized importance sampling (MIS) \citep{liu2018breaking, xie2019towards}:
\begin{align} \label{eq:mis}
\textrm{(MIS Ratio)} \qquad    d_h^\pi(s_h^\ii, a_h^\ii)/d_h^D(s_h^\ii, a_h^\ii).
\end{align}
However, learning the MIS ratio in Eq.\eqref{eq:mis} requires strong assumptions and is sufficient but not necessary for accurate OPE, and our algorithm enjoys guarantees when $\wh_h^\ii$ does not learn this ratio but something more nuanced (see  Section~\ref{sec:weight}). Nevertheless, to develop the first intuition, it is useful to temporarily \textit{\textbf{pretend}} that our algorithm learns such weights inductively in a top-down manner.

\paragraph{Algorithm Intuition}~ We now provide a non-rigorous (but relatively easy to understand) explanation of our Algorithm~\ref{alg:main}, \alg, which computes $\wh_h^\nn=\{\wh_h^\ii\}_{i=1}^n$ level-by-level from $h=0$ to $H$. We do not represent such weights in a parameterized manner as standard MIS methods \citep{liu2018breaking, uehara2019minimax}, but rather learn a separate scalar weight for each of the $n$ data points. 
Since $(s_0, a_0)$ is the fixed dummy initial state-action pair, there is no distribution shift at $h=0$ and we set $w_0^\ii = \frac{d_0^\pi(s_0, a_0)}{d_0^D(s_0, a_0)} \equiv 1$. 

Inductively, suppose we already have $\wh_{h-1}^\ii \approx \frac{d_{h-1}^\pi(s_{h-1}^\ii, a_{h-1}^\ii)}{d_{h-1}^D(s_{h-1}^\ii, a_{h-1}^\ii)}$, and we want to compute $\wh_{h}^\ii \approx \frac{d_{h}^\pi(s_{h}^\ii, a_{h}^\ii)}{d_{h}^D(s_{h}^\ii, a_{h}^\ii)}$. (Again, this is generally not the case in our analysis.) By the principle of importance sampling,\footnote{That is, given distributions $p$ and $q$ over $\Xcal$ and $f:\Xcal\to\RR$, we have $\EE_p[f] = \EE_q[p/q \cdot f]$ when $p/q<\infty$.} $\wh_{h-1}^\nn =\{\wh_{h-1}^\ii\}_{i=1}^n$ reweights the data distribution at level $h-1$ from $d_{h-1}^D$ to $d_{h-1}^\pi$. As a consequence, the following (reweighted) distributions over $(s_h, a_h)$, as characterized by their generative processes, produce the same expectations:
\begin{align*}
&~ s_h \sim P(\cdot|s_{h-1}, a_{h-1}),~ a_h \sim \pi(\cdot|s_h), ~ (s_{h-1}, a_{h-1})\sim {\color{red} d_{h-1}^D}, ~ \textrm{reweighted by} ~ \tfrac{{\color{blue}d_{h-1}^\pi(s_{h-1}, a_{h-1})}}{{\color{red}d_{h-1}^D(s_{h-1}, a_{h-1})}} \\
\Leftrightarrow &~ s_h \sim P(\cdot|s_{h-1}, a_{h-1}),~ a_h \sim \pi(\cdot|s_h),~ (s_{h-1}, a_{h-1})\sim {\color{blue} d_{h-1}^\pi}
~ \Leftrightarrow ~ (s_h, a_h) \sim d_h^\pi.
\end{align*}
This way, we successfully produce a weighted distribution that ``looks like'' $d_h^\pi$---by applying the weight $\wh_{h-1}^\ii$ at level $h-1$ to the state-action pair $(s_h^\ii, \pi)$ at level $h$. 
To complete the induction, we will need  $\wh_h^\ii$ that reweight $(s_h^\ii, a_h^\ii)$, but now we only have weights that reweight $(s_h^\ii, \pi)$. To bridge this gap, Eq.\eqref{eq:loss-match} employs a moment-matching objective based on the Integral Probability Metric (IPM):\footnote{The IPM between $p$ and $q$ over function class $\Fcal$ is $\sup_{f\in\Fcal} |\EE_p[f] - \EE_q[f]|$. Our objective is close to an IPM, except that $p$ and $q$ for us are weighted empirical measures that do not necessarily correspond to valid distributions.} 
we find $w_h^\ii$ such that $\wh_h^\nn$-reweighted $\{(s_h^\ii, a_h^\ii)\}_{i=1}^n$ is indistinguishable from $\wh_{h-1}^\nn$-reweighted $\{(s_h^\ii, \pi)\}_{i=1}^n$, up to the discrimination power of $\Fcal_h$.

\subsection{Finite-Sample Guarantee} \label{sec:finite-sample}
We now present the main analyses for \alg, which face two key challenges:
\begin{enumerate}[leftmargin=*]
\item While the earlier intuition appeals to learning the MIS ratio, this is generally not guaranteed when we only assume realizability (Assumption~\ref{asm:realizable}). Under such a weak assumption, we need to bound the OPE error without showing that $\wh_h^\nn$ tracks the MIS ratio.
\item The number of free parameters ($\{\wh_h^\nn\}_h$) scales linearly in $n$, the number of data points. Under the na\"ive parameter counting argument, this usually results in overfitting, i.e., an error bound that does not shrink as $n$ grows. Such a practice of assigning a separate scalar to each data point is found in distributionally robust optimization \citep{rahimian2019distributionally} and balancing weights \citep{kallus2018balanced,hirshberg2021augmented}, and its analyses often require sophisticated concentration machinery.
\end{enumerate}

We address both challenges by a starkly simple and elementary analysis, whose core proof takes but a few lines. Perhaps surprisingly, our analysis does \textbf{not} perform any union bound over the function class $\Fcal$ at all, which avoids the statistical complexity of $\Fcal$ (e.g., $\log|\Fcal|$ for finite $\Fcal$). 

The analysis is given as follows, where we start with the OPE error:
{\allowdisplaybreaks
\begin{align}
&~ \left| \textstyle \sum_{h=1}^H \frac{1}{n} \sum_{i=1}^n \wh_h^\ii r_h^\ii - J(\pi) \right| = \left| \textstyle \sum_{h=0}^H \frac{1}{n} \sum_{i=1}^n \wh_h^\ii r_h^\ii - J(\pi) \right| \tag{$r_0\equiv 0$} \nonumber \\
\le  &~ \left| \textstyle \sum_{h=0}^H \frac{1}{n} \sum_{i=1}^n {\color{blue} \bm{\wh_h^\ii}} \left(Q^\pi(s_h^\ii, a_h^\ii) - {\color{blue} \bm{Q^\pi(s_{h+1}^\ii, \pi)} }\right) - J(\pi) \right| + \textstyle \sum_{h=0}^H \epsstat_h \label{eq:before-match} \\
\le &~ \left| \textstyle \sum_{h=0}^H \frac{1}{n} \sum_{i=1}^n \left(\wh_h^\ii Q^\pi(s_h^\ii, a_h^\ii) - {\color{red} \bm{ \wh_{h+1}^\ii Q^\pi(s_{h+1}^\ii, a_{h+1}^\ii)}}\right) - J(\pi) \right| \label{eq:after-match}  \\
& \quad + \textstyle \sum_{h=1}^H \lossmatch_h(\wh_h^\nn; \wh_{h-1}^\nn) + \sum_{h=0}^H \epsstat_h \nonumber \\
= &~\textstyle \sum_{h=1}^H \lossmatch_h(\wh_h^\nn; \wh_{h-1}^\nn) + \sum_{h=0}^H \epsstat_h \nonumber.
\end{align}}
We explain the derivation step by step:
\begin{enumerate}[leftmargin=*]
\item First, we replace $r_h^\ii$ with $Q^\pi(s_h^\ii, a_h^\ii) -Q^\pi(s_{h+1}^\ii, \pi)$ by paying for an error term $\epsstat_h$, where
$$
\left|\textstyle \frac{1}{n}\sum_{i=1}^n \wh_h^\ii \left( Q^\pi(s_h^\ii, a_h^\ii) - r_h^\ii - Q^\pi(s_{h+1}^\ii, \pi)\right)\right| \le \epsstat_h \propto \secmom{\wh_h^\nn} := \sqrt{\frac{1}{n}\sum_{i} (\wh_h^\ii)^2}
$$
holds with high probability. 
This follows from a standard concentration argument, which treats $\{(s_h^\ii, a_h^\ii)\}_{i=1}^n$ as fixed and is w.r.t.~the randomness of $r_h^\ii$ and $s_{h+1}^\ii$ conditioned on $(s_h^\ii, a_h^\ii)$. As long as $\wh_h^\nn$ are chosen only based on the trajectory prefixes up to $\{(s_h^\ii, a_h^\ii)\}_{i=1}^n$ and do not depend on the randomness after time $h$, the independence across the terms is intact, though the range of each term will be amplified by the size of $\wh_h^\ii$. 
\item From \eqref{eq:before-match} to \eqref{eq:after-match}, we replace ${\color{blue} \bm{\wh_h^\ii Q^\pi(s_{h+1}^\ii, \pi)}}$ with ${\color{red} \bm{ \wh_{h+1}^\ii Q^\pi(s_{h+1}^\ii, a_{h+1}^\ii)}}$. Their difference is almost exactly the matching loss used in our algorithm, except that we do not know $Q_{h+1}^\pi$ in the algorithm and instead relax it to $\sup_{f\in\Fcal_{h+1}}$ using realizability (Assumption~\ref{asm:realizable}): 
\begin{align} \label{eq:Q2F}
&~ \left|\frac{1}{n}\sum_{i=1}^n{\color{red} \bm{ \wh_{h+1}^\ii Q^\pi(s_{h+1}^\ii, a_{h+1}^\ii)}} - \frac{1}{n}\sum_{i=1}^n{\color{blue} \bm{\wh_h^\ii Q^\pi(s_{h+1}^\ii, \pi)}} \right| \\
\le&~ \sup_{f\in\Fcal_{h+1}} \left| \frac{1}{n}\sum_{i=1}^n \wh_{h+1}^\ii f(s_{h+1}^\ii, a_{h+1}^\ii) - \frac{1}{n}\sum_{i=1}^n \wh_{h}^\ii f(s_{h+1}^\ii, \pi) \right|  = \lossmatch_{h+1}(\wh_{h+1}^\nn; \wh_{h}^\nn).   \nonumber
\end{align}
If realizability is not satisfied, an additional misspecification term will occur in the above inequality: 
$\inf_{f\in\Fcal_{h+1}}\left|\frac1n\sum_{i=1}^n\left[\wh_{h+1}^\ii\cdot (Q^\pi-f)(s_{h+1}^\ii,a_{h+1^\ii})-\wh_h^\ii\cdot(Q^\pi-f)(s_{h+1}^\ii,\pi)\right]\right|$. 
In addition, the realizability of $\Fcal_{h+1}$ can be relaxed to that of the convex hull of $\Fcal_{h+1}$ as the matching loss is convex in $f$ \citep{uehara2019minimax}.
\item Finally, the first line of Eq.\eqref{eq:after-match} vanishes due to telescoping, as the negative term in the $h$-th term exactly cancels the positive term in the $(h+1)$-th term, leaving only $\wh_0^\ii Q^\pi(s_0, a_0) = J(\pi)$. 
\end{enumerate}
This immediately leads to the following guarantee of \alg; see proof in  Appendix~\ref{app:proof-main}.
\begin{theorem} \label{thm:main}
For any choice of weight vectors $\{\wh_h^\nn\}_h$ in Algorithm~\ref{alg:main} where $\wh_h^\nn$ only depends on $\{(s_0^\ii, a_0^\ii, r_0^\ii, \ldots, s_h^\ii, a_h^\ii)\}$, under Assumption~\ref{asm:realizable}, with probability at least $1-\delta$, we have 
$$
\left|\emp{J}(\pi) - J(\pi)\right| \le  \textstyle \sum_{h=1}^H \lossmatch_h(\wh_h^\nn; \wh_{h-1}^\nn) + \sum_{h=0}^H \epsstat_h,
$$
where $\epsstat_h = \|\wh_h^\nn\|_\nn\cdot V_{\max}\sqrt{\frac{2\log(2(H+1)/\delta)}{n}}$. 
\end{theorem}

\paragraph{Dimension-freeness and ``WYSIWYG''} Standard OPE analyses in offline RL under general function approximation need to translate the empirical loss to its  population counterpart, during which one pays for the statistical dimension of the function class. 
In contrast, our analysis only performs concentration on $Q^\pi$ (Eq.\eqref{eq:before-match}), and no concentration on $\Fcal$ is needed, resulting in a dimension-free bound. 
In addition, the empirical matching loss $\lossmatch$ observed in the algorithm (Eq.\eqref{eq:loss-match}) shows up in the bound, which makes Theorem~\ref{thm:main} ``what you see is what you get (WYSIWYG)''. One downside is that it can be difficult to tell when the matching loss will be large vs.~small; as we will see below, in special cases like linear $\Fcal$, there are mild sufficient conditions (i.e., invertible sample covariance matrix) which guarantee that \textit{$0$ matching loss is achievable} (Propositions~\ref{prop:H=1} and \ref{prop:linear}). 

\paragraph{Uncertainty quantification}~ In addition to the empirical matching loss, the other term in the bound, $\epsstat_h$, depends on the 2nd moment of $\wh_h^\nn$, which is also available to the learner. Therefore, the bound in Theorem~\ref{thm:main} is directly computable from the data, 
which means that our \alg algorithm automatically comes with uncertainty quantification. Similar results can be found in existing analyses of linear OPE, such as \citet[Theorem 4]{duan2020minimax}, and our bound in Theorem~\ref{thm:main} applies more generally to GFA and improves over \citet{duan2020minimax} in several aspects; see Appendix~\ref{app:prior-linear}. 

\paragraph{``Fixed-design'' analysis}~ Theorem~\ref{thm:main} provides a data-dependent bound, which is an unusual form of guarantee in the OPE literature. By considering the following special case, however, we see that Theorem~\ref{thm:main} is the natural generalization of the standard fixed-design analysis in linear regression. 
\begin{proposition} \label{prop:H=1}
Consider $H=1$ and linear $\Fcal$ (see Definition~\ref{def:linear}). Let $x_i=\phi(s_1^\ii,a_1^\ii)$. For invertible $\hatSigma = \frac{1}{n}\sum_{i=1}^n x_i x_i^\top$, with $\wh_1^\ii = \xmean^\top \hatSigma^{-1} x_i$ where $\xmean = \frac{1}{n}\sum_{i=1}^n\phi(s_1^\ii,\pi)$, we have
\begin{center}
(i) $\lossmatch_1(\wh_1^\nn; \wh_0^\nn) = 0$, \qquad and \qquad (ii)  $\secmom{\wh_1^\nn} = \|\xmean\|_{\hatSigma^{-1}}$.  \vspace*{-.5em}
\end{center}
The bound in Theorem~\ref{thm:main} becomes $|\emp{J}(\pi) - J(\pi)| \le \epsstat_0 + \epsstat_1$, where $\epsstat_1 \propto \secmom{\wh_1^\nn} = \|\xmean\|_{\hatSigma^{-1}}$. When $\hatSigma$ is non-invertible, the minimizer is $\wh_1^\ii = \xmean^\top \hatSigma^{\dagger} x_i$ where $(\cdot)^\dagger$ is pseudo-inverse. 
\end{proposition}
In other words, when $H=1$ and $\Fcal$ is linear, our algorithm and guarantee \textbf{exactly recover} the moment-matching form of linear regression (LR) and its fixed-design analysis (see Section~\ref{sec:lr} and Appendix~\ref{app:covariate}):   
the error bound of LR depends on $\|\xmean\|_{\hatSigma^{-1}}$, which is data dependent due to $\hatSigma$ and coincides with our $\secmom{\wh_1^\nn}$ in the $\epsstat_1$ term. Our $\epsstat_0$ term is a small overhead that corresponds to how $\xmean$ concentrates to its expectation along the direction of the true linear coefficients under the randomness of $s_1$; see Footnote~\ref{ft:initial} in Appendix~\ref{app:covariate} for further details. 

\paragraph{Non-uniqueness of $\wh$ and the principle of least 2nd moment} 
So far, Theorem~\ref{thm:main} holds for any choice of $\{\wh_h^\nn\}_h$, including those that do not minimize Eq.\eqref{eq:loss-match}. Moreover, even if we  require $\wh_h^\nn$ to exactly minimize $\lossmatch_h(\cdot; \wh_{h-1}^\nn)$, the solution is also generally not unique. 
As it turns out, this flexibility is a blessing which allows us to unify and make connections to several existing methods. That said, it is still desirable that we fix a canonical choice of $\wh_h^\nn$ to have a fully-specified algorithm. Among all $\wh_h^\nn$ that achieves the minimal loss, it is naturally preferred to choose the one with the least 2nd moment $\secmom{\wh_h^\nn}$ since this quantity appears in the bound through $\epsstat_h$. We call this \textbf{the principle of least 2nd moment}. In fact, in Proposition~\ref{prop:H=1}, the choice of weights that recovers linear regression is exactly the minimizer of the 2nd moment. 

In practice, there can be situations where even the 2nd moment-minimizing $\wh_h^\nn$ is too large, causing a poorly behaved $\epsstat_h$ term. This can happen, for example, when $\hatSigma$ is near singular in the linear setting. In such cases, one needs to balance the bias-variance trade-off, by potentially giving up some matching loss in exchange for a smaller $\secmom{\wh_h^\nn}$ (and thus $\epsstat_h$). This can be implemented via constraint optimization (e.g., minimizing $\lossmatch$ up to some budget for $\secmom{\wh_h^\nn}$), or via a Lagrangian form that corresponds to ridge regression in the linear setting (see Appendix~\ref{app:computation}). In the rest of the paper, however, we will not consider constraint or regularization on $\wh_h^\nn$ to keep the analyses clean. 

\subsection{Connection to Existing Methods}
\label{sec:connection}

We discuss the relationship between \alg and several existing methods to establish connections and help provide further insights. Due to space limit, we defer the connection with Monte-Carlo methods to Appendix~\ref{app:mc}, and focus on the connection with FQE. 



\subsubsection{Linear and Tabular FQE}
As mentioned in Section~\ref{sec:intro}, the canonical OPE method under value-function approximation is FQE, which fits $Q_h^\pi$ in \textit{parametric} forms recursively from \textit{bottom up} via a series of squared-loss regressions. While our \alg is a \textit{top-down} procedure that learns \textit{non-parametric} weights, it coincides with FQE in the linear and tabular setting (Definitions~\ref{def:linear} and \ref{def:tabular}), which should not be surprising given the connection between \alg and linear regression when $H=1$ (Proposition~\ref{prop:H=1}). 

\begin{proposition} \label{prop:linear}
(i) In the linear setting (Definition~\ref{def:linear}), assume $\hatSigma_h = \tfrac{1}{n} \sum_i \phi(s_h^\ii, a_h^\ii)\phi(s_h^\ii, a_h^\ii)^\top$ is invertible for all $h \ge 1$. Consider $\alg$ where in each step $h$ we choose the least-2nd-moment weight $\wh^\nn$ that sets $\lossmatch_h(\wh_h^\nn; \wh_{h-1}^\nn)=0$. The final prediction of $\emp{J}(\pi)$ is identical to that given by linear FQE \citep{duan2020minimax} without ridge regression. \\ 
(ii) As a further special case in the tabular setting (Definition~\ref{def:tabular}), assume every state-action pair appears in the dataset, \alg and FQE both coincide with the model-based (certainty-equivalence) solution \citep{yin2021towards}, and $\wh_h^\ii = \emp{d_h^\pi}(s_h^\ii, a_h^\ii) / \emp{d_h^D}(s_h^\ii, a_h^\ii)$ where $\emp{d_h^D}$ is the empirical frequency of data and $\emp{d_h^\pi}$ is the occupancy of $\pi$ in the empirical MDP.
\end{proposition}
More detailed definitions (e.g., the FQE procedure) and the proof can be found in Appendix~\ref{app:prop-linear}. 
In the linear and tabular settings, prior works \citep{duan2020minimax, yin2021towards} have established fine-grained analyses where the leading $O(1/\sqrt{n})$ terms are dimension-free, which separates them from the $\log|\Fcal|$-dependent guarantees for GFA and no unification is known to this date. \textbf{Our results fill this gap in theoretical understanding, as these fine-grained results are \textit{recovered} (and sometimes \textit{improved}) when we specialize our Theorem~\ref{thm:main} to the linear/tabular settings.} We provide a detailed discussion and comparison w.r.t.~these prior results in Appendix~\ref{app:prior-linear}. 

Proposition~\ref{prop:linear} also provides intuition for $\lossmatch_h$ when it takes non-zero values. In the tabular case, $\lossmatch_h$ may not be minimized to $0$ if some $(s,a)\in\Scal_h \times\Acal$ does not appear in the data at all, and the corresponding mass in $\emp{d_h^\pi}$ will result in a non-zero matching loss.

\subsection{Computational Efficiency}
So far our analyses are only concerned with statistical efficiency. The computational efficiency of the method boils down to whether we can efficiently minimize $\lossmatch_h(\,\cdot\,; \wh_{h-1}^\nn)$ (Eq.~\eqref{eq:loss-match}), which is a standard IPM-type objective. 
As already discussed in Section~\ref{sec:connection}, this is easy to optimize when $\Fcal$ is linear as it simply amounts to linear regression. For general function approximation, given the minimax form ($\argmin_{w_h^\nn} \sup_{f\in\Fcal}$), we can efficiently approximate the solution via the no-regret + best-response framework \citep{freund1999adaptive} given standard optimization oracles over $\Fcal$; see further discussion in Appendix~\ref{app:computation}. 
\section{Understanding the Learned Weights}
\label{sec:weight}

The results thus far leave a few important and interrelated questions open:

\textbf{Boundedness of weights.}~ Theorem~\ref{thm:main} only assumes realizability and not Bellman completeness. In this regime, classic counterexamples against TD imply that linear FQE may diverge in the worst case, even if data covers all feature directions \citep{tsitsiklis1996feature}. Since our \alg coincides with FQE in the linear setting, this means that the size of $\wh_{h}^\nn$ may blow-up exponentially in the horizon. Thus, naturally we want to understand when the size of $\wh_h^\nn$ stays well bounded, and how stronger conditions such as Bellman completeness help guarantee that.

\textbf{A priori bound.}~  Theorem~\ref{thm:main} is data dependent. Can we turn it into an a priori, data-independent bound? Recall that the bound depends on the size of $\wh_h^\nn$, $\secmom{\wh_h^\nn}$. To remove the data dependence, we can define the population counterpart of $\wh_h^\nn$ and use its size to replace $\secmom{\wh_h^\nn}$. Note that such bounds in general will incur dimension dependence, albeit in higher-order terms \citep{hsu2011analysis}. 


\textbf{The nature of coverage.}~ $\secmom{\wh_h^\nn}$ plays the role of \textit{coverage} in Theorem~\ref{thm:main}. 
Conventional wisdom tells us that coverage is about how the distribution $d_h^D$ covers $d_h^\pi$, and the density coverage can sometimes be refined based on structure of $\Fcal$; for example, in linear case, data only needs to cover the mean feature direction $\EE_{(s_h, a_h)\sim d_h^\pi}[\phi(s_h, a_h)]$. 
However, recent fine-grained linear analyses tell a very different story: \citet{amortila2026unifying} (see also \citet{duan2020minimax, perdomo2023complete}) show that the target of coverage is generally \textit{not} $d_h^\pi$ or its feature expectation, but the vector defined by the following linear dynamical system: starting with $\ft_1 = \EE_D[\phi(s_1, \pi)]$, and define, $\forall h\ge 1$, 
\begin{align} \label{eq:lds}
\ft_{h+1} = \Bpi_h \ft_h, ~~\textup{where}~~ \Bpi_h = (\Sigmacr_h)^\top\Sigma_h^{-1}, ~ \Sigma_h = \EE_{D}[\phi_h\phi_h^\top], ~ \Sigmacr_h = \EE_D[\phi_h(\phi_{h+1}^\pi)^\top].
\end{align}

Here $\phi_h$ is the shorthand of $\phi(s_h, a_h)$ and $\phi_{h+1}^\pi = \phi(s_{h+1}, \pi)$. In fact, linear FQE can be thought of as fitting such a system and using it to predict $J(\pi)$ \citep{duan2020minimax}. While this system mimics the feature evolution in the real MDP, they are generally different (i.e., $\ft_h \ne \EE_{d_h^\pi}[\phi]$). In the fine-grained analyses of linear OPE, it is shown that what really needs to be covered is the $\ft_h$ from such a ``wrong'' system $\{\Bpi_h\}_h$, \textbf{not} the feature expectation $\EE_{d_h^\pi}[\phi]$ from the real MDP,\footnote{There is information-theoretic lower bound against this possibility \citep{foster2022offline}.} and they only coincide under Bellman completeness. However, the analyses of FQE under general function approximation rely on completeness, and this subtle distinction between $\ft_h$ and $\EE_{d_h^\pi}[\phi]$ is therefore not observed. Moreover, the key concepts such as $\Bpi_h$ are defined linear-algebraically, and it is unclear how their GFA counterpart should be defined. Does $\secmom{\wh_h^\nn}$ in our analyses (or its population counterpart) provide such a generalization, and how is it related to the concepts identified in the linear analyses?

We address the above questions by providing a concentration analysis for how $\wh_h^\nn$ tracks its population counterpart, defined below; throughout this section, we assume the empirical algorithm performs exact minimization of $\lossmatch_h$ with least 2nd-moment $\wh_h^\nn$. 
\begin{definition}\label{def:pop}
In Algorithm~\ref{alg:main}, replace  $\wh_h^\nn \in \RR^n$ with $\wstar_h \in \RR^{\Scal\times\Acal}$ ($\wstar_0\equiv 1$), and Eq.\eqref{eq:loss-match} with 
\begin{align} \label{eq:poploss}
\poploss_h(w_h; \wstar_{h-1}) := {\textstyle \sup_{f\in\Fcal_h}}  \left| \EE_D[w_h(s_h, a_h)  f(s_{h}, a_{h})] - \EE_D[\wstar_{h-1}(s_{h-1}, a_{h-1}) f(s_{h}, \pi)] \right|.
\end{align}
We assume the exact minimization of $\poploss$. Among the minimizers, we choose the one that minimizes 
$\|w_h\|_{2, d_h^D} := (\EE_{d_h^D}[w_h^2])^{1/2}$. The output $\{\wstar_h\}_h$ is considered the population counterpart of $\{\wh_h^\nn\}_h$.
\end{definition}
%
%
%

To characterize the least-2nd-moment solution $\wh_h^\nn$ and its population counterpart $\wstar_h$, we will use a Hilbert-space viewpoint based on the linear span of $\Fcal_h$.

\begin{definition}\label{def:hilbert}
For each $h\in[H]$, define \mbox{$\mathrm{span}(\Fcal_h)=\{\sum_{m=1}^M a_mf_m:M<\infty,a_m\in\RR,f_m\in\Fcal_h\}$.} Let $\Hilb_h:=\overline{\mathrm{span}(\Fcal_h)}$ be the closure of $\mathrm{span}(\Fcal_h)$ under $\|f\|_{2,d_h^D}=(\EE_{(s,a)\sim d_h^D}[f(s,a)^2])^{1/2}$, the $d_h^D$-weighted (semi-)norm. We also define $\langle f,g\rangle_{d_h^D}=\EE_{(s,a)\sim d_h^D}[f(s,a)g(s,a)]$. 
\end{definition}

\begin{proposition}[2nd-Moment Minimizers are in $\Hilb$] \label{prop:w-in-span}
The population least-2nd-moment weight $\wstar_h\in\Hilb_h$. In addition, for the least-2nd-moment $\wh_h^\nn$, there exists $\wb_h\in\Hilb_h$ such that $\wb_h\evn=\wh_h^\nn$, where $\evn$ denotes the sample evaluation with $(\wb_h\evn)_i=\wb_h(s_h^\ii,a_h^\ii)$, $\forall i\in[n]$.
\end{proposition}

\subsection{Tracking Error of $\wh_h^\nn$ and Data-independent Guarantee}
We now establish finite-sample guarantees for how $\wh_h^\nn$ tracks $\wstar_h$ under the following assumption.

\begin{assumption}[Bounded generalized leverage]\label{asm:leverage}
$\forall h\in[H]$, there exists $\kappa_h \in (0,  \infty)$ such that \vspace*{-.1em}
\[
|f(s,a)|\le\sqrt{\kappa_h}\|f\|_{2,d_h^D},\quad \forall f\in\Hilb_h,\quad\forall (s,a)\in\Scal_h\times\Acal. \vspace*{-.2em}
\]
\end{assumption}

This assumption generalizes the standard  condition of \emph{bounded statistical leverage} in the linear setting $\|\Sigma_h^{-1/2}\phi(s_h,a_h)\|_2\le\sqrt{\kappa_h}, \forall (s_h,a_h)$, which is often assumed in random-design analyses \citep{hsu2011analysis}. 
It also implies that $\|\cdot\|_{2,d_h^D}$ is identifiable on $\Hilb_h$, in the sense that $\|f\|_{2,d_h^D}=0$ implies $f=0$ on $\Scal_h\times\Acal$. This immediately leads to the following implication:

\begin{proposition}[Population 0 matching loss]\label{prop:population-exact}
For each $h\in[H]$, under Assumption~\ref{asm:leverage}, the population matching problem in Definition~\ref{def:pop} achieves $0$ loss, i.e., $\Lcal_h(\wstar_h;\wstar_{h-1})=0$.
\end{proposition}

As mentioned earlier, FQE without completeness may produce run-away predictions even when data provides good coverage, so we need some additional quantities (weaker than completeness assumption) to characterize and prevent the exponential blow-up of weights. In the linear setting, prior works have identified the operator norm of $\Bpi_h$ (Eq.\eqref{eq:lds}) as the key quantity that controls the blow-up \citep{perdomo2023complete, amortila2026unifying}. Below we define the counterpart of these concepts for \alg under GFA:
\begin{definition}\label{def:backup}
For each $h\in[H-1]$, define the projected transition operator $\bp_h:\Hilb_{h+1}\to\Hilb_h$, such that $\forall f\in\Hilb_{h+1}$, $\bp_h f\in\Hilb_h$ is the unique minimizer of the least square problem
\[
\bp_h f:= {\textstyle \arg\min_{g\in\Hilb_h}}\
\EE_D \bigl[(g(s_h,a_h)-f(s_{h+1},\pi))^2\bigr].
\]
Define $\rho_{t:h}=\sup_{f\in\Hilb_h:\|f\|_\infty\le1}\|\bp_t\bp_{t+1}\cdots\bp_{h-1} f\|_\infty$ as the multi-step $\infty$-operator norm of $\bp_t\bp_{t+1}\cdots\bp_{h-1}:\Hilb_h\to\Hilb_t$ , where $\|g\|_\infty=\sup_{(s,a)\in\Scal_j\times\Acal}|g(s,a)|$ for $g\in\Hilb_j$. Let $\rho_{h:h}=1$.
\end{definition}

\begin{proposition}[Connection between $\bp_h$ and $\Bpi_h$]\label{prop:connection-linear}
In the linear setting, for $f\in\Hilb_{h+1}$, let $\theta_{h+1}$ be its linear coefficient (i.e.,  $f(s,a)=\phi(s,a)^\top\theta_{h+1}$), and let $\theta_h$ be that of $\bp_h f\in\Hilb_h$. 
We have $\theta_h=(B_h^\pi)^\top\theta_{h+1}$, 
and $
\rho_{t:h}\le\sup_{s_t,a_t}\|\Sigma_{h}^{-1/2}B_{h-1}^\pi\cdots B_{t+1}^\pi B_t^\pi\phi(s_t,a_t)\|_2$.
\end{proposition}


Just as the operator norm of $\Bpi_h$ controls the size of $\psi_h = \Bpi_{h-1} \cdots \Bpi_1 \psi_1$ in the linear case \citep{amortila2026unifying}, we have $\rho_{1:h}$, the operator norm of $\bp_1 \cdots \bp_{h-1}$, controls $\EE_{d_h^D}[\wstar_h f]$.

\begin{proposition}\label{prop:bound-weight}
Under Assumption~\ref{asm:leverage}, $|\EE_{d_h^D}[\wstar_h(s_h,a_h)f(s_h,a_h)]|\le\rho_{1:h}\|f\|_\infty, \forall h\in[H], f\in\Hilb_h$. 
\end{proposition}
Note that $\wstar_h$ may have a large size for two reasons: (1) data $d_h^D$ does not provide sufficient coverage (for linear, this corresponds to near-singular $\Sigma_h$), and (2) the algorithm's iterations are \textit{unstable} and produce run-away predictions even when data has good coverage (for linear, this corresponds to large $\|\psi_h\|$). Proposition~\ref{prop:bound-weight} separates out these two factors and only studies the instability in (2) by examining $\wstar_h \cdot d_h^D$. 
Such instability can be eliminated with Bellman completeness, and Proposition~\ref{prop:bound-weight} provides more quantitative understanding 
and is strictly weaker than completeness (see Section~\ref{sec:complete}). 


We are now ready to give the tracking error guarantee, where the bound depends on $\rho_{(\cdot):h}$ (iteration stability), the generalized leverage,  the size of $\wstar_{(\cdot)}$ (which is determined by both iteration stability and the data), and the log-covering number of $\Hilb_{(\cdot)}$ which is the statistical dimension term. 

\begin{theorem}[Empirical weight tracks its population counterpart]\label{thm:tracking} Define the empirical-population tracking error by
$\bigl\|\wh_h^\nn-\wstar_h\evn\bigr\|_\nn:=\bigl(\frac1n\sum_{i=1}^n(\wh_h^\ii-\wstar_h(s_h^\ii,a_h^\ii))^2\bigr)^{1/2}$.
Under Assumptions~\ref{asm:leverage}, fix any $h\in[H]$ and $\delta\in(0,1)$, and let $N_t$ be the $1/8$-covering number of $\{u\in\Hilb_t:\|u\|_{2,d_t^D}\le 1\}$. If $n\gtrsim(\max_{t\le h}\kappa_t)^2(\sum_{t=1}^h\rho_{t:h})^2\log(H^2N_tN_h/\delta)$ for all $t\le h$, then with probability $\ge1-\delta$,
\[
\left\|\wh_h^\nn-w_h^\star\evn\right\|_\nn\lesssim\sum_{t=1}^h\rho_{t:h}\sqrt{\frac{\kappa_h^2\log(H^2N_tN_h/\delta)}{n}} (\|w_{t}^\star\|_{2,d_t^D} + 1). 
\]
\end{theorem}



This result also allows us to obtain a data-independent, ``random-design'' version of 
Theorem~\ref{thm:main}, where we replace $\secmom{\wh_h^\nn}$ with $\|w_h^\star\|_{2,d_h^D}$ and incur a higher-order term  that corresponds to Theorem~\ref{thm:tracking}; such bounds are only dimension-free in the leading $O(1/\sqrt{n})$ term, but incur dimension dependence in the higher-order terms, which is consistent with the prior specialized analyses \citep{hsu2011analysis, duan2020minimax}. 
For the $\lossmatch_h$ term which is also data-dependent, we can show that it becomes $0$: 


\begin{proposition} [Empirical 0 matching loss]\label{prop:empirical-exact}
Under Assumption~\ref{asm:leverage}, if $n\gtrsim\kappa_h\log(HN_h/\delta)$ for all $h\in[H]$, then \alg achieves 0 matching loss w.h.p, i.e., $\lossmatch_h(\wh_h^\nn;\wh_{h-1}^\nn)=0$.
\end{proposition}

\begin{corollary}[Random-design bound]\label{thm:random-design}
Assume realizability and exact minimization, and suppose Assumptions~\ref{asm:leverage} hold. If $n\gtrsim(\max_{h}\kappa_h)^2(\sum_{h=1}^H\rho_{h:H})^2\log(H^2N_hN_H/\delta)$ for all $h\in[H]$, then with probability $\ge1-\delta$,
\[
|\widehat{J}(\pi)-J(\pi)|\lesssim V_{\max}
\sqrt{\frac{\log\tfrac{H}{\delta}}{n}}\sum_{h=0}^H\|w_h^\star\|_{2,d_h^D} + 
\frac{\Vmax}{n}\sum_{h=1}^H\sum_{t=1}^h
\rho_{t:h}\sqrt{\kappa_h^2\log\tfrac{H^2N_tN_h}{\delta}}\left(\|w_t^\star\|_{2,d_t^D}+1\right). 
\]
\end{corollary}

\subsection{Implications of Bellman Completeness} \label{sec:complete}

Below we discuss the implications of Bellman completeness (Assumption~\ref{asm:completeness}), which is consistent with and  generalizes similar conclusions in the linear setting \citep{perdomo2023complete, amortila2026unifying}. Under completeness, $\wstar_h \cdot d_h^D$ learns $d_h^\pi$ up to the discrimination power of $\Fcal_h$, and the stability term $\rho_{1:h}$ is always well controlled so no exponential blow-up will occur. 

\begin{proposition}\label{prop:completeness}
Under Assumptions~\ref{asm:completeness} and~\ref{asm:leverage}, assuming $\mathbf{0}\in\Fcal_h$ for all $h\in[H]$, we have,\\
(i) $\EE_{(s_h,a_h)\sim d_h^D}[w_h^\star(s_h,a_h)f(s_h,a_h)]=\EE_{(s_h,a_h)\sim d_h^\pi}[f(s_h,a_h)],\quad\forall f\in\Hilb_h;$\\
(ii) if $d_h^\pi\ll d_h^D$, then $w_h^\star=\arg\min_{w\in\Hilb_h}\EE_{d_h^D}\left[(w-\frac{d_h^\pi}{d_h^D})^2\right]$, i.e., $w_h^\star$ is the projection of the MIS ratio $\frac{d_h^\pi}{d_h^D}$ onto $\Hilb_h$ w.r.t. $\|\cdot\|_{2,d_h^D}$ norm;\\
(iii) $(\bp_{t}\bp_{t+1}\cdots\bp_{h-1}f)(s,a)=\EE_\pi[f(s_h,\pi)\mid s_t=s,a_t=a],\quad \forall (s,a)\in\Scal_t\times\Acal,f\in\Hilb_{h};$\\
(iv) $\rho_{t:h}\le 1,\quad\forall 1\le t<h\le H.$
\end{proposition}

We also compare $\|\wstar_h\|_{2, d_h^D}$ as a coverage parameter to the literature under various additional assumptions; see Table~\ref{tab:compare} for a summary.

\begin{proposition}\label{prop:coverage}
Under the assumptions of Proposition~\ref{prop:completeness}, we have
$\|w_h^\star\|_{2,d_h^D}^2=\sup_{f\in\Hilb_h}\frac{(\EE_{d_h^\pi}[f])^2}{\EE_{d_h^D}[f]^2}$.
In the linear setting, without completeness, we have $\|\wstar_h\|_{2,d_h^D}=\|\psi_h\|_{\Sigma_h^{-1}}$, the finite-horizon version of ``feature-dynamics coverage'' identified by \citet{amortila2026unifying}. If completeness is further assumed, $\|\wstar_h\|_{2,d_h^D}=\|\EE_{(s_h,a_h)\sim d_h^\pi}[\phi(s_h,a_h)]\|_{\Sigma_h^{-1}}.$ 
\end{proposition}
When completeness holds, $\|\wstar_h\|_{2, d_h^D}$ almost coincides with the coverage parameter of FQE, up to a minor difference in the function space that we take $\sup$ over. In the linear case (with or without completeness), we also recover previously identified coverage parameters in literature, providing unification of our understanding.

\section{Limitations and Future Work}
\label{sec:conclude}

We conclude the paper by discussing limitations and future work. One major limitation is that we need  i.i.d.~trajectory data; an adaptively collected dataset will break the conditional independence between $\wh_h^\ii$ and $s_{h+1}^\ii | s_h^\ii, a_h^\ii$. A straightforward extension of our results to the infinite-horizon discounted case will result in similar difficulties. One potential solution is to leverage the fact that $\wh_h^\nn$ always lies in $\Hilb_h$ (Proposition~\ref{prop:w-in-span}), so we can union bound over $\Hilb$. However, doing so will incur the dimension dependence in the leading term, and we hope to resolve this issue in the future with more advanced concentration tools.

As another limitation, our weights $\wh_h^\nn$ are computed to minimize matching loss and variance at level $h$ in a rather greedy and local manner, as the effects of $\wh_h^\nn$ on downstream levels $h'> h$ are not considered in optimization. One might attempt to address this issue by optimizing the sum of losses across all levels simultaneously, which essentially recovers MIS methods and runs into a double-sampling issue with non-parametric weights (Appendix~\ref{app:double-sampling}). It remains an open question whether we can optimize the weights $\wh_h^\nn$ across levels while circumventing the issue. 

\section*{Acknowledgement} 
NJ thanks Philip Amortila, Audrey Huang, Akshay Krishnamurthy, and Steven Wu for helpful discussions in the early phases of the project, and  acknowledges funding support from NSF CAREER IIS-2141781 
and Sloan Fellowship.

\clearpage

\bibliographystyle{plainnat}
\bibliography{RL}

\newpage
\appendix

\section{Linear Regression under Covariate Shift}
\label{app:covariate}

The form of our main theoretical guarantee in Section~\ref{sec:method} can feel somewhat unconventional in the OPE literature due to its data dependence. As we will see, however, in the special case of $H=1$ with a linear function class, our algorithm and analyses recover those of \textit{linear regression under covariate shift} 
\citep{yu2012analysis}, where the data-dependent bound corresponds to a standard fixed-design analysis of linear regression \citep{hsu2011analysis}. Therefore, it is instructive to quickly review such an analysis, which will help clarify the nature of our results. 

\paragraph{The covariate shift problem} We consider $(x,y)$ pairs that satisfy $y = x^\top \theta^\star + \epsilon$, where $x \in \RR^\dm$, $\theta^\star$ is the unknown linear coefficients, and $\epsilon$ is some $0$-mean i.i.d.~noise. We are given training data points $\{(x_\train^\ii, y_\train^\ii)\}_{i=1}^n$ sampled i.i.d.~from distribution $P_\train$, and some test data points without $y$ labels $\{x_\test^\ii\}_{i=1}^m$ sampled i.i.d.~from distribution $P_\test$, which is potentially different from $P_\train$. Now the goal is to estimate
$$
\EE_{P_\test}[y_\test].
$$
Note that this problem naturally arises in our OPE setup when $H=1$ and $\Fcal$ is linear in $\phi$ (Definition~\ref{def:linear}): our $\phi(s_1^\ii, a_1^\ii)$ is $x_\train^\ii$, $r_1^\ii$ is $y_\train^\ii$, $\phi(s_1^\ii, \pi)$ is $x_\test^\ii$ (so that $m=n$), and the estimand is exactly $J(\pi)$; see also  \citet{amortila2026unifying} for the connection between these two problems.

\paragraph{Fixed-design analysis of linear regression} Consider a natural algorithm that (1) performs linear regression from the training data to estimate $\emp\theta$, and (2) outputs the estimate as $\tfrac{1}{m}\sum_{i=1}^m (x_\test^\ii)^\top \emp\theta = \xmean^\top \emp{\theta}$, where $\xmean := \tfrac{1}{m}\sum_{i=1}^m x_\test^\ii$. A standard analysis shows that the error can be bounded as:

\begin{proposition}[Informal] \label{prop:covariate}
Under appropriate boundedness assumptions, with probability at least $1-\delta$, $|\EE_{P_\test}[y_\test] - \xmean^\top \emp\theta| \le O(\|\xmean\|_{\hatSigma^{-1}} \sqrt{\log\tfrac{1}{\delta} / n})$, where $\hatSigma := \tfrac{1}{n} \sum_{i=1}^n x_\train^\ii (x_\train^\ii)^\top $ is the sample covariance matrix from the training data.
\end{proposition}

To understand the guarantee, the first observation is that we only care about how accurate $\emp\theta$ is when probed in the mean direction of $x_\test$, $\EE_{P_\test}[x_\test] \approx \xmean$, instead of individual $x_\test \sim P_\test$. The accuracy of the estimation then depends on how well this direction is \textit{covered} in the training distribution, measured by $\|\xmean\|_{\hatSigma^{-1}}$. 
For readers familiar with the online RL literature, one may directly recall a bound of $O(\|\xmean\|_{\hatSigma^{-1}} \sqrt{(\dm + \log\tfrac{1}{\delta})/n})$ that is widely used in the linear bandit/MDP analyses \citep{abbasi-yadkori2012online, jin2020provably}. However, we can avoid paying the $\sqrt{\dm}$ factor here because the online RL bound needs to hold for all probing directions simultaneously, whereas in the covariate shift (OPE) problem we only care about a single probing direction $\xmean$. 

Concretely, Proposition~\ref{prop:covariate} can be proved by rewriting the estimate as
\begin{align*}
\xmean^\top \emp\theta = \xmean^\top \hatSigma^{-1} \left(\frac{1}{n}\sum_{i=1}^n x_\train^\ii y_\train^\ii\right) =: \frac{1}{n} \sum_{i=1}^n \wh^\ii y_\train^\ii. 
\end{align*}
That is, we are effectively computing weights $\wh^\ii = \xmean^\top \hatSigma^{-1} x_\train^\ii$ and using them to compute the weighted average of training labels. These weights satisfy
\begin{align} \label{eq:cov_match}
\frac1n\sum_{i=1}^n \wh^\ii x_\train^\ii = \xmean.
\end{align}
Intuitively, if $\{\wh^\ii\}$ reweight $\{x_\train^\ii\}$ to match $x_\test$ in mean, then $\{\wh^\ii\}$-weighted $y_\train^\ii$ will also match the mean of $y_\test$. 
Based on this intuition, we can rewrite the error as\footnote{A component of the final error ($\xmean^\top \emp\theta - \EE_{P_\test}[y_\test]$) is omitted here, namely, $(\xmean - \EE_{P_\test}[x_\test] )^\top \theta^\star$. While this term reflects the finite sample error of using $\xmean$ to approximate its expectation, we only need to perform concentration along the direction of $\theta^\star$, so the error term is dimension-free and incurs a negligible overhead. In our analyses (e.g., Proposition~\ref{prop:H=1}), this corresponds to the $\epsstat_0$ term. \label{ft:initial} }
$$\xmean^\top \emp\theta - \xmean^\top \theta^\star = \frac{1}{n} \sum_{i=1}^n \wh^\ii (y_\train^\ii - (x_\train^\ii)^\top \theta^\star).$$ 
Note that $y_\train^\ii - (x_\train^\ii)^\top \theta^\star$ are $0$-mean i.i.d.~random variables, and $\{\wh^\ii\}$ are chosen with no dependence on the randomness of $\{y_\train^\ii\}$, so by a standard concentration argument we can bound the error using the 2nd moment of $\wh^\ii$ (see the proof of Theorem~\ref{thm:main} in Appendix~\ref{app:proof-main}), which is exactly $\|\xmean\|_{\hatSigma^{-1}} = \sqrt{\tfrac{1}{n}\sum_{i=1}^n (\wh^\ii)^2}$.

A few remarks are in order:
\begin{itemize}[leftmargin=*]
\item Proposition~\ref{prop:covariate} is a \textit{fixed-design} result, as the bound depends on the matrix $\hatSigma$ which is data dependent, and the analysis applies to settings where $\{x_\train^\ii\}$ is arbitrarily chosen (as in e.g., online learning) and not sampled from some distribution. This contrasts with the \textit{random-design} analysis that has the population $\Sigma = \EE_{P_\train}[x_\train x_\train^\top]$ in the bound instead of $\hatSigma$. Random-design bounds typically require additional assumptions and/or have extra (dimension-dependent) higher-order terms \citep{hsu2011analysis}, as is the case for our analyses in Section~\ref{sec:weight}. 
\item The above analysis assumes $\hatSigma$ is invertible. This is implicitly an assumption on our matching loss $\lossmatch$: invertible $\hatSigma$ implies that Eq.\eqref{eq:cov_match} has solution(s), and Eq.\eqref{eq:cov_match} can be written as: $\forall \theta$,
\begin{align*}
\left|\frac1n\sum_{i=1}^n \wh^\ii (x_\train^\ii)^\top \theta -  \frac1m\sum_{i=1}^m (x_\test^\ii)^\top \theta \right| = 0.
\end{align*}
Translating this back to the OPE notation system, we have 
\begin{align}
0  = &~ \sup_{\theta} \left|\frac1n\sum_{i=1}^n \wh_1^\ii \phi(s_1^\ii, a_1^\ii)^\top \theta -   \frac1n\sum_{i=1}^n\phi(s_1^\ii, \pi)^\top \theta \right| \nonumber \\
= &~ \sup_{f \in \Fcal_\phi} \left|\frac1n\sum_{i=1}^n \wh_1^\ii f(s_1^\ii, a_1^\ii) -  \frac1n\sum_{i=1}^n f(s_1^\ii, \pi) \right| = \lossmatch_1(\wh_1^\nn; \wh_0^\nn). \label{eq:invertible_match}
\end{align}
If $\hatSigma$ is not invertible, we can still use pseudo-inverse to compute $\emp\theta$, and the violation of Eq.\eqref{eq:cov_match} will appear in the error bound as a bias term. Such a term exactly corresponds to the matching error term in Theorem~\ref{thm:main}.
\end{itemize}
\section{Other Related Works}
\label{sec:related}

Throughout the paper we have discussed the connection and comparison to Monte-Carlo estimation and FQE. Here we discuss several other related methods.

\paragraph{Marginalized Importance Sampling (MIS)} 
Our method is closely related to MIS methods \citep{liu2018breaking, nachum2019dualdice}; for example, MWL \citep{uehara2019minimax} also solves a similar minimax problem (min over weights and max over value functions), whose loss is similar to our Eq.~\eqref{eq:loss-match}. Both losses are related to the Bellman flow equation, and roughly speaking, MWL can be viewed as using $\sum_h \lossmatch_h$ as a loss to optimize the weights at all levels simultaneously. As a consequence, however, this introduces dependencies of $\wh_h^\nn$ on $\{s_{h+1}^\ii\}_i$ and violates the condition in our Theorem~\ref{thm:main}, so MWL cannot learn a weight per data point but must introduce a weight class $\Wcal \subset (\Scal\times\Acal\to\RR)$ and learn a parametric form of weight $\wh_h \in \Wcal$ instead. This means an extra realizability condition (of $\Wcal$) and a complexity term (e.g., $\log|\Wcal|$) compared to our analyses. 

Besides the difference in loss and parameterization, there is also a significant difference between the existing analyses of MIS methods and ours, and some insights from our analyses may be transferred. For example, we may use our proof techniques to provide data-dependent bounds for MWL that avoid the $\log|\Fcal|$ term (see Appendix~\ref{app:mis}). In addition, while most MIS methods assume $\Wcal$ models the MIS ratio (Eq.\eqref{eq:mis}), \citet{uehara2019minimax} point out that the method may learn degenerate weights when $\Fcal$ is structured, but a systematic understanding of what the learned weights really represent is still lacking, which our analyses shed light on given the close connection between the methods.

\paragraph{Methods that only rely on realizability} The main guarantee of our method (Theorem~\ref{thm:main}) only relies on realizability $Q^\pi \in \Fcal$ and not the stronger and more widely used Bellman-completeness assumption. Getting such results is also the central focus of the model selection problem in offline RL, where we want to select $Q^\pi$ (or sometimes $Q^\star$) from a set of candidates on holdout data. Given challenges such as the double-sampling problem, these model-selection algorithms often require a tournament procedure and is only computationally efficient for a small number of candidates \citep{xie2020batch, liu2025model}. In comparison, our \alg is computationally efficient for rich function classes under standard optimization oracles.

\paragraph{Top-down algorithms} 
Our algorithm design bears structural similarities to some of the existing methods. For example, \citet{sun2019provably} designs an online imitation-learning algorithm with a dataset that does not contain actions, which performs top-down, level-by-level moment matching against a value discriminator class similar to \alg. However, they have online access to the MDP and the task is to imitate the expert policy, and their analysis requires Bellman completeness for the discriminator class. \citet{huang2023reinforcement} learns the parametric form of the MIS ratio (Eq.\eqref{eq:mis}) in finite-horizon MDPs in a top-down manner, but they do not use a value discriminator but relies on a squared-loss objective, and their analyses also require a version of completeness for the MIS ratio. 
\section{Proofs and Additional Results for Section~\ref{sec:method}}
\label{app:method}

\subsection{Proof of Theorem~\ref{thm:main}} 
\label{app:proof-main}

For each $h\in\{0,1,\ldots,H\}$, define the prefix sigma-field
\[
\Gcal_h:=\sigma\left(\left\{(s_0^\ii,a_0^\ii,r_0^\ii,\ldots,s_{h-1}^\ii,a_{h-1}^\ii,r_{h-1}^\ii,s_h^\ii,a_h^\ii)\right\}_{i=1}^n\right).
\]
For each $h\in\{0,\ldots,H\}$ and sample $i\in[n]$, define $\varepsilon_h^\ii:=Q_h^\pi(s_h^\ii,a_h^\ii)-r_h^\ii-Q_{h+1}^\pi(s_{h+1}^\ii,\pi)$ as the Bellman error. By the Bellman equation, we know that $\EE[\varepsilon_h^\ii\mid \Gcal_h]=0$. Also, since rewards lie in $[0,R_{\max}]$ and total return lies in $[0,V_{\max}]$, we have $|\varepsilon_h^\ii|\le V_{\max}$ almost surely.

Now, condition on $\Gcal_h$. Because the $n$ trajectories are i.i.d., once the prefixes up to $(s_h^\ii,a_h^\ii)$ are fixed, the random variables $\{\varepsilon_h^\ii\}_{i=1}^n$ remain conditionally independent across $i$. Since $\wh_h^\nn$ is $\Gcal_h$-measurable, the weighted variables $X_h^\ii:=\wh_h^\ii\varepsilon_h^\ii$ are also conditionally independent given $\Gcal_h$, satisfy
\[
\EE[X_h^\ii\mid \Gcal_h]=0,\quad |X_h^\ii|\le|\wh_h^\ii|V_{\max}.
\]
Applying Hoeffding's inequality conditionally on $\Gcal_h$, for every $t>0$,
\[
\PP\left[\left|\frac1n\sum_{i=1}^nX_h^\ii\right|\ge t \middle| \Gcal_h\right]\le2\exp\left(-\frac{n^2t^2}{2\sum_{i=1}^n(V_{\max}\wh_h^\ii)^2}\right)=2\exp\left(-\frac{nt^2}{2V_{\max}^2\|\wh_h^\nn\|_\nn^2}\right).
\]
Setting $t=\epsstat_h:=V_{\max}\|\wh_h^\nn\|_\nn\sqrt{\frac{2\log(2(H+1)/\delta)}{n}}$ and taking a union bound over $h=0,1,\ldots,H$, we obtain an event $\mathcal{E}$ with $\Pr[\mathcal{E}]\ge1-\delta$ such that on $\mathcal{E}$, simultaneously for all $h=0,\ldots,H$,
\begin{equation}\label{eq:proof-concentration}
\left|\frac1n\sum_{i=1}^n\wh_h^\ii\left(Q_h^\pi(s_h^\ii,a_h^\ii)-r_h^\ii-Q_{h+1}^\pi(s_{h+1}^\ii,\pi)\right)\right|\le\epsstat_h.
\end{equation}

After obtaining the statistical concentration rate, we can follow the idea in Section~\ref{sec:method} to finish the proof, which reduces the OPE error to matching losses plus the statistical errors. Fix the event $\mathcal{E}$. Since $r_0\equiv 0$ and $\wh_0^\ii\equiv 1$,
\begin{align*}
    |\widehat{J}(\pi)-J(\pi)|&=\left|\sum_{h=0}^H\frac1n\sum_{i=1}^n\wh_h^\ii r_h^\ii-J(\pi)\right|\\
    &\le\left|\sum_{h=0}^H\frac1n\sum_{i=1}^n\wh_h^\ii\left(Q_h^\pi(s_h^\ii,a_h^\ii)-Q_{h+1}^\pi(s_{h+1}^\ii,\pi)\right)-J(\pi)\right|+\sum_{h=0}^H\epsstat_h,
\end{align*}
where we used Eq.~\eqref{eq:proof-concentration} and the triangle inequality. Now abbreviate $Q_h^\ii:=Q_h^\pi(s_h^\ii,a_h^\ii)$, $Q_{h,\pi}^\ii:=Q_h^\pi(s_h^\ii,\pi)$. Then the first term becomes
\[
\left|\sum_{h=0}^H\frac1n\sum_{i=1}^n\wh_h^\ii\left(Q_h^\ii-Q_{h+1,\pi}^\ii\right)-J(\pi)\right|.
\]
For each $h=1,\ldots,H$, realizability gives $Q_h^\pi\in\Fcal_h$. Hence, by plugging $f=Q_h^\pi$ into the definition of $\lossmatch_h$ (Eq.~\eqref{eq:loss-match}),
\[
\left|\frac1n\sum_{i=1}^n\wh_h^\ii Q_h^\ii-\frac1n\sum_{i=1}^n\wh_{h-1}^\ii Q_{h,\pi}^\ii\right|\le\lossmatch_h(\wh_h^\nn;\wh_{h-1}^\nn).
\]
Thus, for $h=1,2,\ldots,H$, we have
\begin{align*}
    &\left|\sum_{h=0}^H\frac1n\sum_{i=1}^n\wh_h^\ii\left(Q_h^\ii-Q_{h+1,\pi}^\ii\right)-J(\pi)\right|\\
    &\le\left|\sum_{h=0}^H\frac1n\sum_{i=1}^n\left(\wh_h^\ii Q_h^\ii-\wh_{h+1}^\ii Q_{h+1}^\ii\right)-J(\pi)\right|+\sum_{h=1}^H\lossmatch_h(\wh_h^\nn;\wh_{h-1}^\nn)\\
    &=\left|\frac1n\sum_{i=1}^n\left(\wh_0^\ii Q_0^\ii-\wh_{H+1}^\ii Q_{H+1}^\ii\right)-J(\pi)\right|+\sum_{h=1}^H\lossmatch_h(\wh_h^\nn;\wh_{h-1}^\nn),
\end{align*}
where the second equality is due to telescoping. Since $\wh_0^\ii=1$, $Q_{H+1}^\ii=0$, and $Q_0^\pi(s_0,a_0)=J(\pi)$ due to the dummy initial pair $(s_0,a_0)$ is fixed, thus the absolute value term above is exactly zero. Therefore, the OPE error reduces to matching losses plus the statistical errors, i.e.,
\[
|\widehat{J}(\pi)-J(\pi)|\le\sum_{h=1}^H\lossmatch_h(\wh_h^\nn;\wh_{h-1}^\nn)+\sum_{h=0}^H\epsstat_h
\]
on the event $\mathcal{E}$. Since $\PP[\mathcal{E}]\ge 1-\delta$, the theorem follows.

\subsection{Proof of Proposition~\ref{prop:H=1}} \label{app:prop-H-1}

We now give the proof for Proposition~\ref{prop:H=1}, which heavily relies on the analysis of linear regression with covariate shift (see Appendix~\ref{app:covariate}).
\begin{proof}[Proof of Proposition~\ref{prop:H=1}]
    Consider $H=1$ and linear $\Fcal_1$. Write
    \[
    x_i:=\phi(s_1^\ii,a_1^\ii),\quad\xmean:=\frac1n\sum_{i=1}^n\phi(s_1^\ii,\pi),\quad\hatSigma:=\frac1n\sum_{i=1}^n x_ix_i^\top.
    \]
    We first consider the case where $\hatSigma$ is invertible, define $\wh_1^\ii:=\xmean^\top\hatSigma^{-1} x_i$ for $i\in[n]$. Take any $f\in\Fcal_1$ such that $f(s,a)=\phi(s,a)^\top\theta$, then we have
    \[
    \frac1n\sum_{i=1}^n\wh_1^\ii f(s_1^\ii,a_1^\ii)=\frac1n\sum_{i=1}^n\bigl(\xmean^\top\hatSigma^{-1}x_i\bigr)x_i^\top\theta=\xmean^\top\theta,
    \]
    since $\frac1n\sum_{i=1}^nx_ix_i^\top$ cancels with $\hatSigma$. On the other hand, by the definition of $\xmean$, we have
    \[
    \frac1n\sum_{i=1}^n f(s_1^\ii,\pi)=\frac1n\sum_{i=1}^n\phi(s_1^\ii,\pi)^\top\theta=\xmean^\top\theta.
    \]
    Therefore, for any $f\in\Fcal_1$, since $\wh_0^\ii=1$ for all $i\in[n]$,
    \[
    \frac1n\sum_{i=1}^n\wh_1^\ii f(s_1^\ii,a_1^\ii)=\frac1n\sum_{i=1}^n \wh_0^\ii f(s_1^\ii,\pi),
    \]
    which indicates that $0$ empirical matching loss can be achieved, i.e., $\lossmatch_1(\wh_1^\nn;\wh_0^\nn)=0$. This proves the first claim. For the empirical 2nd-moment of the weight, we have
    \[
    \bigl\|\wh_1^\nn\bigr\|_\nn^2=\frac1n\sum_{i=1}^n\bigl(\xmean\hatSigma^{-1}x_i\bigr)^2=\xmean^\top\hatSigma^{-1}\left(\frac1n\sum_{i=1}^n x_ix_i^\top\right)\hatSigma^{-1}\xmean=\xmean^\top\hatSigma^{-1}\xmean,
    \]
    where we again use $\hatSigma=\frac1n\sum_{i=1}^nx_ix_i^\top$. Hence, $\|\wh_1^\nn\|_\nn=\|\xmean\|_{\hatSigma^{-1}}$, which proves the second claim. And the OPE error bound can be obtained immediately by Theorem~\ref{thm:main}:
    \[
    \bigl|\emp{J}(\pi)-J(\pi)\bigr|\le\epsstat_0+\epsstat_1=V_{\max}\sqrt{\frac{2\log(4/\delta)}{n}}\left(1+\|\xmean\|_{\hatSigma^{-1}}\right).
    \]

    For non-invertible $\hatSigma$, define $\wh_1^\ii:=\xmean^\top\hatSigma^\dagger x_i$ for $i\in[n]$, where $\hatSigma^\dagger$ is the Moore-Penrose pseudo-inverse of $\hatSigma$. We show that this weight choice minimizes the matching loss.

    For any $w\in\RR^n$, define its induced empirical feature moment $\mu(w):=\frac1n\sum_{i=1}^n w^\ii x_i$. Since $\Fcal_1$ is linear and $\|\theta\|_2\le\Theta$, we have
    \[
    \lossmatch_1(w;\wh_0^\nn)=\sup_{\theta:\|\theta\|_2\le\Theta}\left|\frac1n\sum_{i=1}^n w^\ii x_i^\top\theta-\xmean^\top\theta\right|=\sup_{\theta:\|\theta\|_2\le\Theta}|\langle\mu(w)-\xmean,\theta\rangle|=\Theta\|\mu(w)-\xmean\|_2,
    \]
    where the last step is due to Cauchy-Schwarz inequality. Thus minimizing the empirical matching loss is equivalent to choosing $\mu(w)$ as close as possible to $\xmean$ in Euclidean norm.
    
    Let $\Xcal:=\mathrm{span}(x_1,\ldots,x_n)$. For every $w$, clearly $\mu(w)\in\Xcal$ by its definition. Conversely, for any $v\in\Xcal$, there exist coefficients $\alpha_i$ such that $v=\sum_{i=1}^n\alpha_ix_i$, and choosing $w^\ii=n\alpha_i$ gives $\mu(w)=v$. Hence the set of all achievable empirical moments is exactly $\Xcal$.

    Moreover, the range space of $\hatSigma$ is $\operatorname{Range}(\hatSigma)=\Xcal$. This is because $\hatSigma=\frac1n\sum_{i=1}^n x_ix_i^\top$, so that $\operatorname{Range}(\hatSigma)\subseteq\Xcal$. On the other hand, since $\hatSigma$ is symmetric, $\operatorname{Range}(\hatSigma)=\operatorname{Null}(\hatSigma)^\bot$, where $\operatorname{Null}(\hatSigma)$ is the nullspace of $\hatSigma$. But $v^\top\hatSigma v=\frac1n\sum_{i=1}^n(v^\top x_i)^2$, so $v\in\operatorname{Null}(\hatSigma)$ is equivalent to $v\bot x_i$ for all $i\in[n]$, thus $v\in\Xcal^\bot$. Therefore, $\operatorname{Null}(\hatSigma)=\Xcal^\bot$, and hence $\operatorname{Range}(\hatSigma)=\Xcal$.

    It follows that the minimum empirical matching loss is obtained by projecting $\xmean$ onto the range space of $\hatSigma$, whose projection operator is indeed $\hatSigma\hatSigma^\dagger$:
    \[
    \mu(w)=\Pi_\Xcal\xmean=\hatSigma\hatSigma^\dagger\xmean.
    \]
    The proposed weights indeed achieve this moment, since
    \[
    \frac1n\sum_{i=1}^n\wh_1^\ii x_i=\frac1n\sum_{i=1}^n\bigl(\xmean^\top\hatSigma^\dagger x_i\bigr)x_i=\left(\frac1n\sum_{i=1}^n x_ix_i^\top\right)\hatSigma^{\dagger}\xmean=\hatSigma\hatSigma^\dagger \xmean.
    \]
    Therefore $\wh_1^\nn$ minimizes the empirical matching loss. The minimum loss is
    \[
    \lossmatch_1(\wh_1^\nn;\wh_0^\nn)=\Theta\bigl\|(I-\hatSigma\hatSigma^\dagger)\xmean\bigr\|_2.
    \]
    In particular, empirical 0 matching loss is achievable if and only if $\xmean\in\operatorname{Range}(\hatSigma)$.
\end{proof}

\subsection{Proof of Proposition~\ref{prop:linear}} \label{app:prop-linear}

Now we consider Proposition~\ref{prop:linear} which connects our \alg algorithm to the canonical FQE algorithm in the linear and the tabular settings. Recall that FQE estimates the target-policy value functions recursively from the last step to the first step by solving a sequence of supervised regression problems: given the current estimate $\emp{Q}_{h+1}^\pi$, the next estimate $\emp{Q}_h^\pi$ is obtained by regressing the Bellman target $r_h+\emp{Q}_{h+1}^\pi(s_{h+1},\pi)$ onto the current state-action features.

\paragraph{Linear setting} In the linear case, this admits a simple matrix-vector implementation.
In particular, linear FQE sets $\emp{Q}_{H+1}\equiv0$, or equivalently $\hattheta_{H+1}=0$, and for $h=H,H-1,\ldots,1$, computes $\emp{Q}_h^\pi(s,a)=\phi(s,a)^\top\hattheta_h$, where
\[
\emp{\theta}_h=\argmin_{\theta\in\RR^\dm}\frac1n\sum_{i=1}^n\left(\phi(s_h^\ii,a_h^\ii)^\top\theta-r_h^\ii-\phi(s_{h+1}^\ii,\pi)^\top\hattheta_{h+1}\right)^2.
\]
Equivalently, by the normal equation,
\[
\hatSigma_h\hattheta_h=\frac1n\sum_{i=1}^n\phi(s_h^\ii,a_h^\ii)\left(r_h^\ii+\phi(s_{h+1}^\ii,\pi)^\top\hattheta_{h+1}\right).
\]
Since $s_0,a_0$ is fixed and $r_0\equiv0$, the FQE estimate of $J(\pi)$ is
\begin{equation}\label{eq:fqe-linear}
    \emp{J}_\mathrm{FQE}(\pi):=\frac1n\sum_{i=1}^n\emp{Q}_1^\pi(s_1^\ii,\pi)=\frac1n\sum_{i=1}^n\phi_1(s_1^\ii,\pi)^\top\hattheta_1.
\end{equation}
The following proof shows that the \alg algorithm produces exactly the same $J(\pi)$ estimate as linear FQE in Eq.~\eqref{eq:fqe-linear}.
\begin{proof}[Proof of Proposition~\ref{prop:linear}(1)]
    We first characterize the \alg weights in the linear case. At level $h$, empirical 0 matching loss means that, for all linear functions $f(s,a)=\phi(s,a)^\top\theta$,
    \[
    \frac1n\sum_{i=1}^n\wh_h^\ii\phi(s_h^\ii,a_h^\ii)^\top\theta=\frac1n\sum_{i=1}^n\wh_{h-1}^\ii\phi(s_h^\ii,\pi)^\top\theta.
    \]
    Define $\hatnu_h:=\frac1n\sum_{i=1}^n\wh_{h-1}^\ii\phi(s_h^\ii,\pi)$ as the target moment. Then \alg needs to choose $\wh_h^\nn$ such that
    \[
    \frac1n\sum_{i=1}^n\wh_h^\ii\phi(s_h^\ii,a_h^\ii)=\hatnu_h.
    \]
    By the same linear-regression / moment-matching calculation as in Proposition~\ref{prop:H=1}, the least-2nd-moment solution that achieves empirical $0$ matching loss is given by $\wh_h^\ii=\phi(s_h^\ii,a_h^\ii)^\top\hatSigma_h^{-1}\hatnu_h$ for all $i\in[n]$, since $\hatSigma_h$ is assumed to be invertible. It remains to show that the final \alg prediction equals to the FQE prediction $\emp{J}_\mathrm{FQE}(\pi)$. Define, for $h\in[H]$,
    \[
    V_h:=\frac1n\sum_{i=1}^n\wh_{h-1}^\ii\emp{Q}_h^\pi(s_h^\ii,\pi)=\frac1n\sum_{i=1}^n\wh_{h-1}^\ii\phi(s_h^\ii,\pi)^\top\hattheta_h.
    \]
    Since $\wh_0\equiv1$, we have $V_1=\emp{J}_\mathrm{FQE}(\pi)$ in Eq.~\eqref{eq:fqe-linear}. We now prove the recursion
    \begin{equation}\label{eq:fqe-v-recursion}
    V_h=\frac1n\sum_{i=1}^n\wh_h^\ii r_h^\ii+V_{h+1},\quad V_{H+1}\equiv0.
    \end{equation}
    Let $y_h^\ii:=r_h^\ii+\phi(s_{h+1}^\ii,\pi)^\top\hattheta_{h+1}$. By the linear FQE normal equation, the target moment can be rewritten as $\hatSigma_h\hattheta_h=\frac1n\sum_{i=1}^n\phi(s_h^\ii,a_h^\ii)y_h^\ii$. Therefore, using the closed form of the \alg weights $\wh_h^\ii=\phi(s_h^\ii,a_h^\ii)^\top\hatSigma_h^{-1}\hatnu_h$, we obtain
    \begin{align*}
    \frac1n\sum_{i=1}^n\wh_h^\ii y_h^\ii&=\frac1n\sum_{i=1}^n\phi(s_h^\ii,a_h^\ii)^\top\hatSigma_h^{-1}\hatnu_h y_h^\ii=\hatnu_h^\top\hatSigma_h^{-1}\left(\frac1n\sum_{i=1}^n\phi(s_h^\ii,a_h^\ii)y_h^\ii\right)\\
    &=\hatnu_h^\top\hatSigma_h^{-1}\cdot \hatSigma_h\hattheta_h=\hatnu_h^\top\hattheta_h=\frac1n\sum_{i=1}^n\wh_{h-1}^\ii\phi(s_h^\ii,\pi)^\top\hattheta_h=V_h,
    \end{align*}
    where the first equality is by the definition of $\wh_h^\ii$, the second equality is by rearranging, the third equality is by the linear FQE normal equation, the fourth equation cancels $\hatSigma$, the fifth equation is by the definition of $\hatnu_h$, and the last equation is by the definition of $V_h$. Therefore,
    \[
    V_h=\frac1n\sum_{i=1}^n\wh_h^\ii\left(r_h^\ii+\phi(s_{h+1}^\ii,\pi)^\top\hattheta_{h+1}\right)=\frac1n\sum_{i=1}^n\wh_h^\ii r_h^\ii+V_{h+1}.
    \]
    This proves the recursion in Eq.~\eqref{eq:fqe-v-recursion}. Unrolling this recursion from $h=1$ to $H$, and using $V_{H+1}=0$, yields
    \[
    \emp{J}_\mathrm{FQE}(\pi)=V_1=\sum_{h=1}^H\frac1n\sum_{i=1}^n\wh_h^\ii r_h^\ii,
    \]
    while the RHS is exactly the \alg estimate $\emp{J}(\pi)$ in Algorithm~\ref{alg:main}. This proves the claim that in linear setting with invertible $\{\hatSigma_h\}_h$, our \alg algorithm is equivalent to linear FQE.
\end{proof}

\paragraph{Tabular setting} Now we consider the tabular function class $\Fcal_\tab=[0,V_{\max}]^{\Scal\times\Acal}$ in Definition~\ref{def:tabular}. As in Proposition~\ref{prop:linear}(2), we assume every state-action pair at each level appears in the dataset, i.e., for any $h\in[H]$ and $(s,a)\in\Scal_h\times\Acal$,
\[
\emp{d}_h(s,a)>0,\quad\text{where } \emp{d}_h(s,a):=\frac1n\sum_{i=1}^n\bone\{s_h^\ii=s,a_h^\ii=a\}.
\]
For tabular FQE. Define the empirical reward and empirical transition respectively by
\[
\emp{r}_h(s,a):=\frac{\sum_{i=1}^n\bone\{s_h^\ii=s,a_h^\ii=a\}r_h^\ii}{\sum_{i=1}^n\bone\{s_h^\ii=s,a_h^\ii=a\}},\quad \hatP_h(s'|s,a):=\frac{\sum_{i=1}^n\bone\{s_h^\ii=s,a_h^\ii=a,s_{h+1}^\ii=s'\}}{\sum_{i=1}^n\bone\{s_h^\ii=s,a_h^\ii=a\}}.
\]
Together, $\{\emp{r}_h,\hatP_h\}_{h=1}^H$ define the empirical MDP. Tabular FQE computes $\emp{Q}_{H+1}\equiv 0$, and for $h=H,H-1,\ldots,1$, solves the squared-loss regression problem
\[
\emp{Q}_h\in\argmin_{Q:\Scal_h\times\Acal\to[0,V_{\max}]}\frac1n\sum_{i=1}^n\left(Q(s_h^\ii,a_h^\ii)-r_h^\ii-\emp{Q}_{h+1}(s_{h+1}^\ii,\pi)\right)^2.
\]
Since the function class is , the regression separates over each state-action pair $(s,a)$. Therefore, for every $(s,a)$,
\[
\emp{Q}_h(s,a)=\emp{r}_h(s,a)+\sum_{s'}\hatP_h(s'\mid s,a)\emp{Q}_{h+1}(s',\pi).
\]
Thus, tabular FQE is exactly Bellman evaluation of $\pi$ in the empirical MDP. Hence its prediction is the model-based certainty-equivalence value
\begin{equation}\label{eq:fqe-tabular}
\emp{J}_\mathrm{FQE}(\pi)=\emp{J}_{\hatnu}(\pi),
\end{equation}
where $\hatnu$ denotes the empirical MDP $(\Scal,\Acal,\{\emp{r}_h\}_h,\{\hatP_h\}_h)$. The following proof shows that the \alg algorithm produces exactly same $J(\pi)$ estimate as tabular FQE (or model-based certainty-equivalence) in Eq.~\eqref{eq:fqe-tabular}.
\begin{proof}[Proof of Proposition~\ref{prop:linear}(2)]
    We first characterize the \alg weights in the tabular case. Let $\emp{d}_h^\pi$ be the occupancy distribution of $\pi$ in the empirical MDP $\hatnu$. Since $s_0,a_0$ is fixed and $\wh_0^\ii\equiv1$, the empirical occupancy is recursively defined by
    \[
    \emp{d}_h^\pi=\pi(a\mid s)\sum_{\tilde{s},\tilde{a}}\emp{d}_{h-1}^\pi(\tilde{s},\tilde{a})\hatP_{h-1}(s\mid \tilde{s},\tilde{a}).
    \]
    We prove by induction on $h$ that the least-2nd-moment weight $\wh_h^\nn$ that achieves empirical $0$ matching loss satisfy
    \[
    \wh_h^\ii=\frac{\emp{d}_h^\pi(s_h^\ii,a_h^\ii)}{\emp{d}_h^D(s_h^\ii,a_h^\ii)}.
    \]
    Fix $h$. Since the tabular class contains state-action indicators, empirical 0 matching loss at level $h$ is equivalent to matching every state-action pair:
    \[
    \frac1n\sum_{i=1}^n\wh_h^\ii\bone\{s_h^\ii=s,a_h^\ii=a\}=\frac1n\sum_{i=1}^n\wh_{h-1}^\ii\bone\{s_h^\ii=s\}\pi(a\mid s),
    \]
    for all $(s,a)\in\Scal_h\times\Acal$. By the induction hypothesis, the RHS becomes
    \begin{align*}
        &\frac1n\sum_{i=1}^n\wh_{h-1}^\ii\bone\{s_h^\ii=s\}\pi(a\mid s)\\
        &=\pi(a\mid s)\sum_{\tilde{s},\tilde{a}}\frac{\emp{d}_{h-1}^\pi(\tilde{s},\tilde{a})}{\emp{d}_{h-1}^D(\tilde{s},\tilde{a})}\cdot\frac1n\sum_{i=1}^n\bone\{s_{h-1}^\ii=\tilde{s},a_{h-1}^\ii=\tilde{a},s_h^\ii=s\}\\
        &=\pi(a\mid s)\sum_{\tilde{s},\tilde{a}}\emp{d}_{h-1}^\pi(\tilde{s},\tilde{a})\hatP_{h-1}(s\mid \tilde{s},\tilde{a})\\
        &=\emp{d}_h^\pi(s,a).
    \end{align*}
    On the other hand, the LHS can be written as $\frac1n\sum_{i:s_h^\ii=s,a_h^\ii=a}\wh_h^\ii$. Thus, empirical 0 matching loss requires $\sum_{i:s_h^\ii=s,a_h^\ii=a}\wh_h^\ii=n\emp{d}_h^\pi(s,a)$ for all $(s,a)\in\Scal_h\times\Acal$.
    
    Among all weights satisfying this cell-wise constraint, the least-2nd-moment solution is constant within each cell. This is because for a fixed cell $(s,a)$, minimizing $\sum_{i:s_h^\ii=s,a_h^\ii=a}(\wh_h^\ii)^2$ subject to the fixed sum constraint is achieved when all weights in that cell are equal. Hence, for every $i$ with $(s_h^\ii,a_h^\ii)=(s,a)$,
    \[
    \wh_h^\ii=\frac{n\emp{d}_h^\pi(s,a)}{\sum_{j=1}^n\bone\{s_h^{(j)}=s,a_h^{(j)}=a\}}=\frac{\emp{d}_h^\pi(s,a)}{\emp{d}_h^D(s,a)}.
    \]
    This proves the desired formula for the \alg weights $\wh_h^\nn$. It remains to compare the final estimates. Recall the \alg estimate in Algorithm~\ref{alg:main} is given by
    \[
    \emp{J}(\pi)=\sum_{h=1}^H\frac1n\sum_{i=1}^n\wh_h^\ii r_h^\ii.
    \]
    Using the weight formula above, we get
    \begin{align*}
        \emp{J}(\pi)&=\sum_{h=1}^H\sum_{s,a}\frac1n\sum_{i:s_h^\ii=s,a_h^\ii=a}\frac{\emp{d}_h^\pi(s,a)}{\emp{d}_h^D(s,a)}r_h^\ii\\
        &=\sum_{h=1}^H\sum_{s,a}\emp{d}_h^\pi(s,a)\frac{\sum_{i:s_h^\ii=s,a_h^\ii=a}r_h^\ii}{\sum_{i=1}^n\bone\{s_h^\ii=s,a_h^\ii=a\}}\\
        &=\sum_{h=1}^H\sum_{s,a}\emp{d}_h^\pi(s,a)\emp{r}_h(s,a),
    \end{align*}
    where the last expression is exactly the expected return of $\pi$ in the empirical MDP $\hatnu$, i.e.,
    \[
    \emp{J}_{\hatnu}(\pi)=\sum_{h=1}^H\sum_{s,a}\emp{d}_h^\pi(s,a)\emp{r}_h(s,a).
    \]
    Therefore, by Eq.~\eqref{eq:fqe-tabular}, $\emp{J}(\pi)=\emp{J}_\mathrm{FQE}(\pi)$. This tells that in the tabular setting with full empirical support, \alg and tabular FQE both coincide with the model-based certainty-equivalence solution, and the \alg weights are exactly the empirical occupancy ratios.    
\end{proof}

\subsection{Comparison to Prior Analyses in Linear and Tabular Settings}
\label{app:prior-linear}

In this section we compare our guarantees with those of \citet{duan2020minimax} in the linear setting. By Proposition~\ref{prop:linear}(1), in the linear setting with invertible empirical covariance matrices $\hatSigma_h$, \alg with exact minimization and the least-2nd-moment rule produces exactly the same estimator as linear FQE without ridge regularization. 

We now compare their data-independent bound to ours, and will comment on the data-dependent bounds below. For our result, we specialize Corollary~\ref{thm:random-design} to the linear setting:
\begin{equation}\label{eq:guarantee-ours}
|\emp{J}(\pi)-J(\pi)|\lesssim V_{\max}\sqrt{\frac{\log(H/\delta)}{n}}\sum_{h=0}^H\|\wstar_h\|_{2,d_h^D}+\text{higher-order terms}.
\end{equation}
\citet[Theorem 2]{duan2020minimax} also give a finite sample error bound; after translation between settings and notation,\footnote{We apply the bound of \citet{duan2020minimax} to our setup of disjoint state spaces, which admits (nominally) time-homogeneous dynamics and feature map that match their setup. In their bound, the coverage term is $\sup_{f\in\Fcal} |\EE_{(s,a)\sim d_h^\pi}[f_h(s,a)]| / \sqrt{\EE_{D}[\tfrac{1}{H}\sum_{h=1}^H f_h(s,a)^2]}$. When $\{\Scal_h\}_{h}$ are disjoint and feature/linear coefficients are essentially independent across levels, the $\sup_f$ is always achieved by setting $f_{h'} \equiv 0, \forall h' \ne h$. This way, their coverage becomes $\sqrt{H} \cdot \sup_{f_h\in\Fcal_h} |\EE_{(s,a)\sim d_h^\pi}[f_h(s,a)]|] / \sqrt{\EE_{D}[f_h(s,a)^2]} = \sqrt{H} \|\EE_{d_h^\pi}[\phi]\|_{\Sigma_h^{-1}}$ (Proposition~\ref{prop:coverage}). The extra $\sqrt{H}$ cancels out when we replace their $1/\sqrt{N}$ with $1/\sqrt{n}$, as their $N$ is the number of transitions and hence $N=nH$.} it is:
\begin{equation}\label{eq:guarantee-duan}
|\emp{J}(\pi)-J(\pi)|\lesssim\sum_{h=0}^H V_h \sqrt{\frac{\log(1/\delta)}{n}} \cdot \|\EE_{d_h^\pi}[\phi]\|_{\Sigma_h^{-1}} +\text{higher-order terms}
\end{equation}
where $V_h:=(H-h+1)R_{\max} = O(\Vmax)$ is a stage-dependent range. 

\begin{itemize}[leftmargin=*]
\item \emphb{Same coverage term.} As we recognize $\|\wstar_h\|_{2,d_h^D}$ as the coverage term in Eq.~\eqref{eq:guarantee-ours}, the corresponding term in Eq.~\eqref{eq:guarantee-duan} is $\|\EE_{d_h^\pi}[\phi]\|_{\Sigma_h^{-1}}$. These two terms are exactly identical under Bellman completeness in the linear setting 
(Proposition~\ref{prop:representer-linear}). 

\item \emphb{Same $n^{-1/2}$ rate.} Both bounds have the same leading $n^{-1/2}$ statistical rate, horizon dependence on $H$, and the range parameter up to minor differences (e.g., $V_{\max}$ vs.~$V_h$ and the $\log H$ factor). Both bounds are dimension-free in the leading $O(1/\sqrt{n})$ term.

\item \emphb{Burn-in condition and higher-order term.} The burn-in condition (i.e., requirement of sample size $n\gtrsim(\cdot)$) and the higher-order term of \citet{duan2020minimax} depend on a more stringent notion of coverage, such as $\operatorname{cond}(\Sigma_h^{-1/2}\Sigma_h^\pi\Sigma_h^{-1/2})$, where $\Sigma_h$ is the feature covariance under $d_h^D$ and $\Sigma_h^\pi$ is that under $d_h^\pi$. 
This ``squared'' version coverage requires data to cover all feature directions under $\pi$, instead of just the mean direction $\EE_{d_h^\pi}[\phi]$ 
\citep{jiang2025offline}. 
In contrast, the higher-order term in our bound in Eq.~\eqref{eq:guarantee-ours} relies on the tracking analysis (Theorem~\ref{thm:main}), where only the size of $\wstar_h$ shows up and no squared coverage is needed. 

\item \emphb{Data-dependent bounds.} Both our work and \citet{duan2020minimax} also feature data-dependent bounds that are computable from data, i.e., our Theorem~\ref{thm:main} and their Theorem 4. 
The bounds  look structurally similar to Eqs.\eqref{eq:guarantee-ours} and \eqref{eq:guarantee-duan}, except that the coverage terms become their empirical counterparts. 
Moreover, our Theorem~\ref{thm:main} is dimension-free while their bound incurs dependence on $\sqrt{\dm}$; this might be due to the fact that their Theorem 4 (data-dependent bound) assumes i.i.d.~transitions (their data-independent bound in Theorem 2 assumes i.i.d.~trajectories), whereas all our results are under i.i.d.~trajectories and the telescoping step in the proof of our Theorem~\ref{thm:main} relies on trajectory data to have perfect cancellation. In addition, their Theorem 4 also has additional higher-order terms, whereas our Theorem~\ref{thm:main} is very clean and only has the leading term. 

\item \emphb{Assumptions and setup.}  Technically speaking, \citet{duan2020minimax} considers the linear-MDP setup which is even stronger than Bellman completeness, but their proofs and results can be adapted to much weaker settings (e.g., only realizability), especially given the later insights of \citet{perdomo2023complete, amortila2026unifying}. In particular, \citet{amortila2026unifying} analyze LSTDQ, an algorithm closely related to linear FQE, in the infinite-horizon discounted setting under only realizability, and their data-independent bound matches that of \citet{duan2020minimax} in the stable regime ($\|B^\pi\|_{\textrm{op}} \le 1$) under the standard conversion between finite-horizon and discounted results. On top of that, key quantities such as $\psi_h$ also appear in \citet{duan2020minimax}'s analysis, so there is good reason to believe that \citet{duan2020minimax}'s bound applies to the setting where only realizability of $Q^\pi$ is assumed, where the $\EE_{d_h^\pi}[\phi]$ term in Eq.\eqref{eq:guarantee-duan} will be replaced by $\psi_h$.
\end{itemize}

We also note that similar dimension-free bounds have also been obtained in the tabular setting \citep[Theorem 3.1]{yin2020asymptotically}, and the leading term matches the form of Eqs.\eqref{eq:guarantee-ours} and \eqref{eq:guarantee-duan}. One improvement in \citet{yin2020asymptotically} is that they replace the range ($\Vmax$ or $V_h$) with (the square-root of) the variance of reward and value function $V^\pi$ w.r.t.~state transitions; indeed, they show that this is the form of the asymptotically minimax-optimal bound. This improvement is easy to implement in our analyses, as we can replace the Hoeffding's inequality in Eq.\eqref{eq:proof-concentration} with a Bernstein's inequality, where the variance of $V^\pi$ w.r.t.~state transitions will naturally replace $\Vmax$ in the leading term.

\subsection{Computational Efficiency} \label{app:computation}

For linear $\Fcal$, the computational efficiency of \alg is clear given that it reduces to linear FQE, where the computation involved is just a sequence of linear regressions. For an RKHS with kernel $k$ and unit-ball discriminator class $\Fcal$, the matching loss becomes the maximum mean discrepancy (MMD) \citep{gretton2012kernel}:
\[
\widehat{\Lcal}_h(w;\wh_{h-1}^\nn)=\left\|\frac{1}{n}\sum_{i=1}^n\left(w^{(i)}k((s_h^{(i)},a_h^{(i)}),\cdot)-\hat w_{h-1}^{(i)}k((s_h^{(i)},\pi),\cdot)\right)\right\|_{\mathcal H_k},
\]
where $k((s,\pi),\cdot):=\mathbb E_{a\sim\pi(\cdot|s)}[k((s,a),\cdot)]$.
The squared loss is a convex quadratic function of $w$, and hence can be efficiently minimized.

For general $\Fcal$, computational guarantees are often given in the form of \textit{oracle efficiency} \citep{dann2018oracle}, and below we show that the computation is efficient given access to some optimization primitives over $\Fcal$. 
Recall the only computation-heavy step in \alg is the matching step, which is ($w_h^\nn$ in Eq.\eqref{eq:loss-match} is replaced with $w$ for notational simplicity) $
\argmin_{w \in \RR^n}\lossmatch_h(w; \wh_{h-1}^\nn) := \sup_{f\in\Fcal_h}  \left|\ell(w,f)\right|$, where
\[
\ell(w,f)=\frac{1}{n}\sum_{i=1}^n w^\ii f(s_{h}^\ii, a_{h}^\ii) - \frac{1}{n}\sum_{i=1}^n \wh_{h-1}^\ii f(s_{h}^\ii, \pi) 
\]
To remove the absolute value, we would like to assume $\Fcal_h$ is a symmetric function class in the following discussion. That is, $\Fcal_h$ is closed under negation: $-f \in \Fcal_h, \forall f\in\Fcal_h$.

Now, this version only minimizes the matching loss but does not control the second moment of $\wh_h^\nn$. To enable second-moment control, we can either add a regularization term $\lambda \secmom{w}^2$ to $\ell(w_h^\nn, f)$
for an appropriate $\lambda > 0$, or add constraint on $w$ that we can only choose from $\Wcal_C = \{w: \secmom{w} \le C \}$ for some budget $C$. Below we discuss the latter which admits cleaner results, i.e., we want to solve
\begin{equation}\label{eq:minimax}
    \argmin_{w\in\Wcal_C} \sup_{f\in\Fcal_h} \ell(w, f).
\end{equation}

To solve the above minimax problems, we apply the standard \emph{no-regret + best-response framework} \citep{freund1999adaptive} (see \citep{miryoosefi2019reinforcement} for an example), where a $w$ player and an $f$ player alternately play a sequence of $w_1, w_2, \ldots, w_T \in \Wcal_C$ and $f_1, f_2, \ldots, f_T \in \Fcal_h$, respectively, and we finally output $\wh_h^\nn = \frac{1}{T}\sum_{t=1}^T w_t$. 

\paragraph{No regret on $w$ and best response on $f$} 
In this case, after each $w_t \in \Wcal_C$ is chosen, $f_t$ is chosen as the best response to maximize $\ell(w_t, \cdot)$, which requires a standard linear-optimization oracle over $\Fcal_h$ \citep{dann2018oracle}: \footnote{Here are some examples: if $\Fcal_h$ is an RKHS ball, the oracle has a kernel witness-function form; if $\Fcal_h$ is a neural network class, it corresponds to training a discriminator/value network to maximize the following weighted empirical objective; if $\Fcal_h$ is finite, it is just choosing the function in the finite class with largest signed score.}
\[
f_t\in\mathrm{BR}(w_t):=\argmax_{f\in\Fcal_h}\left\{\frac1n\sum_{i=1}^n w_t^\ii f(s_h^\ii,a_h^\ii)-\frac1n\sum_{i=1}^n \wh_{h-1}^\ii f(s_h^\ii,\pi)\right\}.
\]
From the point of view of the $w$ player, it receives a loss $\ell_t^w(\cdot) = \ell(\cdot, f_t)$, which is a convex function and we have access to the gradient information, so standard online learning algorithm such as projected OGD \citep{hazan2016introduction} can be applied to obtain a regret guarantee:
\[
w_{t+1}=\Pi_{\Wcal_C}[w_t-\eta\nabla_w \ell(w_t,f_t)],\quad \nabla_w\ell(w_t,f_t)=\frac1n\Big(f_t(s_h^{(1)},a_h^{(1)}),\ldots,f_t(s_h^{(n)},a_h^{(n)})\Bigr).
\]

\paragraph{No regret on $f$ and best response on $w$} We can also swap the roles of $f$ and $w$. 
For each fixed $f_t$, we solve $w_t\in\arg\min_{w\in\Wcal_C}\ell(w,f_t)$. Since $\Wcal_C$ is a Euclidean ball and $\ell(\cdot,f_t)$ is linear, this best response has a closed form:
\[
w_t=-C\frac{f_t\evn}{\|f_t\evn\|_\nn},\quad f_t\evn=\Bigl(f(s_h^{(1)},a_h^{(1)}),\ldots,f(s_h^{(n)},a_h^{(n)})\Bigr)\in\RR^n,
\]
for $f_t\evn\ne0$; and if $f_t\evn=0$, then any $w_t\in\Wcal_C$ is a best response. From the point of view of the $f$ player, it receives a loss $\ell_t^f(\cdot) = \ell(w_t, \cdot)$, which is linear in $f$. Similar to before, if $f\in \Fcal_h$ is parametrized in a differentiable manner, we would need an online learning oracle over $\Fcal_h$ for the reward sequences $f\mapsto\ell_t^ff()$.
\section{Connection with Monte-Carlo Estimation and Importance Sampling} \label{app:mc}

In the special case where the data is \textit{on-policy}, i.e., the trajectories are sampled from $\pi$ itself, it is well known that the Monte-Carlo (MC) estimator, $\tfrac{1}{n}\sum_{i=1}^n \sum_{h=1}^H r_h^\ii$, is a simple and effective method that achieves dimension-free error guarantees, since the analysis simply applies concentration arguments to the i.i.d.~random returns. This also shows the looseness of standard OPE analyses that depends on $\log|\Fcal|$ \citep{antos2008learning, chen2019information}, as those guarantees do not automatically eliminate $\log|\Fcal|$ when data is on policy. 
In contrast, \alg immediately recovers MC in this setting if we simply choose $\wh_h^\ii \equiv 1$. (Note that this choice is not necessarily minimizing the 2nd moment of $\wh_h^\nn$.) The matching loss $\lossmatch_h(\wh_h^\nn; \wh_{h-1}^\nn)$ is $0$ under such weights,\footnote{This holds when $\pi$ is deterministic; when $\pi$ is stochastic, it still holds if we consider the IS variant of our algorithm, as described below.} and Theorem~\ref{thm:main} is dimension free. 

More generally, when the behavior policy $\pi_D$ is stochastic, a variant of our algorithm can also recover step-wise importance sampling, the generalization of MC to the off-policy setting. In particular, if we change the target embedding $\frac{1}{n}\sum_i \wh_{h-1}^\ii f(s_{h}^\ii, \pi)$ in Eq.\eqref{eq:loss-match} to\footnote{This resembles how targets are set up in TD, such as expected SARSA vs.~its off-policy version \citep{sutton2018reinforcement}.}
$$
\frac{1}{n}\sum_i \wh_{h-1}^\ii f(s_{h}^\ii, a_h^\ii) \frac{\pi(a_h^\ii|s_h^\ii)}{\pi_D(a_h^\ii|s_h^\ii)},
$$
which produces the same target embedding after marginalizing over the action randomness, then setting $\wh_h^\ii = \prod_{h'=1}^h \frac{\pi(a_{h'}^\ii|s_{h'}^\ii)}{\pi_D(a_{h'}^\ii|s_{h'}^\ii)}$ ensures $\lossmatch_h(\wh_h^\nn; \wh_{h-1}^\nn) = 0$. 

This example also demonstrates the role of $\Fcal$ and the importance of 2nd-moment minimization: if we only care about minimizing the matching loss $\lossmatch_h$ but not the size of $\wh_h^\nn$, we can always set $\lossmatch_h$ to $0$ by setting $\wh_h^\nn$ to be the cumulative importance weights, which generally incurs exponential-in-horizon variance. This is avoided when structured $\Fcal$ admits many choices of $\wh_h^\nn$ that equally minimize the matching loss, and we can choose the smallest one to significantly dampen the weight size.
\section{Proofs and Additional Results of Section~\ref{sec:weight}} \label{app:weight}

\subsection{Auxiliary functionals and representers}
In this section, we introduce several auxiliary concepts that will be used in the subsequent proofs and in our understanding of the learned weights. The main idea is to view the population and empirical moment-matching problems through the lens of Hilbert-space linear functionals. This viewpoint will be useful for two reasons. First, it explains why the population exact-matching problem is feasible under the bounded generalized leverage assumption (Assumption~\ref{asm:leverage}). Second, it allows us to compare the empirical weights $\{\wh_h^\nn\}_h$ with their population counterparts $\{\wstar_h\}_h$ by representing both of them inside the same Hilbert space.

Recall from Definition~\ref{def:hilbert} that the Hilbert space $\Hilb_h$ is the closure of $\mathrm{span}(\Fcal_h)$ under the $d_h^D$-weighted $L_2$ norm, equipped with the inner product $\langle f,g\rangle_{d_h^D}=\EE_{d_h^D}[fg]$. For any $f\in\Hilb_h$, let $P^\pi f$ denote the one-step transition of $f$ under the target policy, i.e., $(P^\pi f)(s,a)=\EE_{s'\sim P(\cdot|s,a)}[f(s',\pi)]$. Given the population weight, define the propagated signed measure
\[
\nu_h^\star:= P^\pi(\wstar_{h-1}\cdot d_{h-1}^D).
\]
Equivalently, $\nu_h^\star$ is the signed measure satisfying, for every $f\in\Hilb_h$,
\[
\EE_{(s_h,a_h)\sim \nu_h^\star}[f(s_h,a_h)]=\EE_{(s_{h-1},a_{h-1})\sim d_{h-1}^D}[\wstar_{h-1}(s_{h-1},a_{h-1})\cdot (P^\pi f)(s_{h-1},a_{h-1})].
\]
We call $\nu_h^\star$ a signed measure because $\wstar_{h-1}$ is not required to be nonnegative. Therefore $\wstar_{h-1}d_{h-1}^D$, and hence its pushforward under $P^\pi$, need not be a probability distribution.

\begin{definition}[Population and empirical target functionals]\label{def:functional}
    For each $h\in[H]$, define the population target functional $\Lambda_h^\star:\Hilb_h\to\RR$ by
    \[
    \Lambda_h^\star(f):=\EE_{\nu_h^\star}[f]=\EE_{d_{h-1}^D}[w_{h-1}^\star\cdot (P^\pi f)].
    \]
    Given the empirical weights $\wh_{h-1}^\nn$, define the empirical target functional $\hatLambda:\Hilb_h\to\RR$ by
    \[
    \hatLambda_h(f):=\frac1n\sum_{i=1}^n\wh_{h-1}^\ii f(s_h^\ii,\pi).
    \]
\end{definition}
Thus, $\Lambda_h^\star$ is the population moment that the level-$h$ weight should match, while $\hatLambda_h$ is its empirical analogue induced by the already-learned weights at level $h-1$. The next proposition shows that these linear functionals admit Riesz representers in $\Hilb_h$. It also shows that, whenever empirical 0 matching loss is feasible (which is shown in Proposition~\ref{prop:empirical-exact} as a high-probability corollary under population 0 matching loss), the least-2nd-moment empirical weight vector can be viewed as the sample evaluation of some function in $\Hilb_h$.

\begin{proposition}[Existence of representers]\label{prop:representer}
Under Assumption~\ref{asm:leverage}, for each $h\in[H]$, we have:
\begin{enumerate}
    \item [(i)] The population target functional $\Lambda_h^\star$ is a continuous linear functional on $\Hilb_h$. Hence, the $\wstar_h$ defined in Definition~\ref{def:pop} satisfies $\wstar_h\in\Hilb_h$, and $\langle \wstar_h,f\rangle_{d_h^D}=\Lambda_h^\star(f)$ for all $f\in\Hilb_h$. Moreover, $\wstar_h$ is the unique least-2nd-moment solution among all $L_2(d_h^D)$ functions $w$ satisfying $\langle w,f\rangle_{d_h^D}=\Lambda_h^\star(f)$ for all $f\in\Hilb_h$.

    \item [(ii)] The empirical functional $\hatLambda_h$ is also a continuous linear functional on $\Hilb_h$. Hence there exists a unique auxiliary representer $\wt_h\in\Hilb_h$ such that $\langle\wt_h,f\rangle_{d_h^D}=\hatLambda_h(f)$ for all $f\in\Hilb_h$.

    \item [(iii)] Suppose empirical 0 matching loss is feasible at level $h$, and let $\wh_h^\nn$ be the least-2nd-moment weight vector. Then, there exists a lift representer $\wb_h\in\Hilb_h$ such that $\wb_h\evn=\wh_h^\nn$. Consequently, $\frac1n\sum_{i=1}^n\wb(s_h^\ii,a_h^\ii)f(s_h^\ii,a_h^\ii)=\hatLambda_h(f)$ for all $f\in\Hilb_h$.
\end{enumerate}
\end{proposition}

\begin{proof}
    For (i), linearity of $\Lambda_h^\star$ follows directly from its definition (Definition~\ref{def:functional}), so we need to show its continuity. By Assumption~\ref{asm:leverage}, every $f\in\Hilb_h$ satisfies $|f(s,a)|\le\sqrt{\kappa_h}\|f\|_{2,d_h^D}$. Therefore, $|f(s,\pi)|\le\sqrt{\kappa_h}\|f\|_{2,d_h^D}$ as well. It follows that
    \[
    |\Lambda_h^\star(f)|\le\EE_{d_{h-1}^D}[|\wstar_{h-1}|\cdot |P^\pi f|]\le\sqrt{\kappa_h}\cdot \EE_{d_{h-1}^D}[|\wstar_{h-1}|]\|f\|_{2,d_h}^D.
    \]
    Using Cauchy-Schwarz, $\EE_{d_{h-1}^D}[|\wstar_{h-1}|]\le\|\wstar_{h-1}\|_{2,d_{h-1}^D}$. Hence $\Lambda_h^\star$ is bounded, and therefore continuous, on $\Hilb_h$. Therefore, by the Riesz representation theorem, there exists a unique $\wstar\in\Hilb_h$ such that $\langle\wstar,f\rangle_{d_h^D}=\Lambda_h^\star(f)$ for all $f\in\Hilb_h$. This already proves that the population 0 matching loss is feasible (Proposition~\ref{prop:population-exact}): the feasible set is nonempty because $\wstar$ itself satisfies all the required moment equations. Since $\wstar_h$ defined in Definition~\ref{def:pop} minimizes the matching loss, we have $\wstar_h=\wstar$.

    It remains to justify the minimum-norm claim. Let $w\in L_2(d_h^D)$ be any other function satisfying the same moment equations, namely $\langle w,f\rangle_{d_h^D}=\Lambda_h^\star(f)$ for all $f\in\Hilb_h$. Then $w-\wstar_h$ is orthogonal to every element of $\Hilb_h$. Since $\wstar_h\in\Hilb_h$, we have $\langle w-\wstar_h,\wstar_h\rangle_{d_h^D}=0$. Therefore,
    \[
    \|w\|_{2,d_h^D}^2=\|\wstar_h\|_{2,d_h^D}^2+\|w-\wstar_h\|_{2,d_h^D}^2\ge \|\wstar_h\|_{2,d_h^D}^2.
    \]
    Thus $\wstar_h$ is the unique minimum-norm solution, up to $d_h^D$-almost-sure equivalence.

    For (ii), the argument is similar. The empirical functional $\hatLambda_h$ is linear in $f$. Moreover, under Assumption~\ref{asm:leverage}, we have
    \[
    |\hatLambda_h(f)|\le\frac1n\sum_{i=1}^n|\wh_{h-1}^\ii|\cdot |f(s_{h}^\ii,\pi)|\le\sqrt{\kappa_h}\biggl(\frac1n\sum_{i=1}^n|\wh_{h-1}^\ii|\biggr)\|f\|_{2,d_h^D}.
    \]
    Therefore $\hatLambda_h$ is continuous on $\Hilb_h$. Applying the Riesz representation theorem again yields a unique $\wt_h\in\Hilb_h$ such that $\langle\wt_h,f\rangle_{d_h^D}=\hatLambda_h(f)$ for all $f\in\Hilb_h$.

    For (iii), let $\Hilb_h\evn:=\{f\evn:f\in\Hilb_h\}\subseteq\RR^n$ be the sample-evaluation subspace of $\Hilb_h$. The empirical matching constraints only depend on the projection of the weight vector onto $\Hilb_h\evn$, because every constraint is tested against some vector $f\evn\in\Hilb_h\evn$. Decompose the least-2nd-moment solution as $\wh_h^\nn=v+v_\bot$, where $v\in\Hilb_h\evn$ and $v_\bot\in(\Hilb_h\evn)^\bot$. Since $v_\bot$ is orthogonal to every $f\evn\in\Hilb_h\evn$, replacing $\wh_h^\nn$ by $v$ preserves all empirical matching constraints. If $v_\bot\ne0$, this replacement strictly decreases $\|\wh_h^\nn\|_\nn$, contradicting the least-2nd-moment choice of $\wh_h^\nn$. Therefore, $v_\bot=0$ and $\wh_h^\nn\in\Hilb_h\evn$.

    By the definition of $\Hilb_h\evn$, there exists $\wb_h\in\Hilb_h$ such that $\wb_h\evn=\wh_h^\nn$. Since by our assumption, $\wh_h^\nn$ achieves empirical 0 matching loss, we obtain $\frac1n\sum_{i=1}^n\wb(s_h^\ii,a_h^\ii)f(s_h^\ii,a_h^\ii)=\hatLambda_h(f)$ for all $f\in\Hilb_h$.
\end{proof}

The above three representers play different roles. The population weight $\wstar_h$ is the object we ultimately want to track. The auxiliary representer $\wt_h$ converts the empirical target functional into the population geometry of $\Hilb_h$. The lift representer $\wb_h$ converts the empirical sample weights back into a function so that empirical pairings can be compared with population pairings. In the linear setting (Definition~\ref{def:linear}) with both population and empirical 0 matching loss, these objects reduce to the familiar forms that correspond to the feature dynamical system in Eq.~\eqref{eq:lds}; see the following proposition.

\begin{proposition}[Linear forms of the three representers]\label{prop:representer-linear}
    Suppose $\Hilb_h$ is linear in a feature map $\phi:\Scal_h\times\Acal\to\RR^\dm$. Write $\phi_h:=\phi(s_h,a_h)$, $\phi_h^\pi:=\phi(s_h,\pi)$. Assume that the population covariance matrix $\Sigma_h:=\EE_{d_h^D}[\phi_h(\phi_h)^\top]$ is invertible. For the empirical lift, additionally assume that $\hatSigma_h:=\frac1n\sum_{i=1}^n\phi_h^\ii(\phi_h^\ii)^\top$ is invertible, where $\phi_h^\ii:=\phi(s_h^\ii,a_h^\ii)$.
    
    Consider the feature dynamical system $\{\psi_h\}_h$ defined in Eq.~\eqref{eq:lds}, i.e., $\psi_1:=\EE_D[\phi_1^\pi]$, and for $h\ge 1$,
    \[
    \psi_{h+1}:=B_h^\pi\psi_h,\quad B_h^\pi:=(\Sigmacr_h)^\top\Sigma_h^{-1},\quad\Sigmacr_h:=\EE_D[\phi_h(\phi_{h+1}^\pi)^\top].
    \]
    Also define the empirical target vector $\widehat{\psi}_h:=\frac1n\sum_{i=1}^n\wh_{h-1}^\ii\phi(s_h^\ii,\pi)$. Then the three representers in Proposition~\ref{prop:representer} take the following forms:
    \[
    \wstar_h(s,a)=\phi(s,a)^\top\Sigma_h^{-1}\psi_h,\quad\wt_h(s,a)=\phi(s,a)^\top\Sigma_h^{-1}\widehat{\psi}_h,\quad \wb_h(s,a)=\phi(s,a)^\top\hatSigma_h^{-1}\widehat{\psi}_h.
    \]
\end{proposition}
\begin{proof}
    We first identify the population representer $\wstar_h$ via induction on $h$. More precisely, we prove jointly that, for every $h$, the population target functional satisfies
    \[
    \Lambda_h^\star(f)=\psi_h^\top\theta,\quad\forall f\in\Hilb_h\text{ such that }f(s,a)=\phi(s,a)^\top\theta,
    \]
    and that its Riesz representer is $\wstar_h(s,a)=\phi(s,a)^\top\Sigma_h^{-1}\psi_h$.

    For $h=1$, since $\wstar_0\equiv1$, for every $f(s,a)=\phi(s,a)^\top\theta$,
    \[
    \Lambda_1^\star(f)=\EE_D[f(s_1,\pi)]=\EE_D[\phi_1^\pi]^\top\theta=\psi_1^\top\theta.
    \]
    The Riesz representer $\wstar_1$ therefore satisfies, for all $\theta$, $\EE_D[\wstar_1(s_1,a_1)\cdot\phi_1^\top\theta]=\psi_1^\top\theta$. Since $\wstar_1\in\Hilb_1$ by Proposition~\ref{prop:representer}(i), we can write $\wstar_1(s,a)=\phi(s,a)^\top\alpha_1$. This becomes $\alpha_1^\top\Sigma_1\theta=\psi_1^\top\theta$ for all $\theta$. Hence $\alpha_1=\Sigma_1^{-1}\psi_1$ for invertible $\Sigma_1$, and therefore $\wstar_1(s,a)=\phi(s,a)^\top\Sigma_1^{-1}\psi_1$.

    Now assume the two claims hold at level $h$. By the definition of the population target functional,
    \begin{align*}
        \Lambda_{h+1}^\star(f)&=\EE_D[\wstar_h(s_h,a_h)f(s_{h+1},\pi)]\\
        &=\EE_D[\phi_h^\top\Sigma_h^{-1}\psi_h\cdot (\phi_{h+1}^\pi)^\top\theta]\\
        &=\psi_h^\top\Sigma_h^{-1}\Sigmacr_h\theta\\
        &=\psi_{h+1}^\top\theta,
    \end{align*}
    where the first equality is due to the inductive hypothesis on $\wstar_h$, the second equality leverages the definition of $\Sigmacr_h$ and the third equality is due to $\psi_{h+1}=(\Sigmacr_h)^\top\Sigma_h^{-1}\psi_h$.

    Similarly, since $\Lambda_{h+1}^\star(f)=\psi_{h+1}^\top\theta$ for all $f(s,a)=\phi(s,a)^\top\theta$, the representer $\wstar_{h+1}$ satisfies
    \[
    \EE_D[\wstar_{h+1}(s_{h+1},a_{h+1})\cdot (\phi_{h+1}^\top\theta)]=\psi_{h+1}^\top\theta,\quad\forall \theta.
    \]
    Since $\wstar_{h+1}\in\Hilb_{h+1}$, we can write $\wstar_{h+1}(s,a)=\phi(s,a)^\top\alpha_{h+1}$. The above matching equation becomes $\alpha_{h+1}^\top\Sigma_{h+1}\theta=\psi_{h+1}^\top\theta$ for all $\theta$ and thus $\alpha_{h+1}=\Sigma_{h+1}^{-1}\psi_{h+1}$ for invertible $\Sigma_{h+1}$, indicating $\wstar_{h+1}(s,a)=\phi(s,a)^\top\Sigma_{h+1}^{-1}\psi_{h+1}$. This closes the induction.

    The proof for $\wt_h$ does not require induction. For any $f(s,a)=\phi(s,a)^\top\theta$, $\hatLambda_h(f)=\widehat{\psi}_h^\top\theta$ by the definition of $\widehat{\psi}_h$. Since $\wt_h\in\Hilb_h$ by Proposition~\ref{prop:representer}(ii), its linear form $\wt_h(s,a)=\phi(s,a)^\top\widetilde{\alpha}_h$ satisfies $\Sigma_h\widetilde{\alpha}_h=\widehat{\psi}_h$, and hence $\wt_h(s,a)=\phi(s,a)^\top\Sigma_h^{-1}\widehat{\psi}_h$.

    Finally, for the empirical lift $\wb_h$, empirical 0 matching loss gives
    \[
    \frac1n\sum_{i=1}^n\wb_h(s_h^\ii,a_h^\ii)\phi_h^\ii=\widehat{\psi}_h.
    \]
    Since $\wb_h\in\Hilb_h$ by Proposition~\ref{prop:representer}(iii), its linear form $\wb_h(s,a)=\phi(s,a)^\top\bar{\alpha}_h$ satisfies $\hatSigma_h\bar{\alpha}_h=\widehat{\psi}_h$. Since $\hatSigma_h$ is invertible, it directly leads to $\wb_h(s,a)=\phi(s,a)^\top\hatSigma_h^{-1}\widehat{\psi}_h$.
\end{proof}

\subsection{Proofs of Propositions~\ref{prop:w-in-span},~\ref{prop:population-exact},~\ref{prop:connection-linear},~\ref{prop:bound-weight},~\ref{prop:completeness} and~\ref{prop:coverage}}

\begin{proof}[Proof of Proposition~\ref{prop:w-in-span}]
    For the second argument that for the least-2nd-moment $\wh_h^\nn$, there exists some $\wb_h\in\Hilb_h$ such that $\wb_h\evn=\wh_h^\nn$, the proof is exactly same as stated in Proposition~\ref{prop:representer}(iii). Note that in there we do not use the empirical 0 matching loss condition for $\wh_h^\nn\in\Hilb_h\evn$ argument.

    For the first argument that the least-2nd-moment population weight $\wstar_h\in\Hilb_h$, we can adopt a similar strategy to prove (which does not require Assumption~\ref{asm:leverage}). For any $w\in L_2(d_h^D)$ and any $f\in\Fcal_h\subseteq\Hilb_h$, we have
    \[
    \EE_{d_h^D}[w\cdot f]=\EE_{d_h^D}[(\Pi_{\Hilb_h }w)\cdot f],
    \]
    where $\Pi_{\Hilb_h}$ is the $L_2(d_h^D)$-projection onto $\Hilb_h$. This is because $w-\Pi_{\Hilb_h}w$ is orthogonal to $\Hilb_h$. Therefore, the population matching loss $\Lcal_h(w;\wstar_{h-1})$ depends on $w$ only through its projection $\Pi_{\Hilb_h}w$, which means $\Lcal_h(w;\wstar_{h-1})=\Lcal_h(\Pi_{\Hilb_h}w;\wstar_{h-1})$. Moreover,
    \[
    \|w\|_{2,d_h^D}^2=\|\Pi_{\Hilb_h}w\|_{2,d_h^D}^2+\|w-\Pi_{\Hilb_h}w\|_{2,d_h^D}^2.
    \]
    Thus, if $\wstar_h$ is chosen among the exact minimizers of $\Lcal_h(\cdot;\wstar_{h-1})$ to have least-2nd-moment, then its orthogonal component outside $\Hilb_h$ must be $0$. Otherwise replacing $\wstar_h$ by $\Pi_{\Hilb_h}\wstar_h$ would preserve the loss and strictly reduce its 2nd-moment. Hence, $\wstar_h\in\Hilb_h$.
\end{proof}

\begin{proof}[Proof of Proposition~\ref{prop:population-exact}]
    Proposition~\ref{prop:representer}(i) already implies this claim.
\end{proof}

\begin{proof}[Proof of Proposition~\ref{prop:connection-linear}]

    For a function $f\in\Hilb_{h+1}$, by Definition~\ref{def:backup}, its one-step backup is given by $\bp_hf=\arg\min_{g\in\Hilb_h}\EE_{D}[(g(s_h,a_h)-f(s_{h+1},\pi))^2]\in\Hilb_h$. That is, $\bp_h f=\Pi_{\Hilb_h}P^\pi f$, where the projection step preserves linearity in $\phi$ of the function. Therefore, if $f(s,a)=\phi(s,a)^\top\theta_{h+1}$, we can write $(\bp_h f)(s,a)=\phi(s,a)^\top\theta_h$ for some $\theta_h$ solved by the above least-square problem. Specifically, the normal equation for this least-square problem is
    \[
    \Sigma_h\theta_h=\Sigmacr_h\theta_{h+1},
    \]
    where $\Sigma_h$ and $\Sigmacr_h$ are defined in Eq.~\eqref{eq:lds}. Thus, for invertible $\Sigma_h$, we have $\theta_h=\Sigma_h^{-1}\Sigmacr_h\theta_{h+1}=(B_h^\pi)^\top\theta_{h+1}$, since $B_h^\pi=(\Sigmacr_h)^\top\Sigma_h^{-1}$. This proves the first claim.

    For the multi-step bound, let $f\in\Hilb_h$ have coefficient $\theta_h$. Repeatedly applying the previous identity gives that $\bp_t\bp_{t+1}\cdots\bp_{h-1}f$ has coefficient
    \[
    (B_t^\pi)^\top(B_{t+1}^\pi)^\top\cdots(B_{h-1}^\pi)^\top\theta_h=(B_{h-1}^\pi\cdots B_{t+1}^\pi B_t^\pi)^\top\theta_h.
    \]
    Therefore, for any $(s_t,a_t)\in\Scal_t\times\Acal$,
    \begin{align*}
    (\bp_t\bp_{t+1}\cdots\bp_{h-1}f)(s_t,a_t)
    &=\phi(s_t,a_t)^\top(B_{h-1}^\pi\cdots B_{t+1}^\pi B_t^\pi)^\top\theta_h\\
    &=\left\langle\Sigma_h^{-1/2}B_{h-1}^\pi\cdots B_{t+1}^\pi B_t^\pi\phi(s_t,a_t),\Sigma_h^{1/2}\theta_h\right\rangle.
    \end{align*}
    By Cauchy-Schwarz,
    \[
    \left|(\bp_t\bp_{t+1}\cdots\bp_{h-1}f)(s_t,a_t)\right|\le\left\|\Sigma_h^{-1/2}B_{h-1}^\pi\cdots B_{t+1}^\pi B_t^\pi\phi(s_t,a_t)\right\|_2\cdot\|\Sigma_h^{1/2}\theta_h\|_2.
    \]
    Since $\|\Sigma_h^{1/2}\theta_h\|_2=\|f\|_{2,d_h^D}\le\|f\|_\infty$, we take the supremum over $(s_t,a_t)$ and then over all $f$ with $\|f\|_\infty\le 1$ yields the bound for $\rho_{t:h}$.
\end{proof}

\begin{proof}[Proof of Proposition~\ref{prop:bound-weight}]
    By Proposition~\ref{prop:population-exact}, we can achieve population 0 matching loss under Assumption~\ref{asm:leverage}, i.e., for any $f\in\Hilb_h$
    \[
    \EE_{d_h^D}[\wstar_h f]=\EE_{d_{h-1}^D}[\wstar_{h-1}\cdot (P^\pi f)]=\EE_{d_{h-1}^D}[\wstar_{h-1}\cdot (\bp_{h-1} f)],
    \]
    where the last step is because $\wstar_{h-1}\in\Hilb_{h-1}$ (Proposition~\ref{prop:representer}(i)), so that we can replace $P^\pi f$ by its $L_2(d_{h-1}^D)$-projection $\bp_{h-1} f$  without changing the overall inner product. Apply the same argument recursively gives
    \[
    \EE_{d_h^D}[\wstar_h f]=\EE_{d_1^D}[\wstar_1\cdot (\bp_1\bp_2\cdots\bp_{h-1}f)].
    \]
    At level $1$, $\Lcal_1=0$ implies $\EE_{d_1^D}[\wstar_1\cdot(\bp_1\bp_2\cdots\bp_{h-1} f)]=\EE_{d_0^D}[\wstar_0\cdot (P^\pi\bp_1\bp_2\cdots\bp_{h-1}f)]$. Since $\wstar_0\equiv 1$ and $(s_0,a_0)$ is fixed, the RHS becomes $\EE_D[(\bp_1\bp_2\cdots\bp_{h-1}f)(s_1,\pi)]$ and hence we have
    \[
    \bigl|\EE_{d_h^D}[\wstar_h f]\bigr|\le\|\bp_1\bp_2\cdots\bp_{h-1}f\|_\infty\le\rho_{1:h}\|f\|_\infty,
    \]
    by the definition of $\rho_{1:h}$. This proves the claim.
\end{proof}

\begin{proof}[Proof of Proposition~\ref{prop:completeness}]
    We first want to claim that Bellman completeness over the original function class $\Fcal_h$ ($\Tcal^\pi f\in\Fcal_{h-1},\forall f\in\Fcal_h$) implies the backup closure property of the Hilbert space $\Hilb_h$ ($P^\pi f\in\Hilb_{h-1},\forall f\in\Hilb_h$). This is because we assume $\mathbf{0}\in\Fcal_h$ and hence $r=\Tcal^\pi\mathbf{0}\in\Fcal_{h-1}$; by Bellman completeness, $\Tcal^\pi f\in\Fcal_{h-1}$ for all $f\in\Fcal_h$. Thus, $P^\pi f=\Tcal^\pi f-r\in\mathrm{span}(\Fcal_{h-1})$ for all $f\in\Fcal_h$. By linearity, $P^\pi\mathrm{span}(\Fcal_h)\subseteq\mathrm{span}(\Fcal_{h-1})$, and by closure this extends to $P^\pi\Hilb_h\subseteq\Hilb_{h-1}$.
    
    For (i), we prove by induction that, for every $f\in\Hilb_h$, $\EE_{d_h^D}[\wstar_h f]=\EE_{d_h^\pi}[f]$. For $h=1$, population $0$ matching loss (Proposition~\ref{prop:population-exact}) gives $\EE_{d_1^D}[\wstar_1f]=\EE_{d_0^D}[\wstar_0\cdot (P^\pi f)]$. Since $\wstar_0\equiv 1$ and $d_0^D=d_0^\pi$ is the fixed initial state-action pair, by Bellman flow equation, $\EE_{d_0^D}[\wstar_0\cdot (P^\pi f)]=\EE_{d_1^\pi} f$. Thus the claim holds for $h=1$.

    Assume that the induction holds at level $h-1$, and take any $f\in\Hilb_h$. By our claim before, we know that $P^\pi f\in\Hilb_{h-1}$. Therefore,
    \[
    \EE_{d_h^D}[\wstar_h f]=\EE_{d_{h-1}^D}[\wstar_{h-1}\cdot(P^\pi f)]=\EE_{d_{h-1}^\pi}[P^\pi f]=\EE_{d_h^\pi f},
    \]
    where the first equality is due to population $0$ matching loss, the second equality is by inductive hypothesis, and the last equality uses Bellman flow equation that $d_h^\pi=P^\pi d_{h-1}^\pi$ for $h\in[H]$. This proves (i).

    For (ii), by Proposition~\ref{prop:w-in-span}, $w_h^\star\in\Hilb_h$. Moreover, by (i), for any $f\in\Hilb_h$, since $d_h^\pi\ll d_h^D$,
    \[
    \EE_{d_h^D}[w_h^\star f]=\EE_{d_h^\pi}[f]=\EE_{d_h^D}\left[\frac{d_h^\pi}{d_h^D}f\right].
    \]
    Therefore, $\left\langle w_h^\star-\frac{d_h^\pi}{d_h^D},f\right\rangle_{d_h^D}=0$ for any $f\in\Hilb_h$. Since $w_h^\star\in\Hilb_h$, it ix exactly the $L_2(d_h^D)$-orthogonal projection of the MIS ratio $\frac{d_h^\pi}{d_h^D}$ onto $\Hilb_h$, proving (ii).
        
    For (iii), let $f\in\Hilb_h$. Since $P^\pi f\in\Hilb_{h-1}$, by the definition of the projected backup operator $\bp_{h-1}$, we have that $\bp_{h-1}f=\Pi_{\Hilb_{h-1}}P^\pi f=P^\pi f$ as an element of $L_2(d_{h-1}^D)$.

    Under Assumption~\ref{asm:leverage}, equality in $L_2(d_{h-1}^D)$ implies pointwise equality for all $(s,a)\in\Scal_{h-1}\times\Acal$, because $|g(s,a)|\le\sqrt{\kappa_{h-1}}\|g\|_{2,d_{h-1}^D}$ for all $g\in\Hilb_{h-1}$. Hence,
    \[
    (\bp_{h-1} f)(s,a)=(P^\pi f)(s,a)=\EE_\pi[f(s_h,\pi)\mid s_{h-1}=s,a_{h-1}=a].
    \]
    Iterating this identity from $t$ to $h-1$ yields
    \[
    (\bp_t\bp_{t+1}\cdots\bp_{h-1} f)(s,a)=\EE_\pi[f(s_h,\pi)\mid s_t=s,a_t=a].
    \]
    
    For (iv), consider any $f\in\Hilb_h$ with $\|f\|_\infty\le 1$, by the multi-step identity in (iii),
    \[
    |(\bp_t\bp_{t+1}\cdots\bp_{h-1} f)(s,a)|=|\EE_\pi[f(s_h,\pi)\mid s_t=s,a_t=a]|\le\|f\|_\infty\le 1.
    \]
    Therefore, by Definition~\ref{def:backup}, $\rho_{t:h}=\sup_{f\in\Hilb_h:\|f\|_\infty\le 1}\|\bp_t\bp_{t+1}\cdots\bp_{h-1}f\|_\infty\le 1$.
\end{proof}

\begin{proof}[Proof of Proposition~\ref{prop:coverage}]
    We first prove the general case with Bellman completeness. By Proposition~\ref{prop:completeness}(1), for any $f\in\Hilb_h$, we have $\EE_{d_h^\pi}[f]=\EE_{d_h^D}[\wstar_h f]=\langle\wstar_h,f\rangle_{d_h^D}$ for all $f\in\Hilb_h$. By Cauchy-Schwarz inequality, we have
    \[
    |\EE_{d_h^\pi}[f]|=|\langle \wstar_h,f\rangle_{d_h^D}|\le\|\wstar_h\|_{2,d_h^D}\|f\|_{2,d_h^D}.
    \]
    Moreover, the above inequality becomes equality if we take $f=\wstar_h\in\Hilb_h$. Therefore we get
    \[
    \|\wstar_h\|_{2,d_h^D}=\sup_{f\in\Hilb_h,f\ne 0}\frac{|\EE_{d_h^\pi}[f]|}{\|f\|_{2,d_h^D}}=\sup_{f\in\Hilb_h,f\ne 0}\frac{|\EE_{d_h^\pi}[f]|}{\sqrt{\EE_{d_h^D}[f^2]}}.
    \]

    Now we consider linear case. By Proposition~\ref{prop:representer-linear}, we have $\wstar_h(s,a)=\phi(s,a)^\top\Sigma_h^{-1}\psi_h$. Therefore,
    \[
    \|\wstar_h\|_{2,d_h^D}^2=\EE_{d_h^D}\bigl[(\phi_h^\top\Sigma_h^{-1}\psi_h)^2\bigr]=\psi_h^\top\Sigma_h^{-1}\psi_h=\|\psi_h\|_{\Sigma_h^{-1}}^2,
    \]
    i.e., $\|\wstar_h\|_{2,d_h^D}=\|\psi_h\|_{\Sigma_h^{-1}}$, which does NOT require Bellman completeness.
    
    Under linear Bellman completeness, by Proposition~\ref{prop:completeness}(1), we have 
    \[
    \EE_{d_h^D}[\wstar_h f_\theta]=\EE_{d_h^\pi}[f_\theta]
    \]
    holds for all $\theta$, where $f_\theta(s,a)=\phi(s,a)^\top\theta$. Recall $\wstar_h(s,a)=\phi(s,a)^\top\Sigma_h^{-1}\psi_h$. The LHS of the above equation becomes $\psi_h^\top\theta$, since $\EE_{d_h^D}[\phi_h\phi_h^\top]$ cancels with $\Sigma_h^{-1}$; while the RHS of the above equation becomes $\EE_{d_h^\pi}[\phi_h]^\top\theta$. Therefore, $\psi_h=\EE_{d_h^\pi}[\phi_h]$ and hence $\|\wstar_h\|_{2,d_h^D}=\|\EE_{d_h^\pi}[\phi_h]\|_{\Sigma_h^{-1}}$. This proves the claim.
\end{proof}

\subsection{Concentration arguments}
We state two standard concentration arguments, both obtained by Bernstein's inequality together with a covering argument. For each $f$, let $N_h$ be the $1/8$-covering number of the unit ball of $\Hilb_h$ under $\|\cdot\|_{2,d_h^D}$, i.e., $\Bcal_h:=\{f\in\Hilb_h:\|f\|_{2,d_h^D}\le 1\}$. Define
\begin{equation}\label{eq:concentration1}
    \chi_h:=\frac{4}{3}\left(\sqrt{\frac{2\kappa_h\log(8HN_h^2/\delta)}{n}}+\frac{4\kappa_h\log(8HN_h^2/\delta)}{3n}\right),
\end{equation}
and, for $1\le t<h\le H$,
\begin{equation}\label{eq:concentration2}
    \tau_{t:h}=\frac{4}{3}(\rho_{t:h}+\rho_{t+1:h})\left(\sqrt{\frac{2\kappa_h\log(8H^2N_tN_h/\delta)}{n}}+\frac{2\sqrt{\kappa_t\kappa_h}\log(8H^2N_tN_h/\delta)}{3n}\right),
\end{equation}
where $\rho_{t:h}$ is defined in Definition~\ref{def:backup}. Now, for two elements $u,v\in\Hilb_h$, we define the sample inner product and norm as
\[
\langle u\evn,v\evn\rangle_\nn:=\frac1n\sum_{i=1}^nu(s_h^\ii,a_h^\ii)v(s_h^\ii,a_h^\ii),\quad\|u\evn\|_\nn^2:=\frac1n\sum_{i=1}^n u(s_h^\ii,a_h^\ii)^2.
\]
\begin{lemma}[Near-isometry]\label{lemma:concentration1}
    Under Assumption~\ref{asm:leverage}, with probability at least $1-\delta/4$, simultaneously for all $h\in[H]$ and $u,v\in\Hilb_h$,
    \[
    \Bigl|\langle u\evn,v\evn\rangle_\nn-\langle u,v\rangle_{d_h^D}\Bigr|\le\chi_h\|u\|_{2,d_h^D}\|v\|_{2,d_h^D},
    \]
    where $\chi_h$ is given in Eq.~\eqref{eq:concentration1}. In particular, take $u=v$, we get the near-isometry norm bound:
    \[
    (1-\chi_h)\|u\|_{2,d_h^D}^2\le \|u\evn\|_\nn^2\le(1+\chi_h)\|u\|_{2,d_h^D}^2.
    \]
\end{lemma}

\begin{lemma}[Cross-step concentration]\label{lemma:concentration2}
    For any $f\in\Hilb_{t+1}$, define the one-step sample realization operator by $(\hatP_t f)(s_t^\ii,a_t^\ii):=f(s_{t+1}^\ii,\pi)$. Under Assumption~\ref{asm:leverage}, with probability at least $1-\delta/4$, simultaneously for all $1\le t<h\le H$, all $u\in\Hilb_t$ and all $f\in\Hilb_h$,
    \[
    \left|\left\langle u\evn,(\hatP_t-\bp_t) g_{t+1}^{(h,f)}\evn\right\rangle_\nn\right|\le\tau_{t:h}\|u\|_{2,d_t^D}\|f\|_{2,d_h^D},
    \]
    where $\bp_t=\Pi_{\Hilb_t}P^\pi$ is the projected transition operator at step $t$ (see Definition~\ref{def:backup}), $g_{t+1}^{(h,f)}$ is the shorthand for $\bp_{t+1}\cdots\bp_{h-1} f$ (that is, the backup function of $f$ onto step $t+1$ from step $h$), and $\tau_{t:h}$ is given in Eq.~\eqref{eq:concentration2}.
\end{lemma}

\begin{proof}[Proof of Lemma~\ref{lemma:concentration1}]
    First, we fix an arbitrary pair of elements $(u,v)\in\Hilb_h\times\Hilb_h$ and time step $h\in[H]$. Let $Y_i=u(s_h^\ii,a_h^\ii)v(s_h^\ii,a_h^\ii)$ and $\mu=\EE[Y_i]$. Then the concentration error to be controlled becomes
    \[
    \Bigl|\langle u\evn,v\evn\rangle_\nn-\langle u,v\rangle_{d_h^D}\Bigr|=\left|\frac1n\sum_{i=1}^nY_i-\mu\right|.
    \]
    Under Assumption~\ref{asm:leverage}, $|Y_i|\le\sup_{s,a}|u(s,a)||v(s,a)|\le\kappa_h\|u\|_{2,d_h^D}\|v\|_{2,d_h^D}$, and by Cauchy-Schwarz inequality, $\mu=|\EE[uv]|\le\sqrt{\EE[u^2]\cdot\EE[v^2]}=\|u\|_{2,d_h^D}\|v\|_{2,d_h^D}$. Thus the range bound is given by
    \[
    |Y_i-\mu|\le(\kappa_h+1)\|u\|_{2,d_h^D}\|v\|_{2,d_h^D}\le2\kappa_h\|u\|_{2,d_h^D}\|v\|_{2,d_h^D},
    \]
    since $\kappa_h\ge1$. On the other hand, the variance of $Y_i$ can also be bounded via
    \begin{align*}
    \Var[Y_i]\le\EE[Y_i^2]&=\EE[u(s,a)^2\cdot v(s,a)^2]\\
    &\le\kappa_h\|v\|_{2,d_h^D}^2\cdot\EE[u(s,a)^2]=\kappa_h\|v\|_{2,d_h^D}^2\|u\|_{2,d_h^D}^2,
    \end{align*}
    where the inequality is also due to Assumption~\ref{asm:leverage}. Thus, by Bernstein's inequality, for fixed $(u,v)$, with probability at least $1-\delta/2$,
    \[
    \left|\frac1n\sum_{i=1}^nY_i-\mu\right|\le\|u\|_{2,d_h^D}\|v\|_{2,d_h^D}\left(\sqrt{\frac{2\kappa_h\log (4/\delta)}{n}}+\frac{4\kappa_h\log(4/\delta)}{3n}\right).
    \]

    Now we complete the proof by a standard covering argument. Let $\Ccal_h$ be a $1/8$-cover of the unit ball $\Bcal_h:=\{f\in\Hilb_h:\|f\|_{2,d_h^D}\le 1\}$ under the norm $\|\cdot\|_{2,d_h^D}$, with $|\Ccal_h|=N_h$. Applying the fixed-pair Bernstein bound above to all pairs $(u_0,v_0)\in\Ccal_h\times\Ccal_h$ and all $h\in[H]$, with failure probability $\delta/(2HN_h^2)$ for each pair, we obtain that with probability at least $1-\delta/2$, simultaneously for all $h\in[H]$ and $u_0,v_0\in\Ccal_h$,
    \[
    \Bigl|\langle u_0\evn,v_0\evn\rangle_\nn-\langle u_0,v_0\rangle_{d_h^D}\Bigr|\le\sqrt{\frac{2\kappa_h\log(8HN_h^2/\delta)}{n}}+\frac{4\kappa_h\log(8HN_h^2/\delta)}{3n}=\frac34\chi_h.
    \]
    It remains to extend this bound from the finite cover to the whole unit ball. Let $M_h=\sup_{u,v\in\Bcal_h}|\langle u\evn,v\evn\rangle_\nn-\langle u,v\rangle_{d_h^D}|$ be the uniform concentration error. By the covering property, for arbitrary $u,v\in\Bcal_h$, there exists $u_0,v_0\in\Ccal_h$ such that
    \[
    \|u-u_0\|_{2,d_h^D}\le\frac18,\quad\|v-v_0\|_{2,d_h^D}\le\frac18.
    \]
    By bilinearity of inner product, we have
    \begin{align*}
        \Bigl|\langle u\evn,v\evn&\rangle_\nn-\langle u,v\rangle_{d_h^D}\Bigr| 
        \le \Bigl|\langle u_0\evn,v_0\evn\rangle_\nn-\langle u_0,v_0\rangle_{d_h^D}\Bigr| \\
        &+\Bigl|\langle (u-u_0)\evn,v\evn\rangle_\nn-\langle u-u_0,v\rangle_{d_h^D}\Bigr| 
        +\Bigl|\langle u_0\evn,(v-v_0)\evn\rangle_\nn-\langle u_0,v-v_0\rangle_{d_h^D}\Bigr|.
    \end{align*}
    The first term is bounded by $3\chi_h/4$. For the second term, since $\|u-u_0\|_{2,d_h^D}\le1/8$ and $\|v\|_{2,d_h^D}\le 1$, we have $8(u-u_0)\in\Bcal_h$ and $v\in\Bcal_h$. Therefore, by the definition of $M_h$, we have
    \[
    \Bigl|\langle (u-u_0)\evn,v\evn\rangle_\nn-\langle u-u_0,v\rangle_{d_h^D}\Bigr| \le\frac{1}{8}M_h.
    \]
    Similarly, the third term is also bounded by $M_h/8$. Combining the three terms, we have that for arbitrary $u,v\in\Bcal_h$,
    \[
    \Bigl|\langle u\evn,v\evn\rangle_\nn-\langle u,v\rangle_{d_h^D}\Bigr|\le\frac{3}{4}\chi_h+\frac{1}{4}M_h.
    \]
    Taking the supremum over $u,v\in\Bcal_h$, we have $M_h\le3\chi_h/4+M_h/4$, which is equivalent to $M_h\le\chi_h$. Finally, for arbitrary $u,v\in\Hilb_h$, if both norms are non-zero (if either norm is zero, then the same inequality is trivial under Assumption~\ref{asm:leverage}, since $\|u\|_{2,d_h^D}=0$ implies $u(s,a)=0$ on all sample points as well), apply the above unit-ball bound to $u/\|u\|_{2,d_h^D}$ and $v/\|v\|_{2,d_h^D}$ yields
    \[
    \Bigl|\langle u\evn,v\evn\rangle_\nn-\langle u,v\rangle_{d_h^D}\Bigr|\le\chi_h\|u\|_{2,d_h^D}\|v\|_{2,d_h^D},\quad\forall u,v\in\Hilb_h,\quad\forall h\in[H].
    \]
    This proves the uniform concentration bound.
\end{proof}

\begin{proof}[Proof of Lemma~\ref{lemma:concentration2}]
    The proof structure is similar to that of Lemma~\ref{lemma:concentration1}. We first fix $t<h$, $u\in\Bcal_t$ and $f\in\Bcal_h$, where $\Bcal_t$ and $\Bcal_h$ are unit balls. Define
    \[
    Z_i:=u(s_t^\ii,a_t^\ii)\left(g_{t+1}^{(h,f)}(s_{t+1}^\ii,\pi)-g_t^{(h,f)}(s_t^\ii,a_t^\ii)\right),
    \]
    where $g_t^{(h,f)}=\bp_t g_{t+1}^{(h,f)}=\bp_t\cdots\bp_{h-1}f$. Then, the concentration quantity we want to control becomes
    \[
    \left|\left\langle u\evn,(\hatP_t-\bp_t) g_{t+1}^{(h,f)}\evn\right\rangle_\nn\right|=\frac1n\sum_{i=1}^nZ_i.
    \]
    Notice that $\EE[Z_i]=0$. This is because when we take expectation,
    \[
    \EE[Z_i]=\EE_{(s_t,a_t)\sim d_t^D}\left[u(s_t,a_t)\left((P^\pi g_{t+1}^{(h,f)})(s_t,a_t)-(\bp_t g_{t+1}^{(h,f)})(s_t,a_t)\right)\right].
    \]
    But $\bp_t g_{t+1}^{(h,f)}=\Pi_{\Hilb_t}P^\pi g_{t+1}^{(h,f)}$ where $\Pi_{\Hilb_t}$ is the projection onto $\Hilb_t$. Since $u\in\Hilb_t$, the projection residual is orthogonal to $\Hilb_t$, i.e., $\bp_t$ and $P^\pi$ agree when tested against functions in $\Hilb_t$. Hence $\EE[Z_i]=0$.

    Now we are going to bound $Z_i$. Since $\|u\|_{2,d_t^D}\le 1$, generalized leverage gives $|u(s,a)|\le\sqrt{\kappa_t}$. Also, $\|f\|_{2,d_h^D}\le 1$ implies $\|f\|_\infty\le\sqrt{\kappa_h}$. Thus by Definition~\ref{def:backup}, we have
    \[
    \|g_t^{(h,f)}\|_\infty\le\rho_{t:h}\sqrt{\kappa_h},\quad\|g_{t+1}^{(h,f)}\|_\infty\le\rho_{t+1:h}\sqrt{\kappa_h}.
    \]
    This gives an upper bound for the range of $Z_i$, i.e., $  |Z_i|\le(\rho_{t:h}+\rho_{t+1:h})\sqrt{\kappa_t\kappa_h}$. On the other hand, for the variance control, we can use a slightly sharper bound that
    \[
    \Var[Z_i]\le \EE[Z_i^2]\le(\rho_{t:h}+\rho_{t+1:h})^2\kappa_h\cdot \EE_{(s_t,a_t)\sim d_t^D}[u(s_t,a_t)^2]\le(\rho_{t:h}+\rho_{t+1:h})^2\kappa_h,
    \]
    since $\|u\|_{2,d_t^D}\le 1$. Therefore, Bernstein's inequality gives for fixed $u,f$, with probability at least $1-\delta/2$,
    \[
    \left|\frac1n\sum_{i=1}^nZ_i\right|\le(\rho_{t:h}+\rho_{t+1:h})\left(\sqrt{\frac{2\kappa_h\log(4/\delta)}{n}}+\frac{2\sqrt{\kappa_t\kappa_h}\log(4/\delta)}{3n}\right).
    \]
    The subsequent steps are similar. We can take $1/8$-nets $\Ccal_t\subset \Bcal_t$ and $\Ccal_h\subset\Bcal_h$, with sizes $|\Ccal_t|=N_t$ and $|\Ccal_h|=N_h$. Union bound over all $(u_0,f_0)\in\Ccal_t\times\Ccal_h$ and all $t,h$ gives the same bound with $\log(4/\delta)\leadsto\log(8H^2N_tN_h/\delta)$. We can apply the same procedure to extend the bound from the net to the whole unit ball, by paying an additional $4/3$ factor to get
    \[
    \sup_{u\in\Bcal_t,f\in\Bcal_h}\left|\left\langle u\evn,(\hatP_t-\bp_t) g_{t+1}^{(h,f)}\evn\right\rangle_\nn\right|\le\tau_{t:h}.
    \]
    Finally, for general nonzero $u,f$, apply the above unit-ball result to $u/\|u\|_{2,d_t^D}$ and $f/\|f\|_{2,d_h^D}$ proves the simultaneous concentration bound.
\end{proof}

\subsection{Proof of Theorem~\ref{thm:tracking}}

Now we are going to prove the empirical-population tracking theorem (Theorem~\ref{thm:tracking}). We will break the proof into three steps. Throughout, define
\[
\eta_h:=\|\wt_h-\wstar_h\|_{2,d_h^D},\quad\Delta_h:=\|\wh_h^\nn-\wstar_h\evn\|_\nn.
\]
Our goal is to control the tracking error $\Delta_h$.

\paragraph{Step 1: duality conversion} We will first transfer the weight error $\Delta_h$ to the functional-side error $\eta_h$. This is inspired by linear case where the linear functional (dual) space corresponds to the feature space, and the dynamical system in linear case is also defined in the feature space, which is a much more convenient way to do analysis; see also \citet{duan2020minimax}. To see why this is a duality, define the dual norm $\|\cdot\|_*$ for linear functional $\Lambda$ such that
\[
\|\Lambda\|_*=\sup_{f\in\Hilb_h:\|f\|_{2,d_h^D}\le 1}|\Lambda(f)|.
\]
The following proposition relates $\eta_h$ with this functional error's dual norm.
\begin{proposition} [Error duality]\label{prop:duality}
    For two linear functionals $\hatLambda_h,\Lambda_h^\star:\Hilb_h\to\RR$, suppose their Riesz representers are two elements $\wt_h,\wstar_h\in\Hilb_h$, respectively. Then, $\|\wt_h-\wstar_h\|_{2,d_h^D}=\|\hatLambda_h-\Lambda_h^\star\|_{*}$.
\end{proposition}
\begin{proof}
    For any $f\in\Hilb_h$, we have
    \[
    \bigl|(\hatLambda_h-\Lambda_h^\star)(f)\bigr|=\langle\wt_h-\wstar_h,f\rangle_{d_h^D}\le\|\wt_h-\wstar_h\|_{2,d_h^D}\|f\|_{2,d_h^D},
    \]
    where the first equality is due to Riesz representation theorem, and the second inequality is by Cauchy-Schwarz. Thus, taking the supremum over $\|f\|_{2,d_h^D}\le 1$ gives one direction, i.e., $\|\hatLambda_h-\Lambda_h^\star\|_*\le\|\wt_h-\wstar_h\|_{2,d_h^D}$.

    On the other hand, if $\wt_h=\wstar_h$, then the claim is trivial. Otherwise, take
    \[
    f_h^\star=\frac{\wt_h-\wstar_h}{\|\wt_h-\wstar_h\|_{2,d_h^D}}\in\Hilb_h.
    \]
    Then $\|f_h^\star\|_{2,d_h^D}=1$, and $(\hatLambda_h-\Lambda_h^\star)(f)=\|\wt_h-\wstar_h\|_{2,d_h^D}$. Hence, $\|\hatLambda_h-\Lambda_h^\star\|_*\ge\|\wt_h-\wstar_h\|_{2,d_h^D}$. Combining the two inequalities proves the claim.
\end{proof}

Now we are going to translate the tracking error $\Delta_h$ with the functional error $\eta_h$; see the following lemma.

\begin{lemma}[Tracking error conversion]\label{lemma:tracking-step1} For any $h\in[H]$, assume $\chi_h\le 1/2$ for $\chi_h$ in Eq.~\eqref{eq:concentration1}, which can be achieved by $n\gtrsim\kappa_h\log(HN_h/\delta)$. Then, $\Delta_h\le (1+3\chi_h)\eta_h+2\chi_h\|\wstar_h\|_{2,d_h^D}$.
\end{lemma}
\begin{proof}
    By triangular inequality,
    \[
    \Delta_h=\|\wh_h^\nn-w_h^\star\evn\|_\nn=\|\wb_h-w_h^\star\|_\nn\le \|\wt_h-w_h^\star\|_\nn+\|\wb_h-\wt_h\|_\nn.
    \]
    For the first term, Lemma~\ref{lemma:concentration1} gives (where we use $\sqrt{1+x}\le 1+x/2$ for $x\ge0$):
    \[
    \|\wt_h-w_h^\star\|_\nn\le\sqrt{1+\chi_h}\|\wt_h-w_h^\star\|_{2,d_h^D}\le\left(1+\frac12\chi_h\right)\eta_h.
    \]
    For the second term, recall the definition of $\wt_h$ (which is the representer of $\hatLambda_h$) and $\wb_h$ (which is the sample-lift of $\wh_h^\nn$), we have
    \[
    \langle\wt_h,f\rangle_{d_h^D}=\hatLambda_h(f)=\langle\wb_h,f\rangle_\nn,\quad \forall f\in\Hilb_h.
    \]
    Taking $f=\wb_h-\wt_h$, we obtain
    \begin{align*}
    \|\wb_h-\wt_h\|_\nn^2
    &=\langle\wb_h,\wb_h-\wt_h\rangle_\nn-\langle\wt_h,\wb_h-\wt_h\rangle_\nn\\
    &=\langle\wt_h,\wb_h-\wt_h\rangle_{d_h^D}-\langle\wt_h,\wb_h-\wt_h\rangle_\nn\\
    &\le\chi_h\|\wt_h\|_{2,d_h^D}\|\wb_h-\wt_h\|_{2,d_h^D}\\
    &\le\frac{\chi_h}{\sqrt{1-\chi_h}}\|\wt_h\|_{2,d_h^D}\|\wb_h-\wt_h\|_{\nn},
    \end{align*}
    where the last two inequalities are both due to Lemma~\ref{lemma:concentration1}. By canceling a $\|\wb_h-\wt_h\|_\nn$, we have
    \[
    \|\wb_h-\wt_h\|_\nn\le\frac{\chi_h}{\sqrt{1-\chi_h}}\|\wt_h\|_{2,d_h^D}\le 2\chi_h(\eta_h+\|w_h^\star\|_{2,d_h^D}),
    \]
    where we use the fact that $\chi_h\le 1/2$ and triangular inequality. Combining the two terms, we have
    \[
    \Delta_h\le\left(1+\frac12\chi_h\right)\eta_h+2\chi_h(\eta_h+\|w_h^\star\|_{2,d_h^D})\le (1+3\chi_h)\eta_h+2\chi_h\|w_h^\star\|_{2,d_h^D}.
    \]
\end{proof}

\paragraph{Step 2: functional telescoping and recursion}
By Proposition~\ref{prop:duality} and Lemma~\ref{lemma:tracking-step1}, we learned that essentially we want to control
\[
\eta_h=\|\hatLambda_h-\Lambda_h^\star\|_*=\sup_{f\in\Hilb_h:\|f\|_{2,d_h^D}\le 1}|\hatLambda_h(f)-\Lambda_h^\star(f)|.
\]
The following telescoping lemma demonstrates the iterative nature of $\Lambda_h^\star$ and $\hatLambda_h$, since our \alg algorithm is level-by-level. As we will see, the iterative property of the population functional $\Lambda_h^\star$ depends on the projected backup operator $\bp_h$ in Definition~\ref{def:backup}; while the operator corresponding to the empirical functional $\hatLambda_h$ is the one-step realization operator $\hatP_h$ defined in Lemma~\ref{lemma:concentration2}, i.e.,
\[
(\hatP_h f)(s_h^\ii,a_h^\ii)=f(s_{h+1}^\ii,\pi),\quad f\in\Hilb_{h+1}.
\]
Such correspondences enable us to do functional telescoping; see the following lemma.

\begin{lemma}[Functional telescoping]\label{lemma:tracking-step2-1}
    For every $h\in[H]$ and $f\in\Hilb_h$,
    \[
    \hatLambda_h(f)-\Lambda_h^\star(f)=\hatLambda_1(g_1^{(h,f)})-\Lambda_1^\star(g_1^{(h,f)})+\sum_{t=1}^{h-1}\left\langle\wb_t\evn,((\hatP_t-\bp_t)g_{t+1}^{(h,f)})\evn\right\rangle_\nn,
    \]
    where $g_t^{(h,f)}$ is the shorthand for $\bp_t\bp_{t+1}\cdots\bp_{h-1} f$, i.e., the multi-step backup of $f\in\Hilb_h$ onto $\Hilb_t$.
\end{lemma}
\begin{proof}
    We first claim that $\Lambda_h^\star(f)=\Lambda_t^\star(g_t^{(h,f)})=\langle w_t^\star,g_t^{(h,f)}\rangle_{d_t^D}$ for every $1\le t\le h\le H$ and $f\in\Hilb_h$. Suppose this claim holds, we have that $\Lambda_h^\star(f)=\Lambda_1^\star(g_1^{(h,f)})$. Thus it remains to relate the empirical functionals. Write $g_t:=g_t^{(h,f)}\in\Hilb_t$. By the definition of $\hatLambda_{t+1}$, we have
    \[
    \hatLambda_{t+1}(g_{t+1})=\frac1n\sum_{i=1}^n\wh_t^\ii g_{t+1}(s_{t+1}^\ii,\pi)=\langle\wb_t,\hatP_t g_{t+1}\rangle_\nn,
    \]
    and by using the matching condition (Proposition~\ref{prop:representer}(iii)), we can also write
    \[
    \hatLambda_t(g_t)=\frac1n\sum_{i=1}^n\wh_t^\ii g_t(s_t^\ii,a_t^\ii)=\langle\wb_t\evn,g_t\evn\rangle_\nn.
    \]
    Since $g_t=\bp_tg_{t+1}$ by the backup recursion, the difference of the above two equations becomes
    \[
    \hatLambda_{t+1}(g_{t+1})-\hatLambda_t(g_t)=\langle\wb_t\evn,((\hatP_t-\bp_t)g_{t+1})\evn\rangle_\nn.
    \]
    By telescoping (summing over $t=1,\ldots,h-1$), we have
    \[
    \hatLambda_h(f)-\hatLambda_1(g_1)=\sum_{t=1}^{h-1}\langle\wb_t\evn,((\hatP_t-\bp_t)g_{t+1})\evn\rangle_\nn,
    \]
    since $f=g_h^{(h,f)}$. This proves the lemma. 
    
    Now we go back to prove the claim that $\Lambda_h^\star(f)=\Lambda_t^\star(g_t)=\langle w_t^\star,g_t\rangle_{d_t^D}$. If $t=h$, $g_t=g_h^{(h,f)}=f$ so the claim trivially holds. Consider $t<h$. For $t=h-1$, we can write
    \[
    \Lambda_h^\star(f)=\EE_{d_{h-1}^D}[w_{h-1}^\star\cdot (P^\pi f)]=\EE_{d_{h-1}^D}[w_{h-1}^\star\cdot (\Pi_{\Hilb_{h-1}}P^\pi f)]=\langle w_{h-1}^\star,\bp_{h-1} f\rangle_{d_{h-1}^D},
    \]
    where the second equality uses $w_{h-1}^\star\in\Hilb_{h-1}$ (Proposition~\ref{prop:representer}(i)) and orthogonality of $\Pi_{h-1}$ (the projection onto $\Hilb_{h-1}$), and the last equality uses the definition of $\bp_{h-1}=\Pi_{\Hilb_{h-1}}P^\pi$. Iterating this identity, we have
    \[
    \Lambda_h^\star(f)=\langle w_t^\star,\bp_{t}\bp_{t+1}\cdots\bp_{h-1}f\rangle_{d_t^D}=\langle w_t^\star,g_t\rangle_{d_t^D}=\Lambda_t^\star(g_t),
    \]
    where the last equality is due to the definition of $\Lambda_t^\star$.
\end{proof}

Using the telescoping in Lemma~\ref{lemma:tracking-step2-1}, we can now decompose this empirical-population tracking error (in the functional space) and thereby control $\eta_h$, yielding a recursion formula to be solved later.

\begin{lemma}[Functional recursion]\label{lemma:tracking-step2-2}
    With probability at least $1-\delta$, assume $\chi_t\le1/2$ for all $t\le h$. Then for every $h\in[H]$,
    \[
    \eta_h\le\sum_{t=1}^{h-1}a_{t:h}\eta_t+\sum_{t=1}^{h-1}a_{t:h}\|w_t^\star\|_{2,d_t^D}+\frac{4}{3}\rho_{1:h}\sqrt{\frac{2\kappa_h\log(4N_h/\delta)}{n}},
    \]
    where $a_{t:h}:=\tau_{t:h}+3\chi_t\sqrt{\kappa_h}(\rho_{t:h}+\rho_{t+1:h})$ with $\chi_h$ in Eq.~\eqref{eq:concentration1} and $\tau_{t:h}$ in Eq.~\eqref{eq:concentration2}.
\end{lemma}
\begin{proof}
    Fix $f\in\Hilb_h$ with $\|f\|_{2,d_h^D}\le 1$. By Lemma~\ref{lemma:tracking-step2-1}, we have
    \[
    \hatLambda_h(f)-\Lambda_h^\star(f)=\hatLambda_1(g_1^{(h,f)})-\Lambda_1^\star(g_1^{(h,f)})+\sum_{t=1}^{h-1}\left\langle\wb_t\evn,((\hatP_t-\bp_t)g_{t+1}^{(h,f)})\evn\right\rangle_\nn.
    \]
    For the first-step error term $\hatLambda_1(g)-\Lambda_1^\star(g)$ with $g:=g_1^{(h,f)}$, by Definition~\ref{def:functional}, we have
    \begin{align*}
    \hatLambda_1(g)-\Lambda_1^\star(g)&=\frac1n\sum_{i=1}^n\wh_0^\ii g(s_1^\ii,\pi)-\EE_D[w_0^\star(s_0,a_0)g(s_1,\pi)]\\
    &=\frac1n\sum_{i=1}^n g(s_1^\ii,\pi)-\EE_D[g(s_1,\pi)]\le\frac{4}{3}\rho_{1:h}\sqrt{\frac{2\kappa_h\log(4N_h/\delta)}{n}},
    \end{align*}
    with probability at least $1-\delta/2$, where we used $\|g\|_\infty\le\rho_{1:h}$ and the standard covering argument over $\Bcal_h$ (since $g=g_1^{(h,f)}$ is uniquely determined by $f\in\Hilb_h$). It remains to tackle the second summation term.

    For notational simplicity, we omit the $\evn$ in $\langle\cdot\evn,\cdot\evn\rangle_\nn$ in the following proof. We decompose this functional error into the following three terms using $\wb_t=(\wt_t-w_t^\star)+w_t^\star+(\wb_t-\wt_t)$:
    \begin{gather*}
        \text{(I)}=\left|\left\langle\wt_t-w_t^\star,(\hatP_t-\bp_t)g_{t+1}^{(h,f)}\right\rangle_\nn\right|,\\
        \text{(II)}=\left|\left\langle w_t^\star,(\hatP_t-\bp_t)g_{t+1}^{(h,f)}\right\rangle_\nn\right|,\\
        \text{(III)}=\left|\left\langle \wb_t-\wt_t,(\hatP_t-\bp_t)g_{t+1}^{(h,f)}\right\rangle_\nn\right|.
    \end{gather*}
    By Lemma~\ref{lemma:concentration2}, we have $\text{(I)}\le \tau_{t:h}\eta_t$ and $\text{(II)}\le\tau_{t:h}\|w_t^\star\|_{2,d_t^D}$. For the third term,
    \begin{align*}
    \text{(III)}&\le\|\wb_t-\wt_t\|_{\nn}\cdot\left\|(\hatP_t-\bp_t)g_{t+1}^{(h,f)}\right\|_\nn\\
    &\le\frac{\chi_t}{\sqrt{1-\chi_t}}(\eta_t+\|w_t^\star\|_{2,d_t^D})\left(\left\|\hatP_t g_{t+1}^{(h,f)}\right\|_\nn+\left\|g_t^{(h,f)}\right\|_\nn\right),
    \end{align*}
    where the first inequality is due to Cauchy-Schwarz, and the second inequality we leveraged proof of Lemma~\ref{lemma:tracking-step1} to control the first $\|\wb_t-\wt_t\|_\nn$ term, and the triangular inequality to control the second term (by noticing that $g_t^{(h,f)}=\bp_t g_{t+1}^{(h,f)}$).

    For the $\|\hatP_t g_{t+1}^{(h,f)}\|_\nn$ term, we can bound it via Definition~\ref{def:backup}:
    \[
    \left\|\hatP_t g_{t+1}^{(h,f)}\right\|_\nn=\sqrt{\frac{1}{n}\sum_{i=1}^ng_{t+1}^{(h,f)}(s_{t+1}^\ii,\pi)^2}\le\max_{s,a}\left|g_{t+1}^{(h,f)}(s,a)\right|\le \sqrt{\kappa_h}\rho_{t+1:h}.
    \]
    where we use Assumption~\ref{asm:leverage} and the fact that $\|f\|_{2,d_h^D}\le 1$. For the $\|g_t^{(h,f)}\|_\nn$ term, we can similarly bound it via Definition~\ref{def:backup} (transferred via isometry concentration at first):
    \[
    \left\|g_t^{(h,f)}\right\|_{\nn}\leq\sqrt{1+\chi_t}\left\|g_t^{(h,f)}\right\|_{2,d_t^D}\le\sqrt{1+\chi_t}\cdot \sqrt{\kappa_h}\rho_{t:h}\le 2\sqrt{\kappa_h}\rho_{t:h}.
    \]
    Hence combining the terms, the third term can be controlled by
    \begin{align*}
    \text{(III)}&\le \frac{2\chi_t\sqrt{\kappa_h}}{\sqrt{1-\chi_t}}(\rho_{t:h}+\rho_{t+1:h})(\eta_t+\|w_t^\star\|_{2,d_t^D})\\&\le 3\chi_t\sqrt{\kappa_h}(\rho_{t:h}+\rho_{t+1:h})(\eta_t+\|w_t^\star\|_{2,d_t^D}).
    \end{align*}
    Summing the three bounds and taking the supremum over $\|f\|_{2,d_h^D}\le 1$ controls the functional error term. By union bound over three events: the first step concentration error with failure probability $\delta/2$, the event in Lemma~\ref{lemma:concentration1} with failure probability $\delta/4$, and the event in Lemma~\ref{lemma:concentration2} with failure probability $\delta/4$; we prove the claim with probability at least $1-\delta$.
\end{proof}

\paragraph{Step 3: solving the recursion} As shown in Lemma~\ref{lemma:tracking-step2-2}, what we get is essentially a recursion in $\eta_h$ (i.e., the RHS also appears $\eta_t$ for $t\le h$). Therefore the final step of proving Theorem~\ref{thm:tracking} is to solve for this recursion to control $\eta_h$, hence $\Delta_h$.

\begin{proof}[Proof of Theorem~\ref{thm:tracking}]
Abbreviate $\epsilon_{1:j}=\frac{4}{3}\rho_{1:j}\sqrt{2\kappa_j\log(4N_j/\delta)/n}$. Let
\[
A_h:=\max_{1\le j\le h}\epsilon_{1:j}+\sum_{t=1}^{j-1}a_{t,j}\|w_t^\star\|_{2,d_t^D},\quad B_h:=\max_{1\le j\le h}\sum_{t=1}^{j-1}a_{t,j},\quad E_h:=\max_{1\le j\le h}\eta_j.
\]
By Lemma~\ref{lemma:tracking-step2-2}, with probability at least $1-\delta$, for every $j\le h$, we have
\[
\eta_j\le\sum_{t=1}^{j-1}a_{t,j}\eta_t+\epsilon_{1:j}+\sum_{t=1}^{j-1}a_{t,j}\|w_t^\star\|_{2,d_t^D}\le B_hE_h+A_h.
\]
Taking the maximum over $j\le h$ gives $E_h\le B_hE_h+A_h$. Hence, we have
\[
\eta_h\le \max_{1\le j\le h}\eta_j=E_h\le\frac{A_h}{1-B_h}\le 2A_h\le 2\epsilon_{1:h}+2\sum_{t=1}^{h-1}a_{t:h}\|w_t^\star\|_{2,d_t^D},
\]
where we use $B_h\le1/2$ since our choice of $n$ that $n\gtrsim(\max_{t\le h}\kappa_t)^2(\sum_{t=1}^h\rho_{t:h})^2\log(H^2N_tN_h/\delta)$ automatically enables $B_h\le 1/2$; and the last inequality is due to the maximum of $A_h$ is exactly taken at $j=h$ (because all elements are non-negative; that is, $a_{t,j}\|w_t^\star\|_{2,d_t^D}\ge0$). Now we can apply Lemma~\ref{lemma:tracking-step1}:
\[
\Delta_h\le(1+3\chi_h)\eta_h+2\chi_h\|w_h^\star\|_{2,d_h^D}\le 5\epsilon_{1:h}+5\sum_{t=1}^{h-1}a_{t:h}\|w_t^\star\|_{2,d_t^D}+2\chi_h\|w_h^\star\|_{2,d_h^D},
\]
where we use $\chi_h\le 1/2$ (hence $2(1+3\chi_h)\le 5$). By the definitions of $\chi_h$ in Eq.~\eqref{eq:concentration1} and $\tau_{t:h}$ in Eq.~\eqref{eq:concentration2}, we have
\begin{align*}
    \Delta_h
    &\le5\epsilon_{1:h}+5\sum_{t=1}^{h-1}a_{t:h}\|w_t^\star\|_{2,d_t^D}+2\chi_h\|w_h^\star\|_{2,d_h^D}\\
    &\le\frac{4}{3}\rho_{1:h}\sqrt{\frac{2\kappa_h\log(4N_h/\delta)}{n}}
    +\frac{20}{3}\sqrt{\kappa_h}\sum_{t=1}^{h-1}(\rho_{t:h}+\rho_{t+1:h})\|\wstar_t\|_{2,d_t^D}\\
    &\qquad\cdot\Biggl(\sqrt{\frac{2\kappa_h\log\tfrac{8H^2N_tN_h}{\delta}}{n}}+\frac{2\sqrt{\kappa_t\kappa_h}\log\tfrac{8H^2N_tN_h}{\delta}}{3n}+3\sqrt{\frac{2\kappa_h\log\frac{8HN_h^2}{\delta}}{n}}+\frac{4\kappa_h\log\tfrac{8HN_h^2}{\delta}}{n}\Biggr)\\
    &\qquad+\frac{8}{3}\|\wstar_h\|_{2,d_h^D}\Biggl(\sqrt{\frac{2\kappa_h\log\frac{8HN_h^2}{\delta}}{n}}+\frac{4\kappa_h\log\tfrac{8HN_h^2}{\delta}}{3n}\Biggr)\\
    &\lesssim\sum_{t=1}^h\rho_{t:h}\sqrt{\frac{\kappa_h^2\log\tfrac{H^2N_tN_h}{\delta}}{n}}\|\wstar_t\|_{2,d_t^D}+\rho_{1:h}\sqrt{\frac{\kappa_h\log\tfrac{N_h}{\delta}}{n}}+\text{higher order term}.
\end{align*}
We use $\lesssim$ to ignore some universal constants (here the constant can be set to $320/3$ if we combine all terms), and the higher order term is of order $O(1/n)$ in the above equation. And we also combined the logarithmic terms and the first step term (by our choice of $n\gtrsim(\max_{t\le h}\kappa_t)^2(\sum_{t=1}^h\rho_{t:h})^2\log(H^2N_tN_h/\delta)$). Therefore, ignoring the higher-order term, we prove Theorem~\ref{thm:tracking}: for any $h\in[H]$, with probability at least $1-\delta$,
\[
\left\|\wh_h^\nn-\wstar_h\evn\right\|_\nn\lesssim\sum_{t=1}^h\rho_{t:h}\sqrt{\frac{\kappa_h\log(H^2N_tN_h/\delta)}{n}}\|\wstar_t\|_{2,d_t^D}+\rho_{1:h}\sqrt{\frac{\kappa_h\log(N_h/\delta)}{n}}.
\]

\end{proof}

\subsection{Proofs of Proposition~\ref{prop:empirical-exact} and Corollary~\ref{thm:random-design}}

Here we provide the proof of Proposition~\ref{prop:empirical-exact}, as well as the random-design result of the \alg algorithm (Corollary~\ref{thm:random-design}). As stated in Proposition~\ref{prop:empirical-exact}, the population 0 matching loss will imply the empirical 0 matching loss with high probability. In the linear case, this is obvious since the invertible covariance $\Sigma$ implies the invertible empirical covariance $\hatSigma$ by matrix Bernstein inequality. However, under GFA, we need to leverage the near-isometry concentration argument (Lemma~\ref{lemma:concentration1}).

\begin{proof}[Proof of Proposition~\ref{prop:empirical-exact}]
As we will see, the empirical 0 matching loss does not rely on the whole function $f\in\Hilb_h$, but its sample-evaluation $f\evn=(f(s_h^{(1)},a_h^{(1)}),\ldots,f(s_h^{(n)},a_h^{(n)}))$, where  the sample evaluation map $\evn:\Hilb_h\to\RR^n$. The near-isometry condition in Lemma~\ref{lemma:concentration1} implies that $\evn$ is injective. Indeed, if $f\evn=0$, then $\|f\evn\|_\nn=0$, which implies
\[
0=\|f\evn\|_\nn^2\ge(1-\chi_h)\|f\|_{2,d_h^D}^2.
\]
Since $\chi_h<1/2$ by our choice of $n$, this gives $\|f\|_{2,d_h^D}=0$, hence $f=0$ as an element of $\Hilb_h$.

Now consider the sample evaluation space $\Hilb_h\evn=\{f\evn:f\in\Hilb_h\}\subseteq\RR^n$. We can define a linear functional $\ell_h:\Hilb_h\evn\to\RR$ by
\[
\ell_h(f\evn):=\hatLambda_h(f)=\frac1n\sum_{i=1}^n\wh_{h-1}^\ii f(s_h^\ii,\pi),\quad \forall f\evn\in\Hilb_h\evn,
\]
This is well defined. To see this, suppose $f\evn=g\evn$. Then $(f-g)\evn=0$ by linearity; and $f=g$ in $\Hilb_h$ by injectivity of this sample evaluation operator $\evn$, i.e., $\|f-g\|_{2,d_h^D}=0$. Therefore, we can apply generalized leverage (Assumption~\ref{asm:leverage}) to give
\[
|f(s,a)-g(s,a)|\le\sqrt{\kappa_h}\|f-g\|_{2,d_h^D}=0,\quad\forall (s,a)\in\Scal_h\times\Acal.
\]
Thus $f$ and $g$ also agree pointwise. In particular, this implies $f(s_h^\ii,\pi)=g(s_h^\ii,\pi)$ for all $i\in[n]$. Hence $\hatLambda_h(f)=\hatLambda_h(g)$. So the linear functional $\ell_h$ is well-defined.

Next, since $\Hilb_h\evn$ is a subspace of $\RR^n$, it is also finite-dimensional. Equip $\Hilb_h\evn$ with the empirical inner product $\langle z,z'\rangle_\nn=\frac1nz^\top z'$. By the finite-dimensional Riesz representation theorem, there exists some vector $a_h\in\Hilb_h\evn\subseteq \RR^n$ such that
\[
\ell_h(z)=\langle a_h,z\rangle_\nn,\quad\forall z\in\Hilb_h\evn.
\]
Taking $z=f\evn$, we have
\[
\hatLambda_h(f)=\ell_h(f\evn)=\langle a_h,f\evn\rangle_\nn=\frac1n\sum_{i=1}^n a_h^\ii f(s_h^\ii,a_h^\ii),\quad\forall f\in\Hilb_h.
\]
By setting $\wh_h^\nn:=a_h$, we would achieve empirical 0 matching loss, i.e.,
\[
\frac1n\sum_{i=1}^n\wh_h^\ii f(s_h^\ii,a_h^\ii)=\hatLambda_h(f)=\frac1n\sum_{i=1}^n\wh_{h-1}^\ii f(s_h^\ii,\pi).
\]
So the empirical matching loss over $\Hilb_h$ is zero. Since $\Fcal_h\subseteq\Hilb_h$, the empirical matching loss over $\Fcal_h$ is also zero. Thus the empirical minimizer used by \alg also achieves zero loss.
\end{proof}

The proof of Corollary~\ref{thm:random-design} is a direct consequence of Theorem~\ref{thm:tracking} (by leveraging the same procedure of telescoping as our fixed-design result in Theorem~\ref{thm:main}). As we will see in the following, the empirical-population tracking error is paired with the average Bellman residual, leading it to a higher-order term so that the dominant term becomes random-design control.

\begin{proof}[Proof of Corollary~\ref{thm:random-design}]

First, we make the following decomposition of the suboptimality gap $J(\pi)-\widehat{J}(\pi)$:
\begin{align*}
    |J(\pi)-\widehat{J}(\pi)|
    &=\left|Q^\pi(s_0,a_0)-\frac1n\sum_{i=1}^n\sum_{h=0}^H\wh_h^\ii r_h^\ii\right|\\
    &=\left|\frac1n\sum_{i=1}^n\wh_0^\ii Q^\pi(s_0^\ii,a_0^\ii)-\frac1n\sum_{i=1}^n\sum_{h=0}^H\wh_h^\ii r_h^\ii\right|\\
    &=\left|\frac1n\sum_{i=1}^n\sum_{h=0}^H\left(\wh_h^\ii Q^\pi(s_h^\ii,a_h^\ii)-\wh_{h+1}^\ii Q^\pi(s_{h+1}^\ii,a_{h+1}^\ii)\right)-\frac1n\sum_{i=1}^n\sum_{h=0}^H\wh_h^\ii r_h^\ii\right|\\
    &=\left|\frac1n\sum_{i=1}^n\sum_{h=0}^H\left(\wh_h^\ii Q^\pi(s_h^\ii,a_h^\ii)-\wh_h^\ii Q^\pi(s_{h+1}^\ii,\pi)-\wh_h^\ii r_h^\ii\right)\right|\\
    &\le \underbrace{\left|\frac1n\sum_{i=1}^n\sum_{h=0}^H \wstar_h (s_h^\ii,a_h^\ii)\varepsilon_h^\ii\right|}_\text{(I): random design control}+\underbrace{\left|\frac1n\sum_{i=1}^n\sum_{h=0}^H\left(\wh_h^\ii-\wstar_h(s_h^\ii,a_h^\ii)\right)\varepsilon_h^\ii\right|}_\text{(II): tracking error},
\end{align*}
where $\varepsilon_h^\ii:=Q^\pi(s_h^\ii,a_h^\ii)-r_h^\ii-Q^\pi(s_{h+1}^\ii,\pi)$ denotes the Bellman residual at current transition sample $(s_h^\ii,a_h^\ii,r_h^\ii,s_{h+1}^\ii)$. We use that $\wh_0\equiv 1$ and $\wh_{H+1}\equiv 0$; and the fourth equation is due to our empirical 0 matching loss in Proposition~\ref{prop:empirical-exact}.

The first term is the random-design control on population weight functional $\wstar_h$. For each fixed $h$, define $X_{h,i}:=\wstar_h(s_h^\ii,a_h^\ii)\varepsilon_h^\ii$. Conditioned on $\{(s_h^\ii,a_h^\ii)\}_{i=1}^n$, the variables $\{X_{h,i}\}_{i=1}^n$ are independent, mean-zero, and satisfy $|X_{h,i}|\le V_{\max}|\wstar_h(s_h^\ii,a_h^\ii)|$ since $|\varepsilon_h^\ii|\le|V_{\max}|$. Therefore, by standard Hoeffding's inequality, we have: with probability at least $1-\alpha$,
\[
\left|\frac1n\sum_{i=1}^nX_{h,i}\right|\le V_{\max}\sqrt{\frac{2\log(2/\alpha)}{n}}\|\wstar_h\evn\|_{[n]},
\]
where $\|\wstar_h\evn\|_\nn=\left(\tfrac1n\sum_{i=1}^n\wstar_h(s_h^\ii,a_h^\ii)^2\right)^{1/2}$. Using the near-isometry transfer in Lemma~\ref{lemma:concentration1} and $\chi_h\le 1/2$ (which is implied by our choice of $n$ in Corollary~\ref{thm:random-design}), we have
\[
\|\wstar_h\evn\|_\nn\le\sqrt{1+\chi_h}\|\wstar_h\|_{2,d_h^D}\le\sqrt{\frac32}\|\wstar_h\|_{2,d_h^D}\le 2\|\wstar_h\|_{2,d_h^D}.
\]
Take $\alpha=\delta/(2(H+1))$, then union bound over $h\in[H]$, we can control the first term by
\[
\text{(I)}\le 2V_{\max}\sqrt{\frac{2\log(4(H+1)/\delta)}{n}}\sum_{h=0}^H\|\wstar_h\|_{2,d_h^D}.
\]
Now for the tracking part, define $Y_{h,i}=\Delta_h^\ii\varepsilon_h^\ii$. Because $\wh_h^\ii$ depends only on the trajectory prefix up to $(s_h^\ii,a_h^\ii)$, and $\wstar_h(s_h^\ii,a_h^\ii)$ is also measurable with respect to $(s_h^\ii,a_h^\ii)$, the coefficient $\Delta_h^\ii$ is fixed when conditioning on the level-$h$ prefixes. Therefore, conditional on those prefixes, the variables $\{Y_{h,i}\}_{i=1}^n$ are again independent, mean-zero, and $|Y_{h,i}|\le V_{\max}|\Delta_h^\ii|$. Apply Hoeffding's inequality again, for each fixed $h$, with probability at least $1-\alpha$,
\[
\left|\frac1n\sum_{i=1}^nY_{h,i}\right|\le V_{\max}\sqrt{\frac{2\log(2/\alpha)}{n}}\|\Delta_h\|_\nn=V_{\max}\sqrt{\frac{2\log(2/\alpha)}{n}}\|\wh_h^\nn-\wstar_h\evn\|_\nn.
\]
Taking again $\alpha=\delta/(2(H+1))$ and union bound over $h$, we can control the second term by
\[
\text{(II)}\le V_{\max}\sqrt{\frac{2\log(4(H+1)/\delta)}{n}}\sum_{h=1}^H\|\wh_h^\nn-\wstar_h\evn\|_\nn,
\]
since $\wstar_0\equiv 1$ and $\wh_0^\ii=1$ for all $i\in[n]$. So now we can apply the tracking theorem (Theorem~\ref{thm:tracking}) to see that the second term is indeed a higher-order term, since both the average Bellman error and the tracking error exhibit the rate of $n^{-1/2}$ (hence an overall rate of $n^{-1}$). Combining $\text{(I)}$ and $\text{(II)}$ proves the corollary.
\end{proof}
\section{Further discussion on MIS} \label{app:mis}

\subsection{Issues with non-parametric weights in MIS} \label{app:double-sampling}
MIS methods often take a minimax form similar to our \alg, and here we provide further discussion on the issue with learning non-parametric weights in MIS. Taking MQL \citep{uehara2019minimax} as an example, whose population loss is (in the infinite-horizon discounted setting)
$$
\argmin_{f\in\Fcal} \max_{w\in\Wcal} \Lcal_{\mathsf{q}}(f, w):= \left|\EE_{d^D}[w(s,a) (f(s,a) - r - \gamma f(s',\pi))]\right|.
$$
When $w$ is not restricted to a parametric class $\Wcal$ but is a separate scalar for each data point, one can immediately see that the loss $\Lcal$ is ill-behaved: the original MIS derivations rely on the fact that $\max_{w\in\Wcal}\Lcal(Q^\pi, w) = 0$ \citep{uehara2019minimax, jiang2020minimax}, but when $w$ is non-parametric, $\Lcal(Q^\pi, w) \ne 0$ in general when the environment is stochastic, since $w$ can simply choose to be positive or negative depending on the sign of $Q^\pi(s,a) - r - \gamma Q^\pi(s',\pi)$. (This issue can be viewed as a version of the infamous double-sampling problem.) In contrast, our algorithm and analyses forbid the choice of $w$ to depend on the next-state randomness, which avoids this problem. 

Now, MQL learns $f$ using $w$ as discriminator, and its ``dual'' method such as MWL \citep{uehara2019minimax} learns $w$ using $f$ as discriminator, which is more similar to our method. 
In particular, the population loss for MWL is (we assume the initial state-action pair is fixed as $(s_0,a_0)$):
\begin{equation}\label{eq:mwl}
\argmin_{w\in\Wcal} \sup_{f\in\Fcal} \Lcal_{\mathsf{w}}(w, f):= \left|f(s_0, a_0) + \frac{1}{1-\gamma} \EE_{d^D}[w(s,a)(\gamma f(s',\pi) - f(s,a))]\right|.
\end{equation}
The following proposition uses a counterexample to show that MWL is subject to similar issues when we use non-parametric weights. The reason is exactly the same: sample-wise weights can implicitly depend on the realized next state $s_i'$ which destroys the usual Bellman-noise cancellation.

\begin{proposition}[Non-parametric MWL can overfit transition noise]
\label{prop:mwl-nonparametric}
Consider the two-state discounted MDP below, with $\gamma=1/2$,
$s_0=A$, and a single action. Let $\Fcal$ be the full bounded tabular
class, and let the dataset consist of the following 5 transition samples: 
\[
A\to A,\quad A\to A,\quad A\to B,\quad A\to B,\quad B\to B.
\]
Note that the stochastic transitions $A \to A$ and $A \to B$ in the data accurately reflect the true transition probabilities $P(A|A) = P(B|A) = 1/2$. 
Then, for empirical MWL with nonnegative sample-wise weights, there exists a set of weights that  admits zero
matching loss  but produces an incorrect OPE estimate. This holds even if the weights minimizes the second moment. 
\end{proposition}

\begin{center}
\begin{tikzpicture}[
node distance=3.0cm,
state/.style={circle,draw,minimum size=9mm},
>=Stealth
]
\node[state] (A) {$A$};
\node[state, right=of A] (B) {$B$};
\node[left=3mm of A] {$r(A)=0$};
\node[right=3mm of B] {$r(B)=1$};

\path[->]
(A) edge[loop above] node {$1/2$} (A)
(A) edge[bend left=10] node[above] {$1/2$} (B)
(B) edge[loop above] node {$1$} (B);
\end{tikzpicture}
\end{center}

\begin{proof}
The true values are $Q^\pi(B)=2$ and $Q^\pi(A)=J(\pi)=2/3$.
For sample-wise weights $w^{[5]}=(w_1,\ldots,w_5)$, writing
$f_A=f(A)$ and $f_B=f(B)$, the rescaled empirical MWL moment is
\[
\widehat{\Lcal}_W(w^{[5]},f)=\left|\left(1-\frac{w_1+w_2+2(w_3+w_4)}5\right)f_A+\frac{w_3+w_4-w_5}{5}f_B\right|.
\]
Since $\Fcal$ is tabular, zero matching loss is equivalent to
\[
w_1+w_2+2(w_3+w_4)=5,\qquad w_5=w_3+w_4.
\]
Minimizing $\sum_{i=1}^5 w_i^2$ subject to these constraints and
$w_i\ge 0$ yields
\[
w^{[5],\star}=\left(\frac{15}{14},\frac{15}{14},\frac57,\frac57,\frac{10}{7}\right).
\]
Although this solution achieves zero MWL loss, its value estimate is
\[
\widehat J_w(\pi)=\frac{1}{1-\gamma}\frac15\sum_{i=1}^5w_i^\star r_i=\frac25w_5^\star=\frac47\neq\frac23=J(\pi). \qedhere
\]
\end{proof}

\subsection{Data-dependent bounds for MIS}
Recall the population loss for the MWL in Eq.~\eqref{eq:mwl}. The actual algorithm optimizes the empirical loss $\emp{\Lcal}_{\mathsf{w}}$ where $\EE_{d^D}$ is replaced with the empirical expectation $\emp{\EE}_{d^D}$, and forms the final prediction $\emp{J}(\pi)= \frac{1}{1-\gamma} \emp{\EE}_{d^D}[\emp{w} r]$ where $\emp{w}$ is the learned $w$.  The standard finite-sample guarantee for MWL pays for the statistical complexities of both $\Wcal$ and $\Fcal$. Below we show that the complexity of $\Fcal$ can be avoided in a data-dependent guarantee similar to our Theorem~\ref{thm:main}: for any fixed $w\in\Wcal$, we have
\begin{align*}
&~ \left|\frac{1}{1-\gamma}\emp{\EE}_{d^D}[w(s,a) r] - J(\pi) \right| \\
= &~ \left|\frac{1}{1-\gamma}\emp{\EE}_{d^D}[w(s,a) (Q^\pi(s,a) -  r - \gamma Q^\pi(s',\pi)] - Q^\pi(s_0,a_0) +  \frac{1}{1-\gamma}\emp{\EE}_{d^D}[w(s,a) (-Q^\pi(s,a) + \gamma Q^\pi(s',\pi))\right| \\
\le &~ \left|\frac{1}{1-\gamma}\emp{\EE}_{d^D}[w(s,a) (Q^\pi(s,a) -  r - \gamma Q^\pi(s',\pi)]\right| + \sup_{f\in\Fcal} \emp\Lcal_{\mathsf{w}}(w, f). \tag{$Q^\pi\in\Fcal$}
\end{align*}
The first term enjoys a ``fixed-design'' concentration bound similar to our proof of Theorem~\ref{thm:main}, but needs a union bound over $w\in\Wcal$ so that the bound holds for $\emp{w}$. The second term is the empirically observed loss, similar to our $\emp{\Lcal}_h$ term in Theorem~\ref{thm:main}.

As discussed in Section~\ref{sec:related}, MWL can be viewed as a ``stacked'' variant of \alg, since it optimizes the matching losses for all $h$ simultaneously (in the finite-horizon setting). This formulation makes a direct extension of \alg to the infinite-horizon discounted setting nontrivial, where we can no longer exploit the layered structure to perform moment matching sequentially across time steps. Although data splitting may provide a possible workaround, we do not pursue this direction here. We leave a principled extension to the infinite-horizon discounted setting for future work.
\section{Further discussion on the Stability of Q-MMR Weights}
\label{sec:stability}

At the end of Section~\ref{sec:finite-sample}, we discussed how one may choose weights that do not exactly minimize the 2nd moment for bias-variance trade-off. This view is, however, myopic, in the sense that we only focused on the effects of the weights at the current time step $h$, and ignored the fact that the weights at time $h$ also sets up the matching target for subsequent time steps. 

Towards understanding such delayed effects of the weights, we start with basic questions on the \textit{stability} of weights, namely: (\textbf{Q1}) if we choose a set of weights that is locally more favorable (e.g., achieving the same matching loss at time $t$ while having smaller sizes), does that guarantee that the induced weights in subsequent steps (or \textit{continuations}) will also be more favorable? A closely related question is the stability w.r.t.~policy perturbations: (\textbf{Q2})  when the target policy changes slightly, does that lead to well-bounded changes in the learned weights?


Below we show that, under value-function realizability alone (Assumption~\ref{asm:realizable}), both questions have negative answers. 
The underlying issue is precisely the phenomenon captured by the feature dynamical system in Eq.~\eqref{eq:lds}: a direction that appears benign at the current step $h$ may be exponentially amplified by the subsequent projected dynamics at steps $h'>h$. We then show that the stronger condition of Bellman completeness (Assumption~\ref{asm:completeness})\footnote{The version of Bellman completeness we actually need is $P^\pi f \in \textrm{span}(\Fcal_h), \forall f\in \textrm{span}(\Fcal_{h+1})$, which is implied by Assumption~\ref{asm:completeness} and $\mathbf{0} \in \Fcal$ as in Proposition~\ref{prop:completeness}. \label{ft:span}} rules out this source of amplification and guarantees stability. 
For simplicity, these results are established for the population version of the algorithm (Definition~\ref{def:pop}).  
In the constructions, we consider problem instances where exact matching is possible, guaranteed by an invertible feature covariance matrix 
$\Sigma_h=\EE_D[\phi_h\phi_h^\top]$ in the linear case.


\begin{proposition}[(\textbf{Q1}) Smaller weights can induce exponentially larger continuations]
\label{prop:stability-1}
For $H\ge 2$, there exists a zero-reward MDP with a realizable $2$-dimensional linear function class such that, at $h=1$, two sets of weights $w_1^{\min}$ and $v_1$ both achieve zero matching loss. Moreover, $w_1^{\min}$ is the minimum-second-moment solution and is pointwise strictly smaller than $v_1$ in absolute value. Nevertheless, their minimum-second-moment continuations with zero matching loss satisfy, for all $2\le h\le H$,
    \[
    v_h\equiv 1,\qquad \|w_h^{\min}\|_{2,d_h^D}^2=\Omega(4^h).
    \]
\end{proposition}

\begin{proof}
    Consider an MDP where the dummy initial state transitions uniformly to one of three states labeled $\{1,2,3\}$, and subsequent transitions preserve the label, i.e., there are 3 separate chains of states, and taking any action deterministically moves down the chain. The action space is $\{0,1\}$, with the behavior policy always choosing $a=0$ and the target policy always choosing $a=1$. At each level, the learned weights can be specified by the weight for each state and action $a=0$, since $a=1$ is not observed in the data. All rewards are zero, so realizability holds trivially.

    Let $q_0=[1,-2,0]^\top$ and $q=[1,-1,0]^\top$, and define
    \[
    \phi_1(s,a)=\begin{pmatrix}1\\(1+2a)q_0(s)\end{pmatrix},\qquad \phi_h(s,a)=\begin{pmatrix}1\\(1+a)q(s)\end{pmatrix},\quad h\ge 2.
    \]
    Since the behavior distribution is uniform over the three labels, exact moment matching at $h=1$ requires $\EE_D[w_1]=1$ and $\EE_D[w_1q_0]=-1$. The unique minimum-2nd-moment solution is $w_1^{\min}=[3/7,12/7,6/7]^\top=6/7-3q_0/7$, obtained by minimizing $\EE_D[w_1^2]$ under these two linear constraints. Another feasible solution is $v_1=(3,3,-3)$, whose absolute value is strictly larger at every coordinate.

    For $h\ge 2$, exact moment matching requires $\EE_D[w_h]=1$ and $m_h:=\EE_D[w_hq]=2\EE_D[w_{h-1}q]$. Since $\EE_D[q]=0$ and $\EE_D[q^2]=2/3$, the minimum-2nd-moment solution is $w_h=1+\frac32m_hq$. For $v_1$, we have $\EE_D[v_1q]=0$, so its minimum-2nd-moment continuation satisfies $v_h\equiv[1,1,1]^\top$ for all $h\ge 2$. In contrast, $\EE_D[w_1^{\min}q]=-3/7$, yielding $m_h=-3\cdot 2^{h-1}/7$ and $w_h^{\min}=1-\frac97 2^{h-2}q$. Consequently,
    \[
    \|w_h^{\min}\|_{2,d_h^D}^2=1+\frac32m_h^2=1+\frac{27}{98}4^{h-1}=\Omega(4^h).
    \]
    Thus, minimizing the current second moment does not necessarily control the second moments of downstream weights.
\end{proof}


We also show that the learned weights are not stable w.r.t.~policy perturbations, whose proof uses a similar construction as the proposition above. 

\begin{proposition}[(\textbf{Q2}) Instability w.r.t.~policy perturbations]\label{prop:stability-2}
    For every $H\ge 1$ and $0<\delta<1/2$, there exists a zero-reward MDP with a realizable $2$-dimensional linear function class with feature $\phi$, and two policies $\pi,\widetilde{\pi}$ that differ only at $h=1$ with $\sup_s D_\mathrm{TV}(\pi_1(\cdot\mid s),\widetilde{\pi}_1(\cdot\mid s))=\delta$. Nevertheless, their minimum-2nd-moment weights with zero matching loss satisfy the following properties:\\
    (i) (Exponential amplification) For all $h\in[H]$, let $\psi_h^\pi:=\EE_{d_h^D}[w_h^{\star,\pi}\phi_h]$ denote the induced moment,
    \[
    \left\|w_h^{\star,\widetilde{\pi}}-w_h^{\star,\pi}\right\|_{2,d_h^D}=\left\|\psi_h^{\widetilde{\pi}}-\psi_h^\pi\right\|_2 =\delta\left(\frac{3}{2}\right)^{h-1};
    \]
    (ii) (Bounded 2nd moment) For exponentially small perturbation $\delta=c(2/3)^{H-1}$ for any constant $c\in(0,1/2)$, both policies satisfy
    \[
    \max_{h}\left\{\EE_{d_h^D}[(w_h^{\star,\pi})^2],\EE_{d_h^D}[(w_h^{\star,\widetilde{\pi}})^2]\right\}\le 1+c^2,\quad \text{while } \|\psi_H^{\widetilde{\pi}}-\psi_H^\pi\|_2=c;
    \]
    (iii) (Well-conditioned covariance) The feature covariance satisfies $\Sigma_h^D=I_2$ for all $h$; and for any policy $\pi'$, its feature covariance $\EE_{d_h^{\pi'}}[\phi_h\phi_h^\top]\preceq c'\Sigma_h^D$ for some constant $c'$.
\end{proposition}

\begin{proof}
    Consider an MDP with a state label $s\in\{-1,+1\}$ sampled uniformly at the initial step and preserved by subsequent transitions. The action space is $\{0,1\}$, and the behavior policy always chooses $a=0$. All rewards are zero, so realizability holds trivially. Define the feature map $\phi_h(s,a)=(1,s(1+a))^\top$ for every $h$. Under the behavior distribution, $\Sigma_h=\EE_D[\phi_h\phi_h^\top]=I_2$.

    Let $\pi_h(1\mid s)=1/2$ for all $h$. Since $\phi_h(s,\pi_h)=(1,\frac32s)^\top$ and the state labels are uniform, the initial target feature mean is $(1,0)^\top$. Exact moment matching therefore gives $w_1^{\star,\pi}(s)\equiv 1$, and the same holds inductively for all $h$, i.e., $w_h^{\star,\pi}(s)\equiv 1$.

    Now perturb only the first-stage policy by setting $\widetilde{\pi}_1(1\mid s)=1/2+\delta s$ and $\widetilde{\pi}_h=\pi_h$ for $h\ge 2$. The two policies differ by total variation distance $\delta$ at $h=1$. Moreover, $\phi_1(s,\widetilde{\pi}_1)=(1,\frac32s+\delta)^\top$, with expectation $(1,\delta)^\top$. Since $\Sigma_1=I_2$, the minimum-2nd-moment matching weight is $w_1^{\star,\widetilde{\pi}}(s)=1+\delta s$. For $h\ge 2$, exact matching requires $\EE_D[w_h]=1$ and $\EE_D[w_hs]=\frac32\EE_D[w_{h-1}s]$. Consequently,
    \[
    w_h^{\star,\widetilde{\pi}}(s)=1+\delta\left(\frac32\right)^{h-1}s,\quad\psi_h^{\widetilde{\pi}}-\psi_h^\pi=\binom{0}{\delta\left(\frac32\right)^{h-1}},
    \]
    which proves (i). Furthermore,
    \[
    \EE_{d_h^D}\left[\left(w_h^{\star,\pi}\right)^2\right]=1,\quad\EE_{d_h^D}\left[\left(w_h^{\star,\widetilde\pi}\right)^2\right]=1+\delta^2\left(\frac32\right)^{2(h-1)},
    \]
    which proves (ii). Finally, $\|\phi_h(s,a)\|_2^2\le 5$ and $\Sigma_h^D=I_2$ gives $\EE_{d_h^{\pi'}}[\phi_h\phi_h^\top]\preceq 5\Sigma_h^D$ for any policy $\pi'$, proving the all-policy squared-feature coverage condition in (iii).
\end{proof}

The most interesting regime to interpret the result is when $\delta$ is exponentially small as in (ii), where $c = \delta(3/2)^{H-1}$ is a constant. In this case, an exponentially small policy perturbation at $h=1$ induces a constant-order difference in the learned weights and their moments at $h=H$, even though both sets of weights have uniformly bounded 2nd-moments and the feature covariance matrix is well conditioned (iii). 
This implies that the exponential amplification is caused by the projected dynamics rather than poor data coverage.


\paragraph{Bellman completeness restores stability.} For $H\ge 2$, both examples in Propositions~\ref{prop:stability-1} and~\ref{prop:stability-2} violate Bellman completeness (Assumption~\ref{asm:completeness}). 
Indeed, Bellman completeness can be shown to eliminate the instabilities in general, which we briefly explain below. For (\textbf{Q1}), if two weights $u_h,v_h$ induce the same moments, i.e., $\EE_D[u_h f_h] = \EE_D[v_h f_h], \forall f_h \in \mathrm{span}(\Fcal_h)$---which can be implied by e.g., both having 0 matching losses---then for all $f\in\Fcal_{h+1}$,
\[
\EE_D[(u_h-v_h)\cdot f(s_{h+1},\pi)]=\EE_D[(u_h-v_h)\cdot (P_h^\pi f)(s_h,a_h)]=0.
\]
This holds because completeness gives us $P_h^\pi f \in \textrm{span}(\Fcal_h)$ (see Footnote~\ref{ft:span}). 
It implies that $u_h$ and $v_h$ induce exactly the same target for downstream matching problems, $\EE_D[u_h f(s_{h+1},\pi)] = \EE_D[v_h f(s_{h+1},\pi)]$, and hence the same min-2nd-moment continuations. 

For (\textbf{Q2}), let $\varepsilon_h=\sup_s D_\mathrm{TV}(\pi_h(\cdot\mid s),\widetilde{\pi}_h(\cdot\mid s))$ be the perturbation at step $h$. A standard coupling argument gives $D_\mathrm{TV}(d_h^\pi,d_h^{\widetilde{\pi}})\le\sum_{h'=1}^h\varepsilon_{h'}$. Therefore, for any $f\in\Hilb_h$,
\[
\Bigl|\EE_{d_h^D}[(w_h^{\star,\pi}-w_h^{\star,\widetilde{\pi}})f]\Bigr|=\Bigl|\EE_{d_h^\pi}[f]-\EE_{d_h^{\widetilde{\pi}}}[f]\Bigr|\le 2\|f\|_\infty\sum_{h'=1}^h\varepsilon_{h'},
\]
where the first equality follows from Proposition~\ref{prop:completeness} under Bellman completeness. In contrast to Proposition~\ref{prop:stability-2}, the induced moment difference grows at most linearly, rather than exponentially, with the cumulative policy perturbation.

Moreover, since both min-2nd-moment weights belong to $\Hilb_h$ (Proposition~\ref{prop:w-in-span}), the generalized leverage condition (Assumption~\ref{asm:leverage}) allows us to translate the moment bound into a weight bound by taking $f=w_h^{\star,\pi}-w_h^{\star,\widetilde\pi}$:
\[
\left\|w_h^{\star,\pi}-w_h^{\star,\widetilde{\pi}}\right\|_{2,d_h^D}\le2\sqrt{\kappa_h}\sum_{t=1}^h\epsilon_t.
\]
In the linear setting, if $\|\phi_h(s,a)\|_2\le B_\phi$ and $\Sigma_h^D$ is invertible, we have $\kappa_h\le B_\phi^2/\lambda_{\min}(\Sigma_h^D)$, where $\lambda_{\min}(\Sigma_h^D)$ is the smallest eigenvalue of $\Sigma_h^D$. In particular, when the feature covariance is well-conditioned as in Proposition~\ref{prop:stability-2}(iii), Bellman completeness ensures that a small policy perturbation also implies a small change in the learned weights, ruling out the instability exhibited in Proposition~\ref{prop:stability-2}.


\end{document}